\documentclass{article}

\usepackage{microtype}
\usepackage{graphicx}
\usepackage{subcaption}
\usepackage{booktabs} 
\usepackage{multirow}
\usepackage{listings}

\usepackage{hyperref}

\usepackage[ruled,vlined]{algorithm2e}

\usepackage[accepted]{icml2026}

\usepackage{amsmath}
\usepackage{amssymb}
\usepackage{mathtools}
\usepackage{amsthm}
\usepackage{enumitem}
\usepackage{wasysym}
\usepackage{xcolor}
\definecolor{myorange}{HTML}{C04000}
\definecolor{owtgreen}{HTML}{4A8F2A}
\definecolor{tinyorange}{HTML}{D66A00}
\usepackage{tikz}
\usepackage[most]{tcolorbox}

\tcbuselibrary{breakable}

\tcbset{
  bluebox/.style={
    breakable,
    colback=blue!5,
    colframe=blue!30,
    boxrule=0.8pt,
    arc=3pt,
    left=6pt,
    right=6pt,
    top=4pt,
    bottom=4pt
  },
  orangebox/.style={
    breakable,
    colback=orange!10,
    colframe=orange!40,
    boxrule=0.8pt,
    arc=3pt,
    left=6pt,
    right=6pt,
    top=4pt,
    bottom=4pt
  }
}

\newtcolorbox{Box1}[2][]{
  lower separated=false,
  colback=white,
  colframe=black,fonttitle=\bfseries,
  colbacktitle=white,
  coltitle=black,
  boxrule=1pt,
  boxsep=1pt,
  left=2mm,
  right=2mm,
  enhanced,
  attach boxed title to top left={yshift=-0.1in,xshift=0.15in},
                 boxed title style={colframe=black},
  title={#2},#1
}

\DeclareMathOperator*{\argmin}{arg\,min}
\DeclareMathOperator*{\argmax}{arg\,max}

\newcommand{\DC}{\mathcal{D}}
\newcommand{\LC}{\mathcal{L}}
\newcommand{\NC}{\mathcal{N}}

\newcommand{\UC}{\mathcal{U}}
\newcommand{\VC}{\mathcal{V}}

\newcommand{\EE}{\mathbb{E}}
\newcommand{\PP}{\mathbb{P}}
\newcommand{\RR}{\mathbb{R}}

\newcommand{\softmax}{\mathrm{softmax}}
\icmlsetsymbol{blank}{}
\usepackage{colortbl}

\newcommand{\ecoparagraph}[1]{\vspace{-0.05cm}\noindent\textbf{#1}}
\usepackage{wasysym}

\usepackage{enumitem}

\allowdisplaybreaks
\everypar{\looseness=-1}

\usepackage[capitalize,noabbrev]{cleveref}

\theoremstyle{plain}
\newtheorem{theorem}{Theorem}[section]

\newtheorem{lemma}[theorem]{Lemma}

\theoremstyle{definition}

\theoremstyle{remark}

\definecolor{pearDark}{HTML}{2980B9}
\definecolor{MyRed}{HTML}{C0392B}
\definecolor{debugcolor}{HTML}{1C1C1C}
\definecolor{algteal}{HTML}{00897B}
\definecolor{algorange}{HTML}{F57C00}
\definecolor{algamber}{HTML}{D6A000}
\definecolor{algmagenta}{HTML}{D81B60}
\definecolor{algviolet}{HTML}{6D28D9}
\definecolor{algblue}{HTML}{2563EB}
\newcommand{\reviewchange}[1]{#1}
\newenvironment{reviewblock}{}{}
\newcommand{\algkw}[1]{\textnormal{\texttt{\textcolor{algblue}{\bfseries\detokenize{#1}}}}}
\newcommand{\algfn}[1]{\textnormal{\texttt{\textcolor{algmagenta}{\detokenize{#1}}}}}
\newcommand{\algid}[1]{\textnormal{\texttt{\textcolor{algviolet}{\detokenize{#1}}}}}
\newcommand{\algstr}[1]{\textnormal{\texttt{\textcolor{algorange}{\detokenize{#1}}}}}
\newcommand{\algcom}[1]{\textnormal{{\ttfamily\textcolor{algteal}{\# #1}}}}
\newcommand{\algcommath}[1]{\textnormal{{\ttfamily\textcolor{algteal}{\# #1}}}}
\newcommand{\idlmalgcaption}[3]{%
  \refstepcounter{algocf}\label{#1}%
  \noindent\rule{\linewidth}{0.8pt}\par\vspace{0.5mm}%
  \noindent{\bfseries Algorithm~\thealgocf\ #2}\par\vspace{0.3mm}%
  {\scriptsize #3}\par\vspace{0.5mm}%
  \noindent\rule{\linewidth}{0.4pt}\par\vspace{-1mm}%
}
\newcommand{\idlmalgbottom}{\vspace{-1.5mm}\noindent\rule{\linewidth}{0.4pt}}
\lstdefinelanguage{IDLMPseudocode}{
  basicstyle=\ttfamily\small,
  keywordstyle=\color{black},
  commentstyle=\color{black},
  stringstyle=\color{black},
  morekeywords={},
  showstringspaces=false,
  columns=fullflexible,
  keepspaces=true,
  frame=none
}

\usepackage[textsize=tiny]{todonotes}

\icmltitlerunning{IDLM: Inverse-distilled Diffusion Language Models}

\begin{document}

\twocolumn[
  \icmltitle{IDLM: Inverse-distilled Diffusion Language Models}



  \icmlsetsymbol{equal}{*}

  \begin{icmlauthorlist}
    \icmlauthor{David Li}{equal,mbzuai}
    \icmlauthor{Nikita Gushchin}{equal,appliedai,axxx}
    \icmlauthor{Dmitry Abulkhanov}{blank}
    \icmlauthor{Eric Moulines}{mbzuai,epita}
    \icmlauthor{Ivan Oseledets}{axxx,appliedai}
    \icmlauthor{Maxim Panov}{mbzuai}
    \icmlauthor{Alexander Korotin}{appliedai,axxx}
  \end{icmlauthorlist}

  \icmlaffiliation{mbzuai}{Mohamed Bin Zayed University of AI, Abu Dhabi, UAE}
  \icmlaffiliation{appliedai}{Applied AI Institute, Moscow, Russia}
  \icmlaffiliation{epita}{EPITA, Laboratoire Recherche de l’EPITA, Paris, France}
  \icmlaffiliation{axxx}{AXXX, Russia}
  \icmlcorrespondingauthor{David Li}{David.Li@mbzuai.ac.ae}
  \icmlcorrespondingauthor{Alexander Korotin}{iamalexkorotin@gmail.com}

  \icmlkeywords{Machine Learning, ICML}

  \vskip 0.3in
]



\printAffiliationsAndNotice{\icmlEqualContribution}

\begin{abstract}
    Diffusion Language Models (DLMs) have recently achieved strong results in text generation. However, their multi-step sampling leads to slow inference, limiting practical use. To address this, we extend Inverse Distillation, a technique originally developed to accelerate continuous diffusion models, to the discrete setting.
    Nonetheless, this extension introduces both theoretical and practical challenges. From a theoretical perspective, the inverse distillation objective lacks uniqueness guarantees, which may lead to suboptimal solutions. From a practical standpoint, backpropagation in the discrete space is non-trivial and often unstable. To overcome these challenges, we first provide a theoretical result demonstrating that our inverse formulation admits a unique solution, thereby ensuring valid optimization. We then introduce gradient-stable relaxations to support effective training. 
    As a result, experiments on multiple DLMs show that our method, \textit{Inverse-distilled Diffusion Language Models (IDLM)}, reduces the number of inference steps by $4 \times$-$64 \times$, while preserving the teacher model’s generation quality.  We provide the code, model checkpoints, and video tutorials on the project page: 
    \looseness=-1

    \vspace{0.3ex}
    \centerline{\href{https://david-cripto.github.io/idlm-project-page/}{https://david-cripto.com/idlm}}
    \vspace{-5mm}
\end{abstract}

\section{Introduction}
\label{sec:introduction}

    \begin{figure}[t!]
    \centering
        \includegraphics[width=\linewidth, trim = 1.6cm 0.4cm 0.8cm 0.4cm, clip]{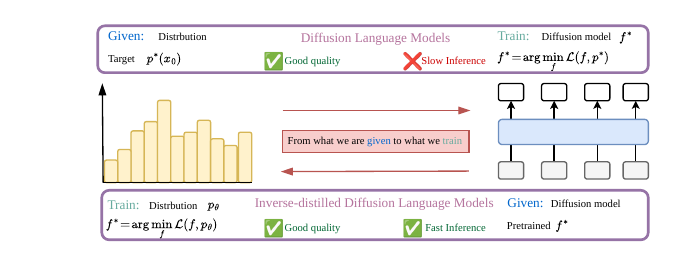}
        \caption{\textbf{Diffusion training vs inverse distillation.} DLMs learn a slow teacher $f^*$ from data. IDLM fixes $f^*$ and learns $p_\theta$ for fast few-step generation.}
        \vspace{-4mm}
        \label{fig:teaser}
    \end{figure}
  
    Generative modeling for discrete data, such as natural language, has seen widespread adoption, driven by the rapid advancement of large language models~\citep{touvron2023llama, groeneveld2024olmo, lozhkov2024starcoder}. 
    Recently, Diffusion Language Models 
    (DLMs; \citealp{sohl2015deep, austin2021structured, campbell2022continuous, lou2024discrete}) have emerged as a promising framework for modeling the distribution of the text data. A general formulation for a broad class of DLMs was introduced by~\citep{lou2024discrete}. However, the general formulation is often intractable in practice and introduces \textcolor{debugcolor}{additional} theoretical complexity. To address this, subsequent work~\citep{sahoo2024simple, sahoo2025the, schiff2024discreteguidance, shi2024simplified} has focused on specific types of diffusion processes, enabling simplified theoretical analyses and more tractable implementations. As a result, DLMs have achieved competitive performance in text generation, comparable to that of autoregressive language models~\citep{nie2025large}. However, the inherently iterative nature of reverse diffusion requires hundreds or even thousands of sampling steps at inference time, resulting in considerable latency that limits their practicality.

    On the other hand, in the continuous domain, the issue of slow inference in diffusion models~\citep{ho2020denoising, song2020score, karras2022elucidating} has been the focus of extensive research~\citep{song2021denoising, luhman2021knowledge, zhang2022fast, lu2022dpm}. Among the various approaches, \emph{distillation} has emerged as one of the most effective solutions~\citep{song2023consistency, salimans2022progressive, luhman2021knowledge, berthelot2023tract, xie2024distillation}. However, directly applying such techniques in the discrete domain is challenging due to the non-differentiable nature of discrete data. Recent efforts have begun to address this limitation for DLMs. Notably, SDTT~\citep{deschenaux2024beyond} and Duo-DCD~\citep{sahoo2025the} propose adaptations of consistency-based distillation methods~\citep{song2023consistency, song2023improved, kim2023consistency} to the discrete setting, demonstrating initial progress toward enabling fast inference in DLMs.

    In this work, we investigate an alternative distillation paradigm by extending the theoretical framework of \emph{Inverse Distillation}~\citep{gushchin2025inverse, kornilov2025universal, selikhanovych2025one}, originally developed for continuous diffusion models, to the discrete domain. However, this direct extension introduces several nontrivial challenges arising from the intrinsic properties of discrete data. First, \underline{theoretical concerns} emerge regarding the validity of the training objective, particularly due to the potential non-uniqueness of its minimizers. Second, \underline{practical challenges} arise from the non-smoothness of the discrete space, which complicates the optimization. 
  
    \textbf{Contributions.} To address the above-mentioned issues, we present \emph{Inverse-distilled Diffusion Language Models (IDLM)}, a principled framework for accelerating a broad class of discrete diffusion models. In summary, our core contributions are as follows:
    \vspace{-0.9em}

    \begin{enumerate}
    \setlength{\itemsep}{1pt}
    \setlength{\parskip}{1pt}
        \item We theoretically establish the uniqueness of the global optimum for the IDLM objective, validating the correctness of the proposed optimization procedure (\wasyparagraph\ref{sec:theoretical extension idlm}).
        \item We propose gradient-stable relaxations for effective training \textcolor{debugcolor}{in the non-smooth discrete domain} (\wasyparagraph\ref{sec:simplex relaxation}).
        \item We empirically demonstrate that IDLM achieves competitive generation quality, matching the performance of 1024-step teacher models while requiring up to $64\times$ fewer sampling steps (\wasyparagraph\ref{sec:exps}).
    \end{enumerate}


\section{Preliminaries}
\label{sec:preliminaries}

\ecoparagraph{Notations.}
    We represent scalar discrete random variables over $N$ categories as one-hot column vectors, with ${\VC = \{x \in \{0,1\}^N\colon \sum_{i=1}^N x_i = 1\}}$ denoting the set of such vectors. We further define the probability simplex $\Delta \vcentcolon= \{ z \in \RR^N\colon z \ge 0, \langle z, \vec{1} \rangle = 1 \}$ as the convex hull of $\VC$, where $\vec{1} \in \RR^N$ denotes the all-ones vector. 

\subsection{Diffusion Language Models}
\label{sec:discrete diffusion models}
    Diffusion Language Models (DLMs; \citealp{austin2021structured, campbell2022continuous, lou2024discrete}) generate samples from a data distribution $p^*(x_0)$ by learning to reverse a prescribed forward noising process. In this work, we focus on two tractable process families that are widely used in language modeling: the \textit{absorbing} process, which underlies MDLM, and the \textit{uniform} process, which underlies UDLM and Duo. \reviewchange{We defer the general CTMC formulation and its SEDD score-entropy instantiation to Appendix~\ref{app:details of considered processes}. The process-specific details for masked diffusion and uniform diffusion are expanded in Appendices~\ref{app:absorbing process} and~\ref{app:uniform process}, respectively.}

\subsubsection{Absorbing Process}
\label{sec:absorbing process mdlm}
    \paragraph{Forward process.}
    The absorbing process corrupts tokens by replacing them with the mask token $m$. Its terminal distribution is concentrated on $m$, and under the usual schedule $\alpha_t \in [0,1]$, the forward transition from a clean token $x_0$ is
    \[
        q_t(x_t \mid x_0)
        = \mathrm{Cat}\bigl(x_t; \alpha_t x_0 + (1-\alpha_t)m\bigr).
    \]
    Thus, a token is either still equal to $x_0$ or has already become masked.

    \paragraph{Training loss.}
    MDLM~\citep{sahoo2024simple, shi2024simplified} parameterizes the denoiser as $f\colon \VC \times [0,1]\to\Delta$ and trains it with the weighted masked prediction objective
    \begin{equation}
    \label{eq:mdlm loss}
    \begin{aligned}
    \LC_{\text{MDLM}}(f, p^*)
    \vcentcolon= -
    \EE_{p^*(x_0), t, q_t}
    \bigl[
        \lambda_{t} \langle  x_0, \log f(x_t, t) \rangle
    \bigr].
    \end{aligned}
    \end{equation}
    where $\lambda_t$ is a positive weighting function and the logarithm is applied element-wise.

    \paragraph{Sampling procedure.}
    The \textsc{subs} parameterization of MDLM encodes two useful properties: unmasked tokens are copied forward, and the model assigns zero probability to predicting the mask token as clean data. For $0\leq s<t\leq1$, the learned reverse transition is the absorbing posterior with the clean token replaced by the model prediction, define $f_t\vcentcolon=f(x_t,t)$:
    \begin{equation}        
    \!\!p^{f}_{s\mid t}(x_s\!\mid\! x_t)
    \!\!=\!\!
    \begin{cases}
    \mathrm{Cat}(x_s; x_t),\!\! & \!\!x_t\neq m,\\[2pt]
    \mathrm{Cat}\!\left(
    \!x_s;
    \!\!\begin{gathered}
    \dfrac{(1\!-\!\alpha_s) m \! +\! (\alpha_s\!-\!\alpha_t)f_t}{1-\alpha_t}
    \end{gathered}\!\!
    \right), \!\! & \!\!x_t=m .
    \end{cases} \nonumber
    \end{equation}
    Sampling starts from a fully masked sequence $x_1\sim\delta_m(\cdot)$ and repeatedly draws $x_s\sim p^{f}_{s\mid t}(\cdot\mid x_t)$ along a decreasing time grid. Once a token is revealed, it is carried over in later denoising steps. The masked-diffusion sampling process is visualized in Figure~\ref{fig:masked-diffusion-sampling-intuition}. The full MDLM formulation is given in Appendix~\ref{app:mdlm-formulation-theory}.

\subsubsection{Uniform Process}
\label{sec:uniform process duo}
    \paragraph{Forward process.}
    The uniform process corrupts tokens toward the uniform distribution $\frac{1}{N}\vec{1}$. With the same notation $\alpha_t$, its forward transition can be written compactly as
    \[
        q_t(x_t \mid x_0)
        = \mathrm{Cat}\bigl(x_t; \alpha_t x_0 + (1-\alpha_t)\tfrac{1}{N}\vec{1}\bigr).
    \]
    Unlike absorbing diffusion, an intermediate token can be changed again at later reverse steps. This makes the process self-correcting: generation must decide both \textit{which} token to move to and \textit{when} to move.

    \paragraph{Training loss.}
    UDLM~\citep{schiff2024discreteguidance} uses the clean-token predictor as MDLM. Duo~\citep{sahoo2025the} keeps the same uniform discrete process but evaluates the model on a Gaussian-relaxed input, which gives smoother gradients during training while recovering the hard uniform process in the low-temperature limit, which makes the following parameterization of the model $f\colon \Delta \times [0,1]\to\Delta$. We write the objective as
    \begin{equation}
    \label{eq:duo loss}
    \LC_{\text{Duo}}(f,p^*)
    \vcentcolon=
      \textstyle
    \EE_{p^*(x_0), t, q_t}
      \bigl[
        g_{\mathrm{Duo}}\bigl(x_t ,x_0,f(x_t,t)\bigr)
      \bigr].
    \end{equation}
    All details of $g_{\mathrm{Duo}}$ are deferred to Appendix~\ref{app:uniform process}.

    \paragraph{Sampling procedure.}
    Sampling starts from a sequence of tokens sampled independently from the uniform distribution, $x_1\sim\mathrm{Cat}(\vec{1}/N)$. Then the sequence is iteratively refined by the learned uniform reverse posterior. For $0\leq s<t\leq1$, define $\alpha_{t\mid s}\vcentcolon=\alpha_t/\alpha_s$ and $f_t\vcentcolon=f(x_t,t)$. The ancestral sampler uses
    \begin{gather}
    \label{eq:uniform reversed}
    p^f_{s\mid t}(x_s\mid x_t) =
        \mathrm{Cat}\!\Bigl(x_s;
        \frac{N\alpha_t\,x_t\odot f_t
        +(\alpha_{t\mid s}-\alpha_t)x_t}{N\alpha_t\langle x_t,f_t\rangle+1-\alpha_t} \nonumber \\[-1pt]
        +\frac{(\alpha_s-\alpha_t)f_t
        +(1-\alpha_{t\mid s})(1-\alpha_s)\vec{1}/N}{{N\alpha_t\langle x_t,f_t\rangle+1-\alpha_t}} \Bigr). \nonumber
    \end{gather}
    \noindent In the sampling procedure each reverse step can revise the whole sequence.
    The uniform-diffusion sampling process is visualized in Figure~\ref{fig:appendix process intuition}. The full UDLM/Duo formulation is given in Appendix~\ref{app:uniform process}.
    
    The Duo paper also introduces an additional greedy sampling procedure. In the experiments, we report both variants and denote ancestral sampling by superscript ($^{\mathrm{a}}$) and greedy sampling by ($^{\mathrm{g}}$).

\subsubsection{Unified Notations \& Sequence Level Loss}
\label{sec:unified notations sequence loss}
    \paragraph{Unified token-level objective.}
    SEDD~\eqref{eq:sedd loss}, MDLM~\eqref{eq:mdlm loss}, and Duo~\eqref{eq:duo loss} use different losses, but they share the same structure. We write $\LC \in \{\LC_{\text{SEDD}}, \LC_{\text{MDLM}}, \LC_{\text{Duo}}\}$ in order to generalize our distillation scheme. Each loss can be expressed by the integrand $g$ of respective loss, so for any distribution $p$ we define:
    \[
    \begin{aligned}
    \LC(f,p)
    &=
    \EE_{p(x_0), t, q_t}\bigl[
        g\bigl(x_t,x_0,f(x_t,t)\bigr)
    \bigr].
    \end{aligned}
    \]

    \paragraph{Sequence-level objective.}
    In practical applications, we are interested in generating sequences $x_0^{1:L}$ of length $L$. To avoid an exponentially large state space, DLMs usually apply the forward corruption independently across sequence positions~\citep{lou2024discrete}. Under this assumption, the forward distribution factorizes as
    \[
        q_t(x_t^{1:L} \mid x_0^{1:L})
        = \prod_{l=1}^L q_t(x_t^l \mid x_0^l).
    \]
    The corresponding sequence-level diffusion loss becomes
    \begin{equation}
    \label{eq:sequence level loss}
    \begin{aligned}
    \LC^{1:L}(f,p^*)
    &=
    \EE_{p^*(x_0), t, q_t}\biggl[
        \sum_{l=1}^L
        g\bigl(
        x_t^l,x_0^l,f^l(x_t^{1:L},t)
        \bigr) \nonumber
    \biggr].
    \end{aligned}
    \end{equation}

\subsection{Inverse Distillation}
\label{sec:inverse distillation}

\ecoparagraph{Main goal.}
    The goal of inverse distillation is to directly train a student distribution $p_\theta(x_0)$, represented by a fast generator, so that it approximates the true data distribution $p^*(x_0)$. We use the pretrained diffusion model $f^*$, which was trained on samples from $p^*$. The good $p_\theta$ should be one for which the same diffusion training procedure would recover $f^*$.

\ecoparagraph{Diffusion training.}
    In usual diffusion training problem: given samples from $p^*$, we train the diffusion model by solving
    \begin{equation}
        f^* = \arg\min\nolimits_{f} \LC(f, p^*).
        \label{eq:teacher_objective_discr}
    \end{equation}
    Thus, the pretrained model $f^*$ is the diffusion model whose predictions are optimal for the real data distribution, we call this model the \textit{teacher}.

\ecoparagraph{Inverse viewpoint.}
    Inverse distillation reverses the roles of the known and unknown objects. We are given the pretrained diffusion model $f^*$ and seek a student distribution $p_\theta(x_0)$, such that training the same diffusion objective on samples from $p_\theta$ would recover the teacher:
    \begin{equation}
        f^* = \arg\min\nolimits_{f} \LC(f, p_{\theta}).
        \label{eq:inverse_optimality_condition}
    \end{equation}
    In words, diffusion training asks which diffusion model fits a distribution, while inverse distillation asks which distribution would make the fixed pretrained diffusion model optimal.

\ecoparagraph{Inverse distillation loss.}
    The optimality condition in~\eqref{eq:inverse_optimality_condition} is not directly feasible as a training objective, since it requires optimizing parameters $\theta$ through an $\argmin$. Following prior work on inverse distillation~\citep{gushchin2025inverse,kornilov2025universal, selikhanovych2025one}, we use tractable inverse loss
    \begin{gather}
    \begin{aligned}
    \LC_{\text{inv}}(\theta)
    \vcentcolon=
        \LC(f^*, p_{\theta})-
        \min\nolimits_{f} \LC(f, p_{\theta}).
    \end{aligned}
    \label{eq:inverse optimization problem}
    \end{gather}
    The inner minimization defines a \emph{fake} diffusion model trained on samples from $p_\theta$. The term $\LC(f^*,p_\theta)$ measures the teacher denoiser's performance on samples from $p_\theta$, while $\min_f \LC(f,p_\theta)$ is the best achievable denoising loss for a diffusion model trained on $p_\theta$. The generator is therefore optimized by minimizing this gap. By construction, $\LC_{\text{inv}}(\theta)\geq0$, and if $p_\theta=p^*$ then $\LC_{\text{inv}}(\theta)=0$. The converse is not guaranteed for this general objective.

\ecoparagraph{Connection to this work.}
    In \emph{continuous} space, inverse distillation was extensively explored by prior works~\citep{gushchin2025inverse,kornilov2025universal, selikhanovych2025one}. Our main goal is to extend this idea to the \emph{discrete} domain. However, it presents several challenges, stemming from the inherent characteristics of discrete spaces. These include \underline{theoretical concerns} regarding the validity of the training objective, due to the non-uniqueness of its minimizers, as well as \underline{practical difficulties} related to the non-differentiability of discrete sampling.

\begin{figure*}[t!]
\centering
    \includegraphics[width=1.0\linewidth, trim = 0.32cm 0.3cm 0.33cm 0.2cm, clip]{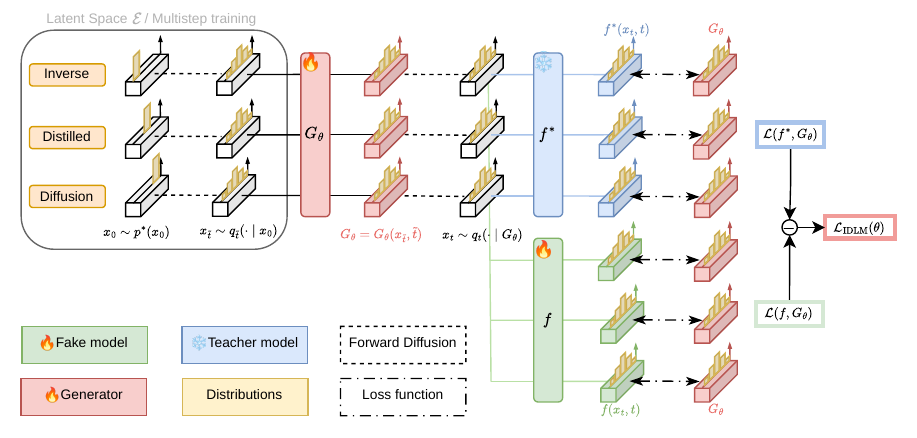}
    \caption{\textbf{IDLM training loop.} The student generator $G_\theta$ maps noised data to a clean sample, which is forward-diffused and passed as input to the frozen teacher $f^*$ and the learnable fake model $f$. Training alternates between fitting the fake model to the current student distribution and updating the generator with the IDLM loss. This loss pushes $G_\theta$ toward samples supported by the teacher more than by the fake model.}
\label{fig:method}
\vspace{-5mm}
\end{figure*}

\vspace{-2mm}
\section{Inverse-distilled Diffusion Language Models}
\vspace{-2mm}
\label{sec:idlm}
    In this section, we present our core methodology for distilling a broad class of Diffusion Language Models into a few-step generator. We begin by providing the theoretical foundations supporting the extension of the Inverse Distillation framework to the discrete domain (\wasyparagraph\ref{sec:theoretical extension idlm}). We then address the practical challenges arising from the non-smooth nature of discrete spaces (\wasyparagraph\ref{sec:practical extension of idlm}). Lastly, we outline key implementation details \textcolor{debugcolor}{and introduce the sequence-level scaling of the proposed objective} (\wasyparagraph\ref{sec:technical aspects}). Appendices~\ref{app:absorbing process} and~\ref{app:uniform process} give the masked- and uniform-diffusion formulations used below.

\vspace{-1mm}
\subsection{Theoretical extension of Inverse Distillation}
\vspace{-1mm}
\label{sec:theoretical extension idlm}

\ecoparagraph{Inverse distillation for DLMs.}
    To address the issue of slow inference, we extend the Inverse Distillation framework, originally developed for continuous diffusion models, to the domain of Diffusion Language Models. Then, following the logic of the continuous case, described in (\wasyparagraph\ref{sec:inverse distillation}), we propose to consider the following objective for the discrete diffusion:
    \begin{equation}
    \label{eq:inverse discrete diffusion distillation}
        \mathcal{L}_{\text{IDLM}} (\theta ) = \LC(f^*, p_{\theta}) \!
        - \! \min_{f} \LC(f, p_{\theta}).\!
    \end{equation}

\ecoparagraph{Theoretical validity of the training objective.}
    However, simply formulating this objective in a manner analogous to the distillation loss used in continuous diffusion models does not guarantee the validity of its solutions. The absence of a uniqueness guarantee implies that the optimization procedure may converge to suboptimal solutions. To address this, we present the following theorem. 
    
    \begin{theorem}[Unique solution]
    \vspace{-1mm}
    \label{thm:unique solution} 
        For the SEDD~\eqref{eq:sedd loss}, MDLM~\eqref{eq:mdlm loss}, and Duo~\eqref{eq:duo loss} (in the limit as $\tau \to 0^{+}$) objectives the IDLM loss defined in equation~\eqref{eq:inverse discrete diffusion distillation} satisfies
        \begin{equation}
            \LC_{\text{IDLM}}(\theta) \geq \mathcal{D}_{\text{KL}}(p_{\theta} \parallel p^*)\geq 0 \nonumber
        \end{equation}
        and achieves its minimum (zero) if and only if the model distribution matches the target distribution
        \vspace{-2mm}
        \begin{equation}
            \LC_{\text{IDLM}}(\theta) = 0 \iff p_{\theta} = p^*. \nonumber
        \end{equation}
    \vspace{-2mm}
    \end{theorem}
    \vspace{-5mm}
    
    We give the proof and discuss the sequence-level case in Appendix~\ref{app: proof of the theorem}.

\ecoparagraph{Probabilistic intuition.}
    The IDLM loss has a simple probabilistic interpretation: it matches full diffusion trajectories. If $\PP^\theta$ and $\PP^*$ denote the trajectory distributions induced by $p_\theta$ and $p^*$, respectively, then the proof shows
    \[
        \LC_{\mathrm{IDLM}}(\theta)
        =
        \DC_{\mathrm{KL}}(\PP^\theta\|\PP^*).
    \]
    Because the final clean sample is part of the trajectory, this KL controls $\DC_{\mathrm{KL}}(p_\theta\|p^*)$, which is exactly the inequality in Theorem~\ref{thm:unique solution}. Hence zero IDLM loss can occur only when the student recovers the data distribution.


\vspace{-2mm}
\subsection{Addressing Discrete Nonsmoothness}
\label{sec:practical extension of idlm}
\vspace{-2mm}
    We now explain the main optimization problems that arise when applying IDLM in a discrete domain. For clarity, we use MDLM as an example because the absorbing process gives the simplest setting in which these bottlenecks can be seen explicitly. The same problems also arise for uniform diffusion. Appendix~\ref{app:mdlm-inverse-distillation} expands the MDLM-specific inverse distillation simplification used below. Appendix~\ref{app:uniform-inverse-distillation} gives the corresponding uniform-diffusion specialization. Specializing the IDLM objective to the MDLM loss gives
    \begin{equation}
    \label{eq:draft_idlm_mdlm}
    \LC_{\mathrm{IDLM-MDLM}}(\theta)
    =
    \LC_{\text{MDLM}}(f^*,p_\theta)
    -
    \min_f \LC_{\text{MDLM}}(f,p_\theta). \nonumber
    \end{equation}
    Here we use the same MDLM loss as in the Equation~\eqref{eq:mdlm loss}, replacing the data distribution $p^*$ by the student distribution $p_\theta$. This is the quantity through which the distribution $p_\theta$ must be optimized:
    \begin{equation}
    \label{eq:draft_mdlm_distribution_loss}
    \LC_{\text{MDLM}}(f,p_\theta)
    \!\!\vcentcolon=\!\!
    -
    \EE_{p_\theta(x_0),\,t,\,q_t}
    \bigl[
    \lambda_t
    \langle x_0,\log f(x_t,t)\rangle
    \bigr].
    \end{equation}
    In practice, we parametrize ditribution $p_{\theta}$ by \textit{student} model ${G_{\theta}\colon \mathcal{E} \to \VC}$, which transforms latent variable $\epsilon \in \mathcal{E}$ to a discrete output $x_0 \in \VC$. The latent space $\mathcal{E}$ is equipped with a tractable prior distribution $p_{\mathcal{E}}$, inducing a distribution $p_{\theta}$ over $\VC$ via the mapping $G_{\theta}$:
    \[
    \epsilon\sim p_{\mathcal E},
    \qquad
    x_0=G_{\theta}(\epsilon).
    \]
    In the equations below, we highlight with \textcolor{MyRed}{red} quantities whose values depend on the generator parameter \textcolor{MyRed}{$\theta$}. Substituting the generator sample into~\eqref{eq:draft_mdlm_distribution_loss} gives the loss:
    \begin{equation}
    \label{eq:draft_mdlm_generator_loss}
    \LC_{\text{MDLM}}(f,p_\theta)
    =
    -
    \EE_{\epsilon,\,t,\,
    \textcolor{MyRed}{x_t}\sim q_t(\cdot\mid G_{\textcolor{MyRed}{\theta}}(\epsilon))}
    \bigl[
    \lambda_t
    \langle
    G_{\textcolor{MyRed}{\theta}}(\epsilon),
    \log f(\textcolor{MyRed}{x_t},t)
    \rangle
    \bigr]. \nonumber
    \end{equation}

    This formulation reveals two gradient bottlenecks. First, if $G_{\textcolor{MyRed}{\theta}}$ is forced to output a one-hot token in $\VC$, then the generator output is not a smooth function of its logits. Second, $\textcolor{MyRed}{x_t}$ is sampled from $q_t(\cdot\mid G_{\textcolor{MyRed}{\theta}}(\epsilon))$, hence the noised token depends on $\textcolor{MyRed}{\theta}$ through the probabilities of this categorical distribution. Backpropagating through the objective would therefore require passing gradients through the sampling of $\textcolor{MyRed}{x_t}$. This is not handled by automatic differentiation, and common alternatives such as hard Gumbel-Softmax can be unstable.

    \ecoparagraph{Simplex relaxation.}
    \label{sec:simplex relaxation}
    To remove the first bottleneck, we relax the generator range from $\VC$ to the simplex $\Delta$:
    \[
    G_{\textcolor{MyRed}{\theta}}(\epsilon)\in\Delta.
    \]
    This relaxation is compatible with both ingredients of the MDLM loss: the cross-entropy term remains a valid soft-label loss, and the absorbing conditional distribution $q_t(\cdot~\mid ~G_{\textcolor{MyRed}{\theta}}(\epsilon))$ remains a categorical distribution when $G_{\textcolor{MyRed}{\theta}}(\epsilon)\in\Delta$.
    Figures~\ref{fig:masked-forward-kernel},~\ref{fig:mdlm-simplex-relaxation}, and~\ref{fig:mdlm-token-loss} in Appendix~\ref{app:mdlm-inverse-distillation} visualize why this remains a valid update and how it changes a one-hot token loss into a weighted token loss.
    Appendix~\ref{app:uniform-inverse-distillation} uses the same relaxation for uniform diffusion.

    \ecoparagraph{MDLM: \textsc{subs} parameterization.}
    For MDLM, the \textsc{subs} parameterization copies already unmasked tokens: if $x_t\neq m$, then $f^*(x_t,t)=f(x_t,t)=x_t$ for any model. Consequently, the corresponding token contribution to $\LC_{\mathrm{IDLM-MDLM}}$ is zero. Hence the IDLM update is nonzero only when the noised token is the mask token, $x_t=m$. The absorbing-process details are given in Appendix~\ref{app:absorbing process}. Appendix~\ref{app:mdlm-inverse-distillation} and Figure~\ref{fig:mdlm-subs-cancellation} show this mask-only cancellation explicitly. Therefore, for the generator update, the MDLM loss reduces to
    \begin{equation}
    \label{eq:draft_mdlm_mask_only}
    \LC_{\text{MDLM}}(f,p_\theta)
    =
    -
    \EE_{\epsilon,t}
    \bigl[
    (1-\alpha_t)\lambda_t
    \langle
    G_{\textcolor{MyRed}{\theta}}(\epsilon),
    \log f(m,t)
    \rangle
    \bigr]. \nonumber
    \end{equation}
    The sampled token $x_t$ no longer appears in the generator update. The remaining dependence on $\textcolor{MyRed}{\theta}$ is only through the differentiable simplex-valued output $G_{\textcolor{MyRed}{\theta}}(\epsilon)$. Appendix~\ref{app:mdlm-inverse-distillation} explains why this can be viewed as an implicit stop-gradient through the intermediate state for masked diffusion.

    \ecoparagraph{Duo: Gaussian reparameterization.}
    Duo provides a differentiable path by introducing an auxiliary Gaussian variable $\xi$ and using the relaxed state
    \begin{equation}
    \label{eq:draft_duo_soft_state}
    \textcolor{MyRed}{x_t}
    =
    \softmax\!\left(
    \frac{
    \tilde{\alpha}_t G_{\textcolor{MyRed}{\theta}}(\epsilon)
    +
    \sqrt{1-\tilde{\alpha}_t^2}\,\xi
    }{\tau}
    \right),
    \qquad
    \xi\sim\NC(0,I), \nonumber
    \end{equation}
    where the \textcolor{MyRed}{red} terms indicate the path through the generator parameters. The map $G_{\textcolor{MyRed}{\theta}}(\epsilon)\mapsto \textcolor{MyRed}{x_t}$ is differentiable. This is possible in the Duo parameterization because the denoiser is trained on simplex-valued states rather than only on discrete one-hot tokens.
    More details are given in Appendix~\ref{app:uniform-inverse-distillation}.

\vspace{-1mm}
\subsection{Loss Intuition}
\label{sec:loss intuition}
\vspace{-1mm}
\ecoparagraph{Alternating Optimization.}
    The IDLM objective contains an inner optimization over the fake diffusion model $f$. In practice, we approximate this nested problem by
    alternating two updates, following prior work~\citep{yin2024one,
    yin2024improved, zhou2024score, gushchin2025inverse,
    kornilov2025universal}.

    First, we fix the student generator $G_{\theta}$ and train
    $f$ on samples from the current student distribution
    $p_{\theta}$. This approximates the inner minimization in the IDLM
    objective:
    \begin{gather}
    \label{eq:update fake distribution}
      \text{\textit{Update:} } f \text{;} \quad
      \text{\textit{Fix:} } \theta  \\
      \LC_{f} = \LC(f, p_{\theta}). \nonumber
    \end{gather}

    Second, we fix $f$ and update $G_{\theta}$ using the IDLM gap:
    \begin{gather}
    \label{eq:update generator distribution}
      \text{\textit{Update:} } \theta \text{;} \quad
      \text{\textit{Fix:} } f \\
      \LC_{\mathrm{IDLM}}(\theta)
      =
      \LC(f^*, p_{\theta})
      -
      \LC(f, p_{\theta}). \nonumber
    \end{gather}

    These two updates are summarized visually in Figure~\ref{fig:method},
    which depicts the fake model and the student generator as the two
    learnable components of the IDLM training loop.

\ecoparagraph{IDLM intuition.}
    For MDLM, the \textsc{subs} parameterization makes the generator update
    especially transparent: only masked positions contribute. Substituting the
    mask-only MDLM loss into the IDLM gap gives
    \[
    \LC_{\mathrm{IDLM-MDLM}}(\theta)
    =
    -\EE_{\epsilon,t}\!
    \left[
    \lambda_t\big\langle G_\theta(\epsilon),a_t\big\rangle
    \right].
    \]
    Here \(a_t=\log f^*(m,t)-\log f(m,t)\), and \(f\) is the current fake model. Thus
    each masked position gives the
    generator a teacher-over-fake advantage vector \(a_t\). If
    \(G_\theta(\epsilon)=\softmax(z_\theta)\), then for a fixed \((\epsilon,t)\)
    \[
    \frac{\partial \LC_{\mathrm{IDLM-MDLM}}}{\partial z_{\theta,i}}
    =
    -\lambda_t G_{\theta,i}(\epsilon)
    \Big(a_{t,i}-\langle G_\theta(\epsilon),a_t\rangle\Big).
    \]
    Hence gradient descent increases token \(i\)'s logit exactly when its advantage
    is above the generator's current average advantage. In words, the student does
    not simply chase the teacher's most likely token, it promotes tokens that the
    teacher prefers more than the fake model does. When the fake model catches up
    to the teacher on student samples, this relative advantage vanishes.

    \begin{center}
    {\small
    \setlength{\tabcolsep}{2.5pt}
    \begin{tabular}{@{}>{\raggedright\arraybackslash}p{0.31\linewidth}>{\raggedright\arraybackslash}p{0.23\linewidth}>{\raggedright\arraybackslash}p{0.37\linewidth}@{}}
    \toprule
    Case & Token behavior & Intuition \\
    \midrule
    \(a_{t,i}>\langle G_\theta,a_t\rangle\)
    & logit increases
    & $f^*$ prefers token \(i\) more than $f$ does \\
    \(a_{t,i}<\langle G_\theta,a_t\rangle\)
    & logit decreases
    & token \(i\) is below the current advantage \\
    \(a_t\approx 0\)
    & no relative signal
    & $f$ has matched $f^*$ on student samples \\
    \bottomrule
    \end{tabular}
    }
    \end{center}

    For uniform diffusion, the same idea applies to a noised transition rather
    than only to a masked clean-token prediction. The local advantage becomes
    \[
    \begin{aligned}
    a_t
    =
    g_{\mathrm{Duo}}(x_t,G_{\theta}(\cdot),f^*(x_t,t))-
    g_{\mathrm{Duo}}(x_t,G_{\theta}(\cdot),f(x_t,t)),
    \end{aligned}
    \]
    Appendix~\ref{app:uniform process} gives \(g_{\mathrm{Duo}}\), and Appendix~\ref{app:uniform-inverse-distillation} expands the advantage.
    Since uniform diffusion can revise tokens repeatedly, this advantage guides
    both which token to choose and when the reverse path should jump.

\ecoparagraph{Why few-step distillation is difficult.}
    Standard discrete diffusion samplers usually draw the next state
    conditionally independently across output positions, even though each
    factor can attend to the whole current sequence. With many reverse steps,
    the composition of these local transitions can still represent a rich
    sequence distribution. In the few-step regime, however, each reverse step
    must replace many small teacher updates, making this factorization a
    bottleneck.
    IDLM addresses this bottleneck by learning a mixture of sequence-level
    distributions, as described next.

\begin{figure*}[!t]
\centering
\begin{minipage}[t]{0.49\linewidth}
\idlmalgcaption{alg:idlm}{IDLM: training.}{One branch is active per update: fit the fake model or distill the student.}
\begin{lstlisting}[language=IDLMPseudocode, escapechar=`, basicstyle=\ttfamily\scriptsize, breaklines=true, gobble=4]
    `\algcom{teacher \algid{f*}, fake model \algid{f}, and student generator \algid{net}}`
    `\algcom{mode in \{\algstr{"fake"}, \algstr{"student"}\} selects the updated component}`

    x_0 = `\algfn{sample_data}`(p_data)
    t = `\algfn{uniform}`(0, 1)
    x_t = `\algfn{corrupt}`(x_0, t)  `\algcommath{$q_t(\cdot\mid x_0)$}`
    x_hat = `\algid{net}`(x_t, t)

    `\algkw{if}` mode == `\algstr{"fake"}`:
        loss = `\algfn{L_seq}`(`\algid{f}`, x_hat)
        `\algfn{update}`(`\algid{f}`, loss)
    `\algkw{else}`:
        loss = `\algfn{L_seq}`(`\algid{f*}`, x_hat) - `\algfn{L_seq}`(`\algid{f}`, x_hat)
        `\algfn{update}`(`\algid{net}`, loss)
    \end{lstlisting}
\idlmalgbottom
\end{minipage}\hfill
\begin{minipage}[t]{0.49\linewidth}
\idlmalgcaption{alg:idlm-sampling}{IDLM: few-step sampling.}{The reverse transition is chosen by the diffusion process.}
\begin{lstlisting}[language=IDLMPseudocode, escapechar=`, basicstyle=\ttfamily\scriptsize, breaklines=true, gobble=4]
    `\algcommath{grid = [$t_K$, ..., $t_0$]}`

    `\algcommath{masked or uniform terminal distribution}`
    x_t = `\algfn{sample_terminal}`(process)
    `\algkw{for}` k `\algkw{in}` `\algfn{range}`(K, 0, -1):
        t = grid[k]
        t_prev = grid[k - 1]
        x_hat = `\algid{net}`(x_t, t)

        `\algcommath{masked or uniform reverse step}`
        x_prev = `\algfn{sample_reverse}`(x_t, x_hat, t, t_prev)
        x_t = x_prev
    `\algkw{return}` x_t
    \end{lstlisting}
\idlmalgbottom
\end{minipage}
\vspace{-2mm}
\end{figure*}

\ecoparagraph{Sequence-level mixture.}
    A full distribution over $\VC^L$ would require $N^L$
    probabilities, which is intractable to parameterize directly.
    The key is that the student uses a structured mixture: it samples a
    shared variable $\epsilon\sim p_{\mathcal E}$ and outputs
    $G_\theta^{1:L}(\epsilon)
    =(G_\theta^1(\epsilon),\ldots,G_\theta^L(\epsilon))\in\Delta^L$.
    For this fixed $\epsilon$, the component may factorize over positions,
    \[
    \begin{aligned}
        p_\theta(x_0^{1:L}\mid\epsilon)
        &=
        \prod_{i=1}^L
        \mathrm{Cat}\!\left(x_0^i;G_\theta^i(\epsilon)\right),\\
        p_\theta(x_0^{1:L})
        &=
        \mathbb{E}_{\epsilon\sim p_{\mathcal E}}
        \left[
        p_\theta(x_0^{1:L}\mid\epsilon)
        \right].
    \end{aligned}
    \]
    The product describes one component, while the expectation mixes many
    such components. Since the same $\epsilon$ controls all positions, a
    component can encode a sentence-level choice. Although we use
    simplex-valued outputs for differentiable training, in theory the
    generator should produce vocabulary tokens. Hence, during training, each
    $G_\theta^i(\epsilon)$ should concentrate near a one-hot token, so the
    component degenerates to a point mass on one full sequence. Thus the
    generator represents text as a mixture over coherent sequence-level atoms,
    allowing $p_\theta$ to capture semantic correlations without explicitly
    enumerating the full sequence space.
    A similar latent-mixture intuition appears in VADD~\citep{xie2026variational}
    and Di4C$^2$~\citep{hayakawa2025distillation}, but with different
    objectives.

\vspace{-1mm}
\subsection{Technical Aspects}
\label{sec:technical aspects}
\vspace{-1mm}
    This section specifies what latent space we use in practice and clarifies the remaining implementation details. The training update is summarized in Algorithm~\ref{alg:idlm}, the few-step sampler in Algorithm~\ref{alg:idlm-sampling}, and the full pipeline in Figure~\ref{fig:method}.
    
\ecoparagraph{Model initialization.}
    We initialize both the student generator $G_\theta$ and the auxiliary fake model $f$ from the pretrained teacher $f^*$. Since the teacher has already learned the semantic structure of the token space, this provides a strong initialization and lets the student refine the teacher's predictions rather than learn them from scratch.

\ecoparagraph{Multistep distillation.}
    We found that directly distilling the model into a one-step generator is difficult in practice, as the student must map terminal noise to a clean sample in a single pass. We therefore allow $G_\theta$ to produce samples in the few-step regime. Concretely, we set the generator latent input to $\epsilon=(x_t,t)$. During training, these inputs are obtained by corrupting real data:
    \[
        x_0\sim p^*(x_0),\qquad
        t\sim \UC[0,1],\qquad
        x_t\sim q_t(\cdot\mid x_0),
    \]
    and then feeding $(x_t,t)$ to the student. This multistep parameterization allows the generator to accept intermediate noised states as input, so at inference we can use the same few-step reverse sampling procedure as the teacher model, see Algorithm~\ref{alg:idlm-sampling}.

\ecoparagraph{Sequence-level loss.}
    All practical updates use the sequence-level loss introduced in the (\wasyparagraph\ref{sec:unified notations sequence loss}). That is, the token-level objective $\LC(f,p)$ is replaced by $\LC^{1:L}(f,p)$. The forward corruption factorizes across positions, but the teacher, fake model, and student are evaluated on the full noised sequence.

\begin{figure*}[t]
    \centering
    \begin{minipage}[t]{0.47\textwidth}
      \centering
      \includegraphics[width=0.85\linewidth]{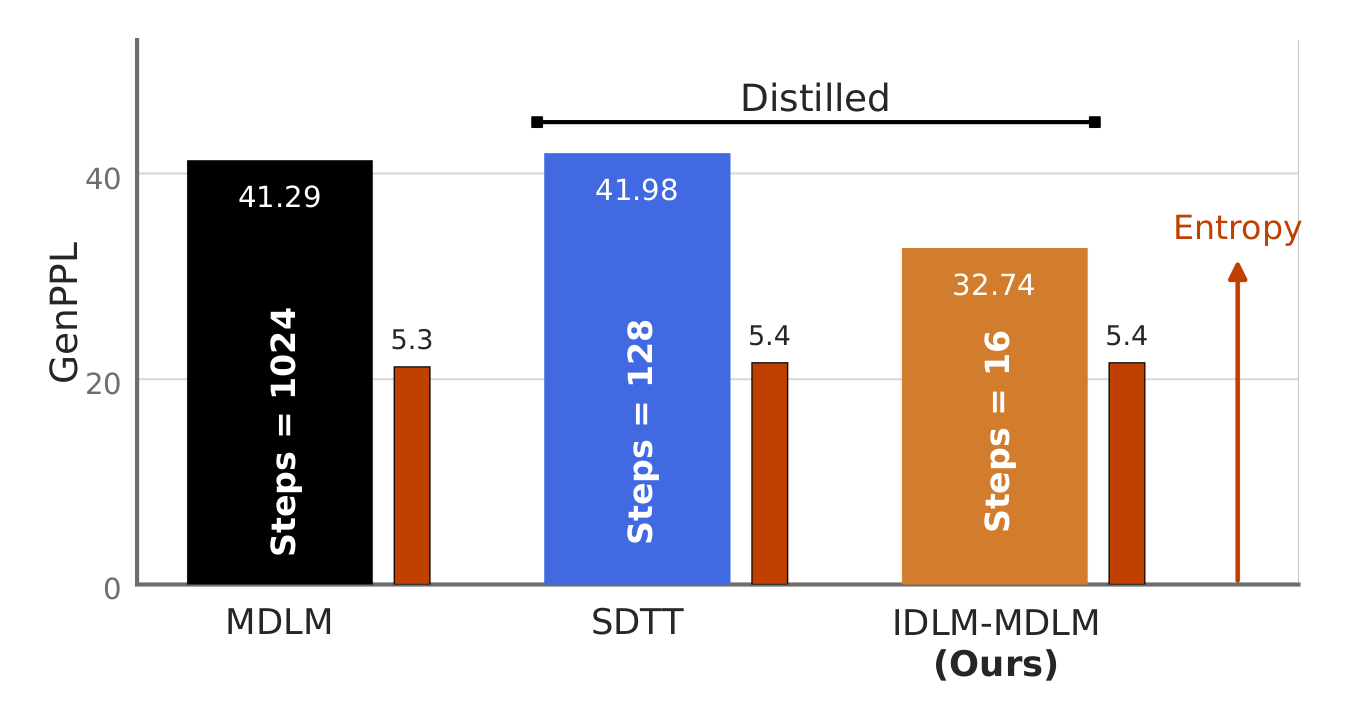}
    \end{minipage}\hfill
    \begin{minipage}[t]{0.47\textwidth}
      \centering
      \includegraphics[width=0.85\linewidth]{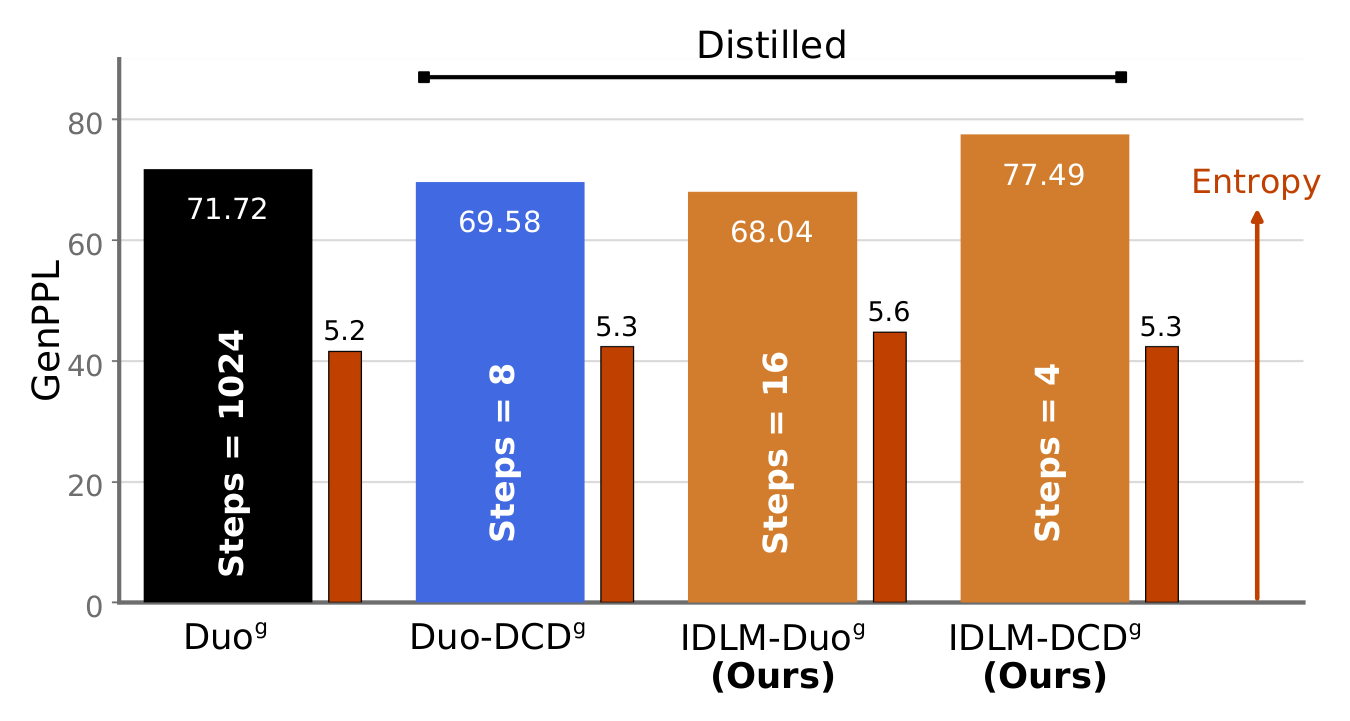}
    \end{minipage}
    \vspace{1mm}

    \begin{minipage}[t]{0.49\textwidth}
      \centering
      \textbf{Masked Diffusion}\\[0.5mm]
      \scriptsize
      \renewcommand{\arraystretch}{1.05}
      \setlength{\tabcolsep}{3pt}
      \resizebox{\linewidth}{!}{%
      \begin{tabular}{llccccc}
          \toprule
          && {\scriptsize Steps} & {\scriptsize GenPPL $\downarrow$} & {\scriptsize MAUVE $\uparrow$} & {\scriptsize Entropy $\uparrow$} & {\scriptsize GM $\downarrow$} \\
          \midrule
          & MDLM (Teacher)~\citep{sahoo2024simple} & 1024 & 41.29 & 0.89 $\pm$ 0.01 & 5.28 & 0.672 \\
          \midrule
          & SDTT~\citep{deschenaux2024beyond} & \multirow{5}{*}{32} & 41.45 & 0.91 $\pm$ 0.03 & \textbf{5.31} & -- \\
          & SDTT + Di4C$^2$~\citep{hayakawa2025distillation} & & 31.51 & \textbf{0.93} $\pm$ 0.02 & 5.28 & -- \\
          & DiDi-Instruct~\citep{zheng2026ultrafast} & & 24.97 & -- & 5.18 & -- \\
          & D-MMD~\citep{hoogeboom2026beyond} & & 72.1 & -- & \textbf{5.57} & \textbf{0.578} \\
          & IDLM-MDLM (\textbf{Ours}) & & \textbf{20.45} $\pm$ 1.88 & 0.91 $\pm$ 0.01 & 5.23 $\pm$ 0.06 & 0.629 \\
          \midrule
          & SDTT~\citep{deschenaux2024beyond} & \multirow{5}{*}{16} & 61.34 & 0.88 $\pm$ 0.03 & 5.36 & -- \\
          & SDTT + Di4C$^2$~\citep{hayakawa2025distillation} & & 44.82 & 0.90 $\pm$ 0.01 & 5.34 & -- \\
          & DiDi-Instruct~\citep{zheng2026ultrafast} & & 38.19 & -- & 5.21 & -- \\
          & D-MMD~\citep{hoogeboom2026beyond} & & 67.7 & -- & \textbf{5.57} & 0.558 \\
          & IDLM-MDLM (\textbf{Ours}) & & \textbf{32.75} $\pm$ 1.62 & \textbf{0.93} $\pm$ 0.01 & 5.42 $\pm$ 0.06 & \textbf{0.527} \\
          \midrule
          & SDTT~\citep{deschenaux2024beyond} & \multirow{4}{*}{8} & 121.31 & 0.89 $\pm$ 0.03 & 5.39 & -- \\
          & SDTT + Di4C$^2$~\citep{hayakawa2025distillation} & & 84.15 & 0.90 $\pm$ 0.04 & 5.38 & -- \\
          & DiDi-Instruct~\citep{zheng2026ultrafast} & & \textbf{62.24} & -- & 5.17 & -- \\
          & IDLM-MDLM (\textbf{Ours}) & & 79.75 $\pm$ 5.61 & \textbf{0.91} $\pm$ 0.01 & \textbf{5.61} $\pm$ 0.04 & \textbf{0.835} \\
          \bottomrule
      \end{tabular}}
    \end{minipage}\hfill
    \begin{minipage}[t]{0.49\textwidth}
      \centering
      \textbf{Uniform \& Continuous Diffusion}\\[0.5mm]
      \scriptsize
      \renewcommand{\arraystretch}{0.90}
      \setlength{\tabcolsep}{7pt}
      \resizebox{\linewidth}{!}{%
      \begin{tabular}{lcccc}
          \toprule
           & {\scriptsize Steps} & {\scriptsize GenPPL $\downarrow$} & {\scriptsize MAUVE $\uparrow$} & {\scriptsize Entropy $\uparrow$} \\
          \midrule
          FLM~\citep{lee2026flow} & \multirow{2}{*}{1024} & \textbf{62.23} & - & \textbf{5.33} \\
          Duo$^{\mathrm{g}}$ (Teacher)~\citep{sahoo2025the} &  & 71.72 & 0.90 $\pm$ 0.02 & 5.22 \\
          \midrule
          ELF~\citep{hu2026elf} & \multirow{4}{*}{32} & \textbf{24.08} $\pm$ 0.16 & -- & 5.14 \\
          FMLM~\citep{lee2026flow} &  & 45.09  & -- & 5.25 \\
          Duo-DCD$^{\mathrm{g}}$~\citep{sahoo2025the} & & 46.31 & \textbf{0.96} $\pm$ 0.01 & \textbf{5.38} \\
          IDLM-DCD$^{\mathrm{g}}$ (\textbf{Ours}) & & 39.34 $\pm$ 8.51 & 0.92 $\pm$ 0.01 & 5.35 $\pm$ 0.11 \\
          \midrule
          ELF~\citep{hu2026elf} & \multirow{4}{*}{16} & \textbf{33.66} $\pm$ 1.09 & -- & 5.16 \\
          FMLM~\citep{lee2026flow} &  & 63.63  & -- & 5.29 \\
          Duo-DCD$^{\mathrm{g}}$~\citep{sahoo2025the} & & 54.11 & \textbf{0.96} $\pm$ 0.01 & 5.37 \\
          IDLM-DCD$^{\mathrm{g}}$ (\textbf{Ours}) & & 43.36 $\pm$ 3.76 & 0.94 $\pm$ 0.01 & \textbf{5.41} $\pm$ 0.04 \\
          \midrule
          ELF~\citep{hu2026elf} & \multirow{4}{*}{8} & 67.32 $\pm$ 2.25 & -- & 5.14 \\
          FMLM~\citep{lee2026flow} &  & 86.5  & -- & 5.36  \\
          Duo-DCD$^{\mathrm{g}}$~\citep{sahoo2025the} & & 69.58 & 0.93 $\pm$ 0.01 & 5.30 \\
          IDLM-DCD$^{\mathrm{g}}$ (\textbf{Ours}) & & \textbf{53.61} $\pm$ 3.08 & \textbf{0.96} $\pm$ 0.01 & \textbf{5.41} $\pm$ 0.03 \\
          \midrule
          FMLM~\citep{lee2026flow} & \multirow{3}{*}{4} & 111.31 & -- & 5.26 \\
          Duo-DCD$^{\mathrm{g}}$~\citep{sahoo2025the} & & 96.24 & 0.69 $\pm$ 0.02 & 4.93 \\
          IDLM-DCD$^{\mathrm{g}}$ (\textbf{Ours}) & & \textbf{77.47} $\pm$ 4.43 & \textbf{0.89} $\pm$ 0.01 & \textbf{5.28} $\pm$ 0.06 \\
          \bottomrule
      \end{tabular}}
    \end{minipage}
    \caption{\textbf{OWT distillation.} IDLM preserves the quality/diversity tradeoff while using far fewer reverse steps for both masked and uniform diffusion. The strongest gains appear in the low-step regime.}
    \label{fig:owt-main-results}
    \vspace{-3mm}
\end{figure*}

\begin{figure}[t]
\centering
\includegraphics[width=\linewidth]{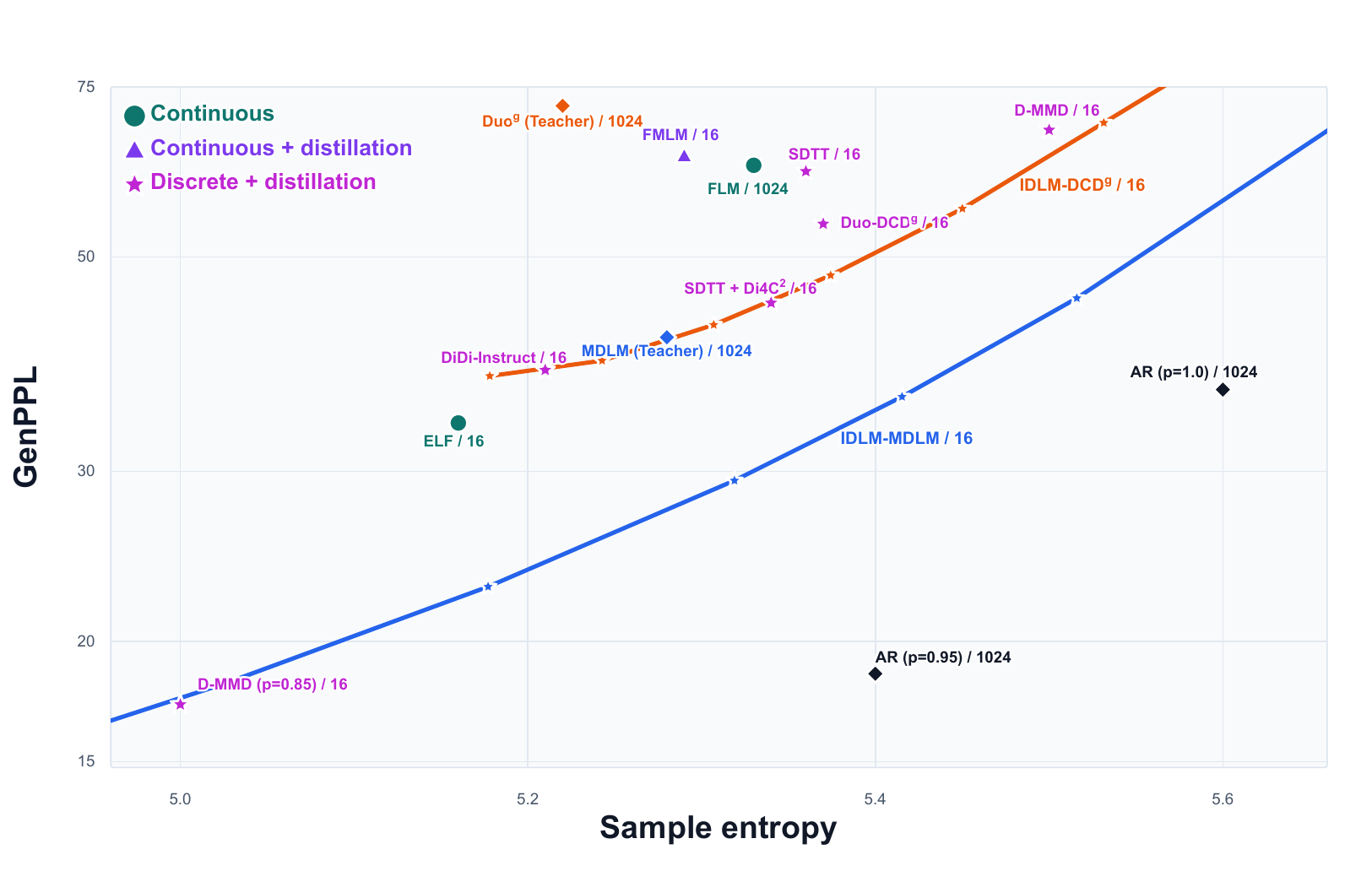}
\caption{\textbf{GenPPL--Entropy tradeoff on OWT.} Each point denotes a Method/Sampling steps. Lower GenPPL and higher entropy are preferred.}
\label{fig:owt-entropy-frontier}
\end{figure}

\begin{figure*}[t]
\centering
\begin{tcolorbox}[
    enhanced,
    colframe=owtgreen,
    colback=green!3!white,
    boxrule=0.65pt,
    arc=4pt,
    left=5pt,
    right=5pt,
    top=3pt,
    bottom=3pt,
    width=\textwidth,
    before upper={\setlength{\parindent}{0pt}}]
\noindent
\begin{minipage}[c]{0.04\linewidth}
\centering
\begin{tikzpicture}[baseline=-0.5ex, x=1cm, y=1cm]
    \filldraw[fill=owtgreen, draw=owtgreen, rounded corners=1.5pt] (0,0) rectangle (0.31,0.31);
    \draw[white, line width=0.95pt, line join=round, rounded corners=0.5pt]
        (0.08,0.10) -- (0.17,0.24) -- (0.25,0.20) -- (0.15,0.06) -- cycle;
    \draw[white, line width=0.8pt, line join=round]
        (0.08,0.10) -- (0.06,0.03) -- (0.15,0.06);
    \draw[white, line width=0.65pt, line cap=round] (0.18,0.23) -- (0.24,0.19);
    \draw[white, line width=0.85pt, line cap=round]
        (0.14,0.05) .. controls (0.18,0.01) and (0.22,0.06) .. (0.27,0.05);
\end{tikzpicture}
\end{minipage}%
\begin{minipage}[c]{0.94\linewidth}
{\normalsize\bfseries\textcolor{owtgreen}{Unconditional Generation (OpenWebText)}}
\end{minipage}

\vspace{-0.9mm}
\noindent\begin{tikzpicture}
\draw[owtgreen, densely dashed, line width=0.5pt] (0,0) -- (\linewidth,0);
\end{tikzpicture}
\vspace{-2.35mm}

\begin{minipage}[t]{0.485\linewidth}
{\footnotesize\bfseries\textcolor{owtgreen}{IDLM-MDLM / 16 Sampling steps}}\\[-0.2mm]
{\scriptsize\bfseries\textcolor{owtgreen}{Gen. PPL:} 34.1 \quad \textcolor{owtgreen}{H:} 5.43}\\[-0.1mm]
{\scriptsize
\ldots{} After this short presentation, let me briefly describe the layout of the grid. I initially intended to drag the page away from the phone screen and place it onto the grid platform. \ldots{}}
\end{minipage}
\hfill
\begin{minipage}[t]{0.485\linewidth}
{\footnotesize\bfseries\textcolor{owtgreen}{IDLM-DCD$^g$ / 4 Sampling steps}}\\[-0.2mm]
{\scriptsize\bfseries\textcolor{owtgreen}{Gen. PPL:} 77.5 \quad \textcolor{owtgreen}{H:} 5.28}\\[-0.1mm]
{\scriptsize
\ldots{} The thing, programs are a language, and structures are highly hierarchical. With the complexity of a programming language, the underlying structure grows, and you improve that structure. \ldots{}}
\end{minipage}
\end{tcolorbox}
\vspace{-1mm}
\begin{tcolorbox}[
    enhanced,
    colframe=tinyorange,
    colback=orange!3!white,
    boxrule=0.65pt,
    arc=4pt,
    left=5pt,
    right=5pt,
    top=3pt,
    bottom=3pt,
    width=\textwidth,
    before upper={\setlength{\parindent}{0pt}}]
\noindent
\begin{minipage}[c]{0.04\linewidth}
\centering
\begin{tikzpicture}[baseline=-0.5ex, x=1cm, y=1cm]
    \filldraw[fill=tinyorange, draw=tinyorange, rounded corners=1.5pt] (0,0) rectangle (0.31,0.31);
    \draw[white, line width=0.85pt, rounded corners=0.35pt]
        (0.085,0.055) -- (0.085,0.255) -- (0.205,0.255) -- (0.255,0.205) -- (0.255,0.055) -- cycle;
    \draw[white, line width=0.65pt] (0.205,0.255) -- (0.205,0.205) -- (0.255,0.205);
    \draw[white, line width=0.65pt, line cap=round] (0.12,0.175) -- (0.22,0.175);
    \draw[white, line width=0.65pt, line cap=round] (0.12,0.135) -- (0.22,0.135);
    \draw[white, line width=0.65pt, line cap=round] (0.12,0.095) -- (0.19,0.095);
\end{tikzpicture}
\end{minipage}%
\begin{minipage}[c]{0.94\linewidth}
{\normalsize\bfseries\textcolor{tinyorange}{Conditional Generation (TinyGSM)}}
\end{minipage}

\vspace{-0.9mm}
\noindent\begin{tikzpicture}
\draw[tinyorange, densely dashed, line width=0.5pt] (0,0) -- (\linewidth,0);
\end{tikzpicture}
\vspace{-2.35mm}

\begin{minipage}[t]{0.485\linewidth}
{\footnotesize\bfseries\textcolor{tinyorange}{IDLM-MDLM / 128 Sampling steps}}\\[-0.2mm]
{\scriptsize\bfseries\textcolor{tinyorange}{Accuracy:} 19.9\%}\\[-0.1mm]
{\scriptsize\bfseries Prompt.}
{\scriptsize A baseball coach buys 9 baseballs for \$3 each, and a basketball coach buys 8 basketballs for \$14 each. How much more did the basketball coach spend?}\\[-0.1mm]
{\scriptsize\bfseries Generation.}\\[-2.4mm]
\begin{lstlisting}[language=Python,basicstyle=\ttfamily\tiny,columns=fullflexible,keepspaces=true,breaklines=true,aboveskip=0pt,belowskip=0pt,xleftmargin=0.08\linewidth,xrightmargin=0.08\linewidth]
def simple_math_problem() -> int:
    baseball_cost = 9 * 3
    basketball_cost = 8 * 14
    difference = basketball_cost - baseball_cost
    result = difference
    return result
\end{lstlisting}
\end{minipage}
\hfill
\begin{minipage}[t]{0.485\linewidth}
{\footnotesize\bfseries\textcolor{tinyorange}{IDLM-Duo / 64 Sampling steps}}\\[-0.2mm]
{\scriptsize\bfseries\textcolor{tinyorange}{Accuracy:} 19.0\%}\\[-0.1mm]
{\scriptsize\bfseries Prompt.}
{\scriptsize The red car is 40\% cheaper than the blue car. The price of the blue car is \$100. How much do both cars cost?}\\[-0.1mm]
{\scriptsize\bfseries Generation.}\\[-2.4mm]
\begin{lstlisting}[language=Python,basicstyle=\ttfamily\tiny,columns=fullflexible,keepspaces=true,breaklines=true,aboveskip=0pt,belowskip=0pt,xleftmargin=0.08\linewidth,xrightmargin=0.08\linewidth]
def simple_math_problem() -> int:
    blue_car = 100
    red_car = blue_car * 0.60
    total_cost = blue_car + red_car
    result = total_cost
    return result
\end{lstlisting}
\end{minipage}
\end{tcolorbox}
\caption{\textbf{Qualitative examples.} Representative IDLM samples. Additional qualitative samples are provided in Appendix~\ref{app: qualitative samples}.}
\label{fig:owt-qualitative}
\vspace{-4mm}
\end{figure*}

\vspace{-2mm}
\section{Related Works}
\label{sec:related-works}
    \ecoparagraph{Continuous diffusion distillation.}
    Continuous distillation has two main branches: consistency-like methods~\citep{song2023consistency,song2023improved}, which train students to jump along teacher trajectories, and distribution-matching methods~\citep{yin2024one,yin2024improved,zhou2024score,huang2024flow,gushchin2025inverse,kornilov2025universal, selikhanovych2025one},
    which directly optimize the student distribution, usually with an auxiliary or
    fake model. IDLM follows the distribution-matching view and extends it to
    discrete diffusion.

    \ecoparagraph{Inverse distillation.}
    Inverse distillation directly optimizes a student generator to approximate the
    target distribution, placing it within the family of distribution-matching
    distillation methods. This idea was introduced by IBMD for diffusion bridge models~\citep{gushchin2025inverse}, then
    generalized to continuous-space flow and diffusion
    models in UID framework~\citep{kornilov2025universal}. RSD further extended this view to
    discrete-time diffusion and clarified its relation to DMD: a DMD-like update
    appears after a stop-gradient reduction~\citep{selikhanovych2025one}. IDLM
    adopts the same inverse view, but extends it to discrete space.

    \ecoparagraph{Consistency-style distillation for DLMs.}
    The consistency branch has discrete analogues such as
    SDTT~\citep{deschenaux2024beyond} and Duo-DCD~\citep{sahoo2025the}. However,
    the objective mainly teaches trajectory skipping. In practice, this preserves a
    conditional-independence factorization of the teacher's denoising process, which
    can miss the full joint distribution and favor common modes. Thus, especially in
    the one-step limit, these methods lack guaranteeing recovery of the data
    distribution and remain largely empirical. In contrast, IDLM represents the
    student as a sequence-level mixture, see (\wasyparagraph\ref{sec:loss intuition}), and Theorem~\ref{thm:unique solution} shows that its ideal
    objective is minimized only by the data distribution.
    \nocite{consistency2024rep, abulkhanov2026unifying}

    \ecoparagraph{Distribution matching in discrete space.}
    Discrete distribution-matching methods such as
    Di[M]O~\citep{zhu_dimathttmo_2025}, D-MMD~\citep{hoogeboom2026beyond}, and
    DiDi-Instruct~\citep{zheng2026ultrafast} match teacher and auxiliary signals at
    sampled noising times, corresponding to KL matching of noised marginals,
    \[
        \mathcal L_{\mathrm{DMD}}(\theta)
        =
        \int_0^1 w(t)\,
        D_{\mathrm{KL}}\!\left(p_t^\theta\,\|\,p_t^*\right)\,dt .
    \]
    Here $p_t^\theta$ and $p_t^*$ are the student and target noised marginal
    distributions at time $t$.
    IDLM instead matches full diffusion path measures,
    \[
        \mathcal L_{\mathrm{IDLM}}(\theta)
        =
        D_{\mathrm{KL}}\!\left(\mathbb P^\theta\,\|\,\mathbb P^*\right),
    \]
    so the comparison is trajectory-level rather than marginal-only. A DMD-like
    update appears only after a stop-gradient reduction~\citep{selikhanovych2025one},
    so in masked diffusion, because of the \textsc{subs} parameterization, the
    marginal and path-based losses coincide. In uniform diffusion, this distinction
    matters in practice, as our experiments show.

    \ecoparagraph{Continuous-space language flows.}
    A parallel line changes the state space: FLM/FMLM~\citep{lee2026flow} use
    one-hot flows, while ELF~\citep{hu2026elf} uses embedding-space flows. These
    models can obtain strong GenPPL, but their Entropy can drop relative to other
    baselines. This matters because lower GenPPL may trade off Entropy, for example
    when sweeping the sampling temperature. Moreover, FLM underperforms on recent
    benchmarks~\citep{deschenaux2026language}. Thus, it remains unclear whether a
    continuous or discrete state space is preferable for language modeling.

\vspace{-2mm}
\section{Experiments}
\label{sec:exps}

\reviewchange{This section demonstrates the applicability of our IDLM distillation method across several types of DLMs. We conduct the main experiments for MDLM and Duo. Although we focus on these two models, our theoretical formulation also covers the general SEDD setting. Therefore, we report an additional IDLM-SEDD evaluation in Appendix~\ref{app:sedd-experiments}. Representative generations are shown in Figure~\ref{fig:owt-qualitative}. Further implementation details, additional experimental results, and uncurated text samples are provided in Appendices~\ref{app: experimental details},~\ref{app:additional-experimental-results}, and~\ref{app: qualitative samples}, respectively.}

\subsection{Unconditional Generation on OpenWebText}
\label{sec:openwebtext-results}
We use the suffix to indicate the teacher: IDLM-MDLM distills MDLM, IDLM-Duo distills the original Duo model, and IDLM-DCD distills the Duo-DCD model. We describe the calculation of all evaluation metrics used in this section in Appendix~\ref{app:metric-definitions}.

\ecoparagraph{Teacher acceleration.}
    Figure~\ref{fig:owt-main-results} summarizes how well IDLM distills teachers. For masked diffusion MDLM, IDLM-MDLM reduces the $1024$-step teacher to $16$ reverse steps, giving a ($64\times$) acceleration while preserving GenPPL and Entropy metrics. For Duo with the greedy sampler (\wasyparagraph\ref{sec:uniform process duo}), IDLM-Duo also reaches $16$ steps ($64\times$), while IDLM-DCD$^{\mathrm{g}}$ reaches $4$ steps ($256\times$). The corresponding ancestral Duo results are reported in Appendix~\ref{app:ancestral-sampling-results}, where IDLM-DCD$^{\mathrm{a}}$ is evaluated down to $4$ steps and gives its best speed-quality tradeoff at $8$ steps ($128\times$). Also, GenPPL and Entropy vs sampling-step curves are provided in Figure~\ref{fig:owt-genppl-steps}.

\ecoparagraph{Comparison with baselines.}
    The lower panels of Figure~\ref{fig:owt-main-results} compare IDLM against other baselines with GenPPL, MAUVE, Entropy, and GM where available. For absorbing-state distillation, IDLM-MDLM gives the best overall low-step tradeoff at $16$ steps, improving GenPPL while preserving diversity. For uniform-state distillation, IDLM-DCD$^{\mathrm{g}}$ is strongest in the $8$--$4$ step regime. The ancestral step curve in Figure~\ref{fig:owt-genppl-steps} and the detailed results in Table~\ref{tab:uniform-ancestral-expanded} show the same trend.

\ecoparagraph{GenPPL--Entropy tradeoff.}
    Recent evaluations of flow and diffusion language models show that GenPPL should be interpreted together with sample entropy: lowering entropy can substantially improve GenPPL without necessarily improving the underlying generative distribution. Figure~\ref{fig:owt-entropy-frontier} visualizes this tradeoff on OWT. IDLM-MDLM improves GenPPL at comparable entropy to the MDLM teacher, and IDLM-DCD remains competitive.

\ecoparagraph{Scaling.}
    We additionally test IDLM on an $0.9$B MDLM checkpoint from SDTT~\citep{deschenaux2024beyond}, the full scaling comparison with SDTT is reported in Table~\ref{tab:mdlm-scaling}. The same pattern persists at larger scale: IDLM improves the low-step GenPPL/Entropy tradeoff over SDTT.

\ecoparagraph{Ablation study.}\phantomsection\label{sec:ablation}
    We ablate the main IDLM design choices in Appendix~\ref{app:ablation-results}. The results indicate that the original masked-diffusion update gives the best early GenPPL--Entropy tradeoff, while the full IDLM-Duo update is more stable than its stop-gradient variant. Adding simple Gaussian input noise to IDLM-MDLM does not improve performance (Table~\ref{tab:gaussian-noise-control}).

\subsection{Conditional Generation on TinyGSM}
\label{sec:tinygsm}

GenPPL and entropy measure unconditional sample quality, but they do not directly test whether a model preserves sequence-level correctness. We therefore follow the TinyGSM/GSM8K protocol used by~\citet{deschenaux2026language}. The results in Table~\ref{tab:tinygsm} show that IDLM improves steadily with more reverse steps and recovers teacher-level conditional generation with far fewer sampling steps. IDLM-MDLM reaches teacher-level accuracy $19.86\%$ with only $128$ steps, matching and slightly surpassing the $1024$-step MDLM teacher at $18.0\%$, corresponding to an ($8\times$) reduction in sampling steps. For Duo, IDLM-Duo already reaches comparable accuracy $19.03\%$ at $64$ steps relative to the $1024$-step Duo teacher at $17.2\%$, giving a ($16\times$) reduction. Overall, IDLM retains conditional generation with far fewer steps.

\par\noindent
\begin{minipage}{\columnwidth}
\vspace{-1mm}
\captionsetup{type=table}
\centering
\captionof{table}{\textbf{Conditional generation on GSM8K.} Baseline accuracies are reproduced from~\citet{deschenaux2026language}.}
\label{tab:tinygsm}
\scriptsize
\setlength{\tabcolsep}{5pt}
\renewcommand{\arraystretch}{0.94}
\begin{tabular}{lcc}
\toprule
Model & Steps & Accuracy (\%) \\
\midrule
\multicolumn{3}{l}{\textit{Autoregressive}} \\
Sample & \multirow{2}{*}{512} & 53.9 \\
Greedy & & \textbf{63.3} \\
\midrule
\multicolumn{3}{l}{\textit{Diffusion}} \\
MDLM (Teacher)~\citep{sahoo2024simple} & \multirow{4}{*}{1024} & \textbf{18.0} \\
Duo (Teacher)~\citep{sahoo2025the} & & 17.2 \\
FLM~\citep{lee2026flow} & & 0.3 \\
S-FLM~\citep{deschenaux2026language} & & \textbf{18.0} \\
\midrule
\multicolumn{3}{l}{\textit{Distillation}} \\
IDLM-MDLM (\textbf{Ours}) & \multirow{2}{*}{128} & 19.86 \\
IDLM-Duo (\textbf{Ours}) & & \textbf{21.38} \\
\midrule
IDLM-MDLM (\textbf{Ours}) & \multirow{2}{*}{64} & 14.94 \\
IDLM-Duo (\textbf{Ours}) & & \textbf{19.03} \\
\midrule
IDLM-MDLM (\textbf{Ours}) & \multirow{2}{*}{32} & 12.81 \\
IDLM-Duo (\textbf{Ours}) & & \textbf{15.39} \\
\bottomrule
\end{tabular}
\vspace{-1mm}
\end{minipage}
\par

\color{debugcolor}
\section{Discussion and Conclusion}
\label{sec:conclusion}
    We extend Inverse Distillation to pretrained Diffusion Language Models (DLMs), deriving the IDLM loss (\wasyparagraph\ref{sec:theoretical extension idlm}) and practical discrete relaxations for non-smooth token spaces (\wasyparagraph\ref{sec:practical extension of idlm}). Experiments (\wasyparagraph\ref{sec:exps}) show that IDLM accelerates teacher sampling by $4\times$--$64\times$ across DLM families while preserving generation quality.
\subsection{Limitations and Future Work}
\label{sec:limitations}
    Our experiments use relatively small-scale DLMs; a natural next step is to test whether IDLM's few-step gains persist for larger models such as LLaDA~\citep{nie2025large} and on standardized downstream benchmarks.
\color{black}

\phantomsection
\section*{Acknowledgements}
\label{sec:acknowledgements}
  The work was supported by the grant for research centers in the field of AI provided by the Ministry of Economic Development of the Russian Federation in accordance with the agreement 000000C313925P4F0002 and the agreement №139-10-2025-033.


\phantomsection
\section*{Impact Statement}
\label{sec:impact-statement}
  This paper presents work whose goal is to advance the field of Machine
  Learning. There are many potential societal consequences of our work, none
  which we feel must be specifically highlighted here.

\nocite{kushnareva2024ai}

\phantomsection
\label{sec:references}
\bibliography{bibliography}
\bibliographystyle{icml2026}

\newpage
\appendix
\onecolumn

\section*{Contents}
\newcommand{\contentsheading}[1]{%
  \vspace{0.8mm}\noindent{\bfseries #1}\par\vspace{0.2mm}%
}
\newcommand{\contentssectionline}[2]{%
  \noindent\hyperref[#1]{\textbf{\ref*{#1}.}\quad #2}\dotfill\pageref*{#1}\par%
}
\newcommand{\contentssubsectionline}[2]{%
  \noindent\hspace*{4mm}\hyperref[#1]{\ref*{#1}\quad #2}\dotfill\pageref*{#1}\par%
}
\newcommand{\contentssubsubsectionline}[2]{%
  \noindent\hspace*{8mm}\hyperref[#1]{\ref*{#1}\quad #2}\dotfill\pageref*{#1}\par%
}
\newcommand{\contentsappendixline}[2]{%
  \noindent\hyperref[#1]{\textbf{Appendix~\ref*{#1}.}\quad #2}\dotfill\pageref*{#1}\par%
}
\newcommand{\contentsplainline}[2]{%
  \noindent\hyperref[#1]{#2}\dotfill\pageref*{#1}\par%
}
\begingroup
\large
\setlength{\parskip}{3pt}
\contentsheading{Main Paper}
\contentssectionline{sec:introduction}{Introduction}
\contentssectionline{sec:preliminaries}{Preliminaries}
\contentssubsectionline{sec:discrete diffusion models}{Diffusion Language Models}
\contentssubsubsectionline{sec:absorbing process mdlm}{Absorbing Process}
\contentssubsubsectionline{sec:uniform process duo}{Uniform Process}
\contentssubsubsectionline{sec:unified notations sequence loss}{Unified Notations \& Sequence Level Loss}
\contentssubsectionline{sec:inverse distillation}{Inverse Distillation}
\contentssectionline{sec:idlm}{Inverse-distilled Diffusion Language Models}
\contentssubsectionline{sec:theoretical extension idlm}{Theoretical Extension of Inverse Distillation}
\contentssubsectionline{sec:practical extension of idlm}{Addressing Discrete Nonsmoothness}
\contentssubsectionline{sec:loss intuition}{Loss Intuition}
\contentssubsectionline{sec:technical aspects}{Technical Aspects}
\contentssectionline{sec:related-works}{Related Works}
\contentssectionline{sec:exps}{Experiments}
\contentssubsectionline{sec:openwebtext-results}{Unconditional Generation on OpenWebText}
\contentssubsectionline{sec:tinygsm}{Conditional Generation on TinyGSM}
\contentssectionline{sec:conclusion}{Discussion and Conclusion}
\contentssubsectionline{sec:limitations}{Limitations and Future Work}
\contentsheading{Appendix}
\contentsappendixline{app:details of considered processes}{SEDD Formulation, Theory, and Experiments}
\contentssubsectionline{app:sedd-formulation-theory}{SEDD Formulation and Theory}
\contentssubsectionline{app:sedd-reverse-distillation}{Inverse Distillation for SEDD}
\contentssubsectionline{app:sedd-experiments}{SEDD Experimental Validation}
\contentsappendixline{app:absorbing process}{Masked Diffusion}
\contentssubsectionline{app:mdlm-formulation-theory}{Masked Diffusion Formulation and Theory}
\contentssubsectionline{app:mdlm-inverse-distillation}{Inverse Distillation for Masked Diffusion}
\contentsappendixline{app:uniform process}{Uniform Diffusion}
\contentssubsectionline{app:uniform-formulation-theory}{Uniform Diffusion Formulation and Theory}
\contentssubsectionline{app:uniform-inverse-distillation}{Inverse Distillation for Uniform Diffusion}
\contentsappendixline{app: proofs}{Proofs}
\contentssubsectionline{app: connection with the main text}{Connection with the Main Text}
\contentssubsectionline{app: proof of the theorem}{Proof of the Theorem}
\contentsappendixline{app: experimental details}{Experimental Details}
\contentssubsectionline{app:owt-experimental-details}{Unconditional Generation on OpenWebText}
\contentssubsectionline{app:tinygsm-experimental-details}{Conditional Generation on TinyGSM}
\contentsappendixline{app:additional-experimental-results}{Additional Experimental Results}
\contentssubsectionline{app:ancestral-sampling-results}{Ancestral Sampling Results}
\contentssubsectionline{app:genppl-step-plots}{GenPPL and Entropy vs. Sampling Steps}
\contentssubsectionline{app:mdlm-scaling}{Scaling to Larger MDLM}
\contentssubsectionline{app:ablation-results}{Ablation Study}
\contentsappendixline{app: qualitative samples}{Qualitative Samples}
\contentssubsectionline{app:owt-qualitative-samples}{Unconditional Generation (OpenWebText)}
\contentssubsectionline{app:tinygsm-qualitative-samples}{Conditional Generation (TinyGSM)}
\endgroup
\clearpage

\section{\texorpdfstring{SEDD Formulation, Theory, and Experiments}{SEDD Formulation, Theory, and Experiments}}
\label{app:details of considered processes}

This appendix collects the material needed to use SEDD within the IDLM framework. Appendix~\ref{app:sedd-formulation-theory} gives the CTMC score-entropy formulation and its sequence factorization. Appendix~\ref{app:sedd-reverse-distillation} specializes the IDLM objective to SEDD and explains the resulting teacher--fake advantage. Appendix~\ref{app:sedd-experiments} reports the SEDD experimental validation. We use the notation of Section~\ref{sec:discrete diffusion models}: $x_0,x_t,y\in\VC$ are one-hot tokens, $p^*$ is the data distribution, $p_\theta$ is the student distribution, and $q_t(x_t\mid x_0)$ is the forward noising kernel. For a matrix $A$, write $\langle u,v\rangle_A\vcentcolon=u^\top A v$; when the subscript is omitted, $\langle u,v\rangle=u^\top v$.

\subsection{\texorpdfstring{SEDD Formulation and Theory}{SEDD Formulation and Theory}}
\label{app:sedd-formulation-theory}

\paragraph{Forward process.}
As stated in Section~\ref{sec:discrete diffusion models}, the main text focuses on absorbing and uniform diffusion and defers the general CTMC formulation to this appendix. A SEDD forward process defines noised data marginals $p_t\in\Delta$ through
\begin{equation}
    \frac{dp_t}{dt} = Q_t p_t, \quad p_0 = p^*,
    \label{eq:forward diffusion}
\end{equation}
where each diffusion matrix $Q_t \in \RR^{N \times N}$ has nonnegative off-diagonal entries and columns summing to zero. The off-diagonal entries are transition rates between tokens. The zero column-sum condition preserves total probability in Eq.~\eqref{eq:forward diffusion}, because probability mass that leaves a token is balanced by mass assigned to outgoing transitions. In the common parameterization $Q_t=\sigma_t Q$, the fixed matrix $Q$ determines the diffusion process, including which token-to-token transitions are allowed, while $\sigma_t$ controls the noise rate. Let $\bar{\sigma}_t \vcentcolon= \int_0^t \sigma_s ds$. Then $\exp(\bar{\sigma}_t Q)$ is the matrix exponential of $\bar{\sigma}_t Q$, and for a fixed clean token $x_0$ we write
\begin{equation}
    q_t(x_t \mid x_0)
    \vcentcolon= \mathrm{Cat}\bigl(x_t;\exp(\bar{\sigma}_t Q)x_0\bigr).
    \label{eq:conditional_distribution}
\end{equation}
The corresponding noised data marginal is $p_t(x_t)=\int q_t(x_t\mid x_0)\,p^*(dx_0)$. Absorbing and uniform diffusion are specific choices of $Q$ and are described in Appendices~\ref{app:absorbing process} and~\ref{app:uniform process}.

\paragraph{Training loss.}
In Section~\ref{sec:unified notations sequence loss}, SEDD is one instance of the unified token-level objective $\LC(f,p)$. SEDD~\citep{lou2024discrete} learns a positive score model $f\colon \VC\times[0,1]\to\RR_{+}^{N}$ that estimates reverse-process probability ratios. For a fixed clean token $x_0$, define the conditional score ratio
\[
    \langle s(x_t,t\mid x_0),y\rangle
    \vcentcolon=
    \frac{q_t(y\mid x_0)}{q_t(x_t\mid x_0)}.
\]
For any clean-token distribution $p$, the score-entropy objective has the same distribution-level form as the MDLM and Duo losses:
\begin{equation}
\label{eq:sedd loss}
\begin{aligned}
\LC_{\mathrm{SEDD}}(f,p)
&\vcentcolon=
\EE_{p(x_0),\,t,\,q_t}
\bigl[
    g_{\mathrm{SEDD}}\bigl(x_t,x_0,f(x_t,t)\bigr)
\bigr],
\\
g_{\mathrm{SEDD}}\bigl(x_t,x_0,f(x_t,t)\bigr)
&\vcentcolon=
\sum_{y \neq x_t} \langle x_t,y\rangle_{Q_t}
\biggl(
\langle f(x_t,t),y\rangle
- \langle s(x_t,t\mid x_0),y\rangle
\log \langle f(x_t,t),y\rangle
\biggr).
\end{aligned}
\end{equation}
Here $\langle x_t,y\rangle_{Q_t}$ weights the loss by the forward transition rate from $y$ to $x_t$, so transitions not allowed by $Q_t$ do not contribute. For a generic training distribution $p$, we use the same notation $p_t$ for the noised marginal induced by $p$:
\[
    p_t(x)=\int q_t(x\mid x_0)\,p(dx_0).
\]
At a minimizer of Eq.~\eqref{eq:sedd loss}, the model satisfies
\begin{equation}
\label{eq:sedd-generic-optimum}
    \langle f(x_t,t),y\rangle
    =
    \EE_{p(x_0\mid x_t)}
    \bigl[
    \langle s(x_t,t\mid x_0),y\rangle
    \bigr]
    =
    \frac{p_t(y)}{p_t(x_t)}, \qquad y\neq x_t .
\end{equation}
In particular, when $p=p^*$ we denote the trained teacher by $f^*$, which gives
\begin{equation}
    \langle f^{*}(x_t,t),y\rangle
    =
    \frac{p_t(y)}{p_t(x_t)}, \qquad y\neq x_t .
    \label{eq:sedd optimal ratio}
\end{equation}
Thus, the trained SEDD model provides the marginal ratios required by the reverse process.

\paragraph{Sampling procedure.}
Sampling uses the learned score model to instantiate a reverse-time CTMC. The reverse process starts from the tractable terminal distribution $\pi$ and evolves backward in time:
\begin{equation}
    \frac{dp_{1-t}}{dt} = \overline{Q}_{1-t}^{f} p_{1-t}, \quad p_1 = \pi.
    \label{eq:backward diffusion}
\end{equation}
For $y\neq x_t$, the reverse transition rate induced by $f$ is
\begin{equation}
\label{eq:sedd-reverse-rates}
    \langle y,x_t\rangle_{\overline Q_t^f}
    \vcentcolon=
    \langle f(x_t,t),y\rangle
    \langle x_t,y\rangle_{Q_t},
    \qquad
    \langle x_t,x_t\rangle_{\overline Q_t^f}
    =-
    \sum_{y\neq x_t}
    \langle y,x_t\rangle_{\overline Q_t^f}.
\end{equation}
With the teacher $f^*$, Eq.~\eqref{eq:sedd-reverse-rates} becomes the exact reverse rate
\[
    \langle y,x_t\rangle_{\overline Q_t^{f^*}}
    =
    \frac{p_t(y)}{p_t(x_t)}
    \langle x_t,y\rangle_{Q_t}.
\]
Therefore, SEDD learns the probability ratios that convert the forward CTMC rates into reverse CTMC rates. Equivalently, for a small reverse step $h>0$, these rates define the transition probabilities
\[
    \PP(X_{t-h}=y\mid X_t=x_t)
    =
    \delta_{y,x_t}
    +
    h\langle y,x_t\rangle_{\overline Q_t^f}
    +
    o(h),
\]
where $\delta_{y,x_t}$ is the Kronecker delta. For $y\neq x_t$, this reduces to
\[
    \PP(X_{t-h}=y\mid X_t=x_t)
    =
    h\langle y,x_t\rangle_{\overline Q_t^f}
    +
    o(h),
\]
while the diagonal term gives the probability of staying at $x_t$ up to order $h$. Since $X_{t-h}$ is a discrete random vector, its expected one-step displacement is
\[
    \begin{aligned}
    \EE[X_{t-h}-x_t\mid X_t=x_t]
    =
    \sum_{y\in\VC}
    (y-x_t)\PP(X_{t-h}=y\mid X_t=x_t)
    =
    \sum_{y\neq x_t}
    (y-x_t)\PP(X_{t-h}=y\mid X_t=x_t),
    \end{aligned}
\]
where the two sums are equal because the $y=x_t$ term is zero.
Substituting the small-step probabilities yields
\[
\begin{aligned}
    \EE[X_{t-h}-x_t\mid X_t=x_t]
    &=
    h\sum_{y\neq x_t}
    (y-x_t)
    \langle y,x_t\rangle_{\overline Q_t^f}
    +
    o(h).
\end{aligned}
\]
The delta and diagonal terms do not contribute to the displacement because they correspond to $y=x_t$, where $y-x_t=0$. Thus the reverse CTMC rates determine both the probabilities of individual jumps and the average local reverse update.
We use this quantity below to interpret IDLM-SEDD. For a fixed corrupted token $x_t$, a SEDD model does not choose a single denoised token directly; it assigns rates to possible reverse jumps $x_t\to y$. The expectation $\EE[X_{t-h}-x_t\mid X_t=x_t]$ compresses these local jump rates into the average one-step denoising motion induced by the model. Thus, it gives a compact summary of how the model locally moves $x_t$ toward cleaner tokens; related local-move views are used in discrete flow matching~\citep{gat2024discrete}.

\paragraph{Loss for sequences.}
For language modeling, directly comparing all token pairs in a full sequence is expensive for large vocabularies. SEDD avoids this by using a coordinate-wise CTMC: each transition changes one position, while the forward conditional distribution factorizes across positions,
\[
    q_t(x_t^{1:L}\mid x_0^{1:L})
    =
    \prod_{\ell=1}^{L}q_t(x_t^\ell\mid x_0^\ell).
\]
Consequently, the conditional score ratio also factorizes. If $y^{1:L}$ differs from $x_t^{1:L}$ only at position $\ell$, then all unchanged-position factors cancel:
\[
\begin{aligned}
    \frac{q_t(y^{1:L}\mid x_0^{1:L})}
         {q_t(x_t^{1:L}\mid x_0^{1:L})}
    &=
    \prod_{j=1}^{L}
    \frac{q_t(y^j\mid x_0^j)}
         {q_t(x_t^j\mid x_0^j)}
    =
    \frac{q_t(y^\ell\mid x_0^\ell)}
         {q_t(x_t^\ell\mid x_0^\ell)},
    \qquad y^j=x_t^j\ \text{for }j\neq \ell .
\end{aligned}
\]
The sequence-level objective follows the unified notation of Section~\ref{sec:unified notations sequence loss}:
\[
    \LC_{\mathrm{SEDD}}^{1:L}(f,p)
    =
    \EE_{p(x_0^{1:L}),\,t,\,q_t}
    \biggl[
    \sum_{\ell=1}^{L}
    g_{\mathrm{SEDD}}
    \bigl(x_t^\ell,x_0^\ell,f^\ell(x_t^{1:L},t)\bigr)
    \biggr],
\]
where the score at position $\ell$ may condition on the full noised sequence. This factorization makes the general CTMC formulation practical for sequence models.

\subsection{\texorpdfstring{Inverse Distillation for SEDD}{Inverse Distillation for SEDD}}
\label{app:sedd-reverse-distillation}

We now specialize the IDLM objective from Section~\ref{sec:idlm} to the SEDD loss above. The teacher $f^*$ is fixed and estimates reverse ratios for the data distribution. The fake model $f$ is trained on the current student distribution $p_\theta$ and estimates reverse ratios for student samples. IDLM-SEDD updates the student by comparing these two SEDD losses on the same student-generated examples.

\paragraph{Student samples.}
The expectations in $\LC_{\mathrm{SEDD}}(f,p_\theta)$ are taken over
\[
    x_0\sim p_\theta(x_0),
    \qquad
    t\sim\UC[0,1],
    \qquad
    x_t\sim q_t(\cdot\mid x_0).
\]
Equivalently, when $p_\theta$ is induced by the generator used in Section~\ref{sec:technical aspects}, we may write $x_0=G_\theta(\epsilon)$. The real data distribution affects this update only through the fixed teacher $f^*$; the fake model and the student are evaluated on samples from $p_\theta$.

\paragraph{IDLM-SEDD objective.}
Using Eq.~\eqref{eq:inverse discrete diffusion distillation}, the inverse-distillation objective for SEDD is
\begin{equation}
\label{eq:idlm-sedd-compact}
    \LC_{\mathrm{IDLM\text{-}SEDD}}(\theta)
    \vcentcolon=
    \LC_{\mathrm{SEDD}}(f^*,p_\theta)
    -
    \min_f \LC_{\mathrm{SEDD}}(f,p_\theta).
\end{equation}
In practice, as in the main IDLM algorithm, we alternate between fitting the fake model and updating the student. With $\theta$ fixed, the fake model minimizes $\LC_{\mathrm{SEDD}}(f,p_\theta)$. With $f$ fixed, the student minimizes
\begin{equation}
\label{eq:idlm-sedd-fixed-compact}
    \LC_{\mathrm{IDLM\text{-}SEDD}}(\theta)
    \vcentcolon=
    \LC_{\mathrm{SEDD}}(f^*,p_\theta)
    -
    \LC_{\mathrm{SEDD}}(f,p_\theta).
\end{equation}
Substituting Eq.~\eqref{eq:sedd loss} into both terms gives
\begin{equation}
\label{eq:idlm-sedd-expanded}
\begin{aligned}
\LC_{\mathrm{IDLM\text{-}SEDD}}(\theta)
&=
\EE_{p_\theta(x_0),\,t,\,q_t}
\biggl[
\sum_{y\neq x_t}
\langle x_t,y\rangle_{Q_t}
\biggl(
\langle f^*(x_t,t),y\rangle
-
\langle s(x_t,t\mid x_0),y\rangle
\log\langle f^*(x_t,t),y\rangle
\biggr)
\biggr]
\\
&\quad-
\EE_{p_\theta(x_0),\,t,\,q_t}
\biggl[
\sum_{y\neq x_t}
\langle x_t,y\rangle_{Q_t}
\biggl(
\langle f(x_t,t),y\rangle
-
\langle s(x_t,t\mid x_0),y\rangle
\log\langle f(x_t,t),y\rangle
\biggr)
\biggr].
\end{aligned}
\end{equation}
Since both expectations use the same $p_\theta(x_0)$, $t$, and $q_t(x_t\mid x_0)$, this simplifies to
\begin{equation}
\label{eq:idlm-sedd-simplified}
\begin{aligned}
\LC_{\mathrm{IDLM\text{-}SEDD}}(\theta)
&=
\EE_{p_\theta(x_0),\,t,\,q_t}
\biggl[
\sum_{y\neq x_t}
\langle x_t,y\rangle_{Q_t}
\biggl(
\langle f^*(x_t,t)-f(x_t,t),y\rangle
-
\langle s(x_t,t\mid x_0),y\rangle
\log
\frac{\langle f^*(x_t,t),y\rangle}{\langle f(x_t,t),y\rangle}
\biggr)
\biggr].
\end{aligned}
\end{equation}
Equation~\eqref{eq:idlm-sedd-simplified} is the fully substituted IDLM-SEDD loss with fixed fake model $f$.

\paragraph{Advantage intuition.}
Following the local-advantage notation of Section~\ref{sec:loss intuition}, for a sampled pair $(x_0,x_t)$ define
\[
\begin{aligned}
    a_t
    &\vcentcolon=
    g_{\mathrm{SEDD}}\bigl(x_t,x_0,f^*(x_t,t)\bigr)
    -
    g_{\mathrm{SEDD}}\bigl(x_t,x_0,f(x_t,t)\bigr)
    \\
    &=
    \sum_{y\neq x_t}
    \langle x_t,y\rangle_{Q_t}
    \biggl(
    \langle f^*(x_t,t)-f(x_t,t),y\rangle
    -
    \langle s(x_t,t\mid x_0),y\rangle
    \log
    \frac{\langle f^*(x_t,t),y\rangle}
         {\langle f(x_t,t),y\rangle}
    \biggr).
\end{aligned}
\]
In SEDD, $a_t$ compares teacher and fake local score-entropy losses on the same corrupted student sample. The comparison is made over the same jumps that define the reverse CTMC motion in Appendix~\ref{app:sedd-formulation-theory}: each candidate jump $x_t\to y$ contributes according to its forward rate $\langle x_t,y\rangle_{Q_t}$ and the teacher--fake difference in the score assigned to that jump. Equivalently, the teacher $f^*$ and the fake model $f$ induce two average local reverse updates,
\[
    d_{f^*}(x_t,t)
    =
    \sum_{y\neq x_t}
    (y-x_t)\langle y,x_t\rangle_{\overline Q_t^{f^*}},
    \qquad
    d_f(x_t,t)
    =
    \sum_{y\neq x_t}
    (y-x_t)\langle y,x_t\rangle_{\overline Q_t^f}.
\]
The quantity $a_t$ should be read as a local teacher-minus-fake loss, not as a Euclidean distance between $d_{f^*}$ and $d_f$. Since IDLM-SEDD minimizes $a_t$, the generator is encouraged to produce samples $x_0$ whose corrupted versions $x_t$ are better explained by the teacher reverse CTMC than by the fake reverse CTMC. In jump terms, a low teacher loss means that the teacher assigns plausible rates to reverse moves $x_t\to y$ that are compatible with denoising $x_t$ toward $x_0$; equivalently, its average local update $d_{f^*}(x_t,t)$ gives a useful teacher signal. The fake model contributes the opposite side of the comparison, so the student is pushed away from samples whose local reverse jumps are explained only by the fake model.

The sign of $a_t=g_{\mathrm{SEDD}}(f^*)-g_{\mathrm{SEDD}}(f)$ gives the corner cases. If $a_t<0$, then the teacher has lower local SEDD loss than the fake model: the teacher likes the reverse jumps around this corrupted sample more, and the generator receives an attractive signal for such samples. If $a_t>0$, then the fake model has lower local loss: the sample is better explained by the current student-trained dynamics than by the teacher, so minimizing the objective discourages the generator from producing it. If $a_t\approx0$, the two models give nearly the same local explanation, and the relative teacher-versus-fake signal is weak.

At the sequence level, we use the factorized SEDD loss from Appendix~\ref{app:sedd-formulation-theory}:
\[
    \LC_{\mathrm{IDLM\text{-}SEDD}}^{1:L}(\theta)
    \vcentcolon=
    \LC_{\mathrm{SEDD}}^{1:L}(f^*,p_\theta)
    -
    \min_f \LC_{\mathrm{SEDD}}^{1:L}(f,p_\theta).
\]

\subsection{\texorpdfstring{SEDD Experimental Validation}{SEDD Experimental Validation}}
\label{app:sedd-experiments}

\begin{figure}[!htbp]
  \centering
  \begin{minipage}[t]{0.48\textwidth}
    \centering
    \includegraphics[width=\linewidth]{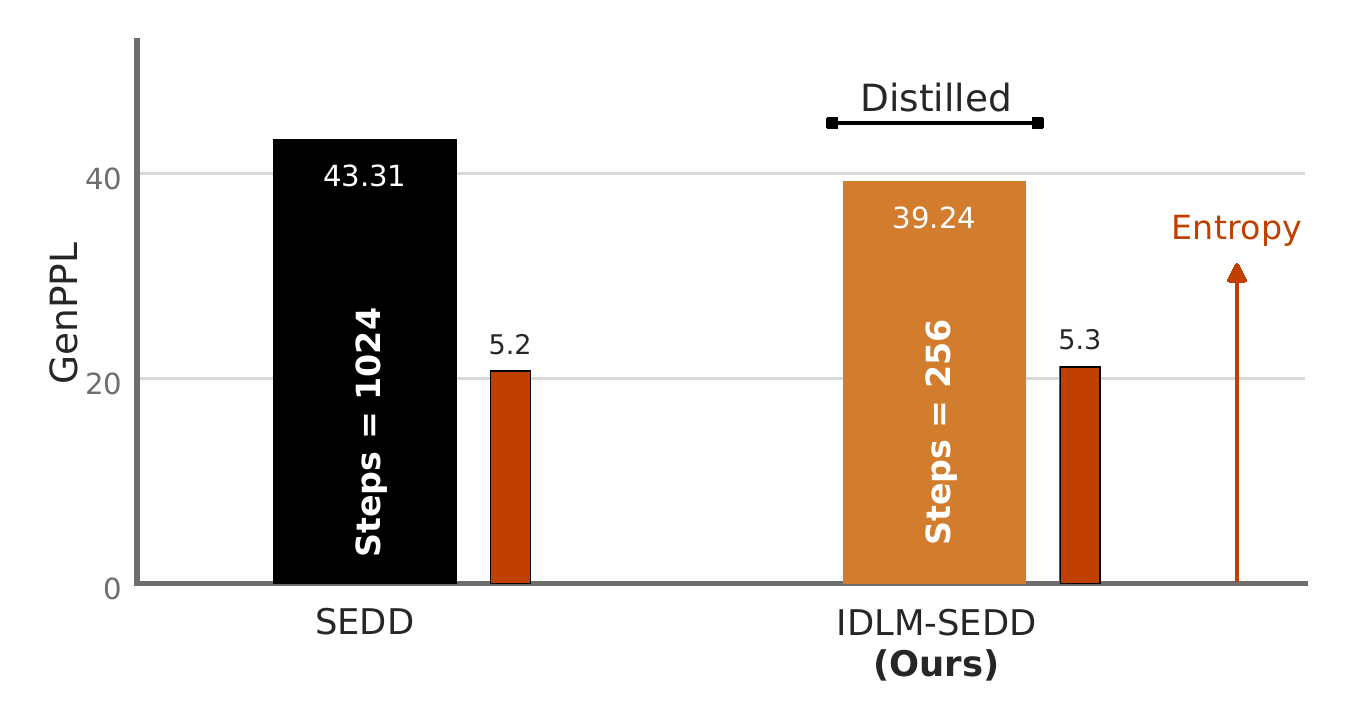}
  \end{minipage}
  \hfill
  \begin{minipage}[t]{0.48\textwidth}
    \centering
    \includegraphics[width=\linewidth]{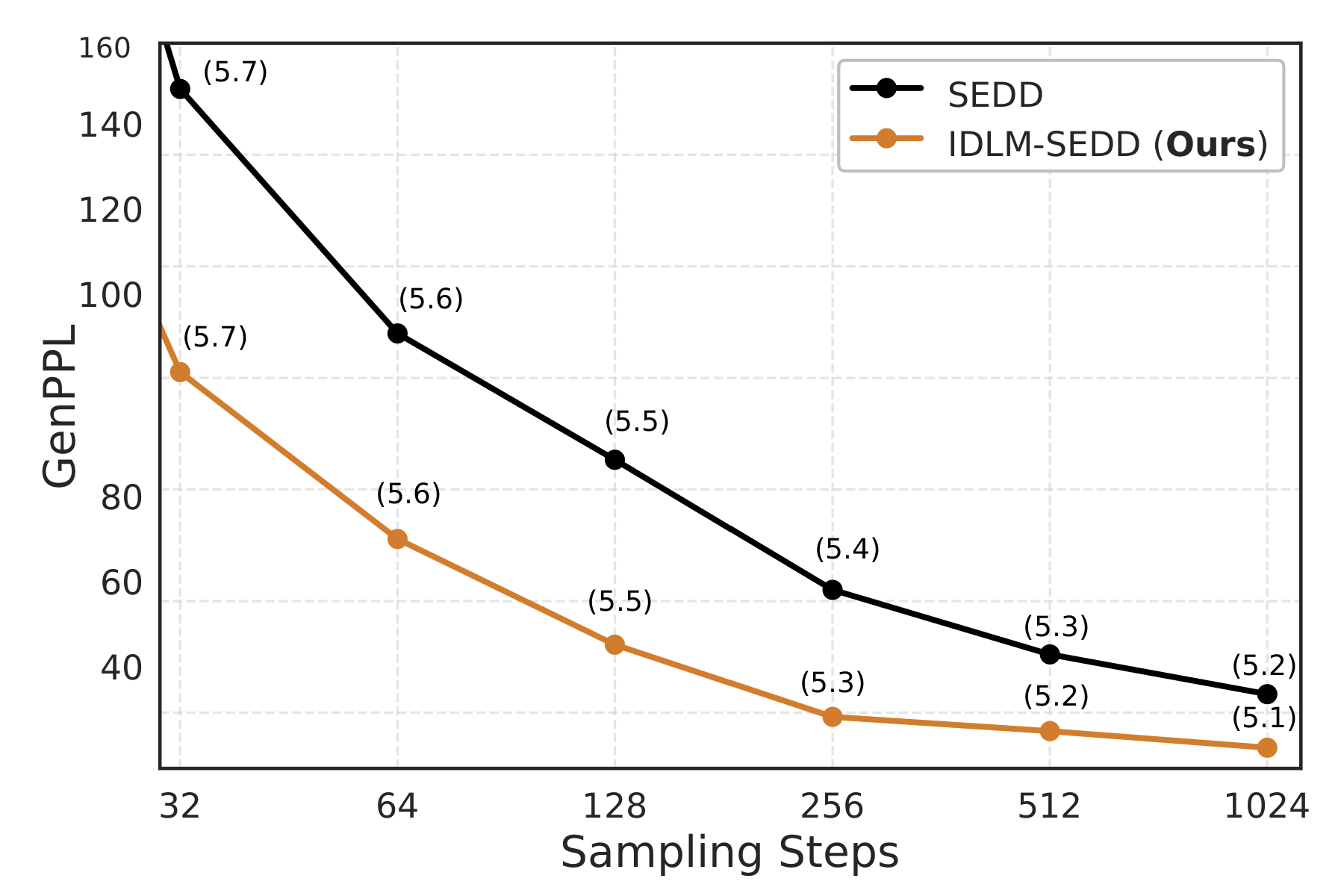}
  \end{minipage}
  \vspace{-1mm}
  \caption{\small\textbf{IDLM-SEDD comparison with SEDD.} \textit{Left:} IDLM-SEDD matches SEDD generation quality and diversity while reducing sampling from $1024$ to $256$ steps ($4\times$). \textit{Right:} IDLM-SEDD improves GenPPL across sampling steps while preserving high entropy.}
  \label{fig:sedd-results}
  \vspace{-3mm}
\end{figure}

We use SEDD to evaluate IDLM under the general CTMC score-entropy formulation. Since only checkpoints for the absorbing-process SEDD variant are publicly available, we distill this variant and refer to it simply as SEDD for clarity.

As shown in Figure~\ref{fig:sedd-results}, IDLM-SEDD improves GenPPL over SEDD across the evaluated sampling-step range while preserving high sequence entropy at low step counts. In particular, it achieves a $4\times$ reduction in sampling steps, accelerating generation from $1024$ to $256$ steps while maintaining both GenPPL and output diversity. Uncurated samples are provided in Appendix~\ref{app: qualitative samples}.

\section{\texorpdfstring{\reviewchange{Masked Diffusion}}{Masked Diffusion}}
\label{app:absorbing process}

\begin{reviewblock}
This appendix expands the masked diffusion process used by MDLM. The section
has two parts. First, in Appendix~\ref{app:mdlm-formulation-theory}, we
describe the MDLM forward process, training objective, and reverse sampler.
Second, in Appendix~\ref{app:mdlm-inverse-distillation}, we explain why this
process gives a particularly simple inverse distillation signal. Masked
diffusion is a process-specific instance of the general continuous-time Markov
chain (CTMC) formulation in
Appendix~\ref{app:details of considered processes}, where ordinary tokens move
to the mask state.

\subsection{\texorpdfstring{Masked Diffusion Formulation and Theory}{Masked Diffusion Formulation and Theory}}
\label{app:mdlm-formulation-theory}

Masked diffusion corrupts text by replacing clean tokens with a special mask
token. Reverse sampling then gradually reveals the masked positions. This
process is also called absorbing diffusion because, once a token reaches the
mask state in the forward process, it stays there.

\paragraph{Forward process.}
Let $m\in\VC$ be the mask token. We use column-vector
conventions throughout: probability vectors are columns, and each diffusion
matrix has columns that sum to zero. The absorbing process sends every
non-mask token to $m$ and leaves $m$ fixed. Hence its terminal distribution is
the point mass $\pi=m$.

The base generator of this continuous-time Markov chain is
\begin{equation}
    Q_{\text{abs}} =
    \begin{bmatrix}
        -1 & 0 & \cdots & 0 & 0 \\
        0 & -1 & \cdots & 0 & 0 \\
        \vdots & \vdots & \ddots & \vdots & \vdots \\
        0 & 0 & \cdots & -1 & 0 \\
        1 & 1 & \cdots & 1 & 0
    \end{bmatrix}.
\end{equation}
With time-dependent rate $\sigma_t$, the forward generator is
\begin{equation}
    \label{eq: absorbing diffusion matrix}
    Q_t = \sigma_t Q_{\text{abs}} =
    \begin{bmatrix}
        -\sigma_t & 0 & \cdots & 0 & 0 \\
        0 & -\sigma_t & \cdots & 0 & 0 \\
        \vdots & \vdots & \ddots & \vdots & \vdots \\
        0 & 0 & \cdots & -\sigma_t & 0 \\
        \sigma_t & \sigma_t & \cdots & \sigma_t & 0
    \end{bmatrix}.
\end{equation}
If $\bar{\sigma}_t\vcentcolon=\int_0^t\sigma_sds$, then the cumulative
transition matrix is
\begin{equation}
    \exp(\bar{\sigma}_t Q_{\text{abs}}) =
    \begin{bmatrix}
        \exp(-\bar{\sigma}_t) & 0 & \cdots & 0 & 0 \\
        0 & \exp(-\bar{\sigma}_t) & \cdots & 0 & 0 \\
        \vdots & \vdots & \ddots & \vdots & \vdots \\
        0 & 0 & \cdots & \exp(-\bar{\sigma}_t) & 0 \\
        1 - \exp(-\bar{\sigma}_t) & 1 - \exp(-\bar{\sigma}_t) & \cdots & 1 - \exp(-\bar{\sigma}_t) & 1
    \end{bmatrix}.
\end{equation}
Equivalently, with $\alpha_t\vcentcolon=\exp(-\bar{\sigma}_t)$, a clean token
$x_0$ is preserved with probability $\alpha_t$ and absorbed into the mask with
probability $1-\alpha_t$:
\[
    q_t(x_t\mid x_0)
    =
    \mathrm{Cat}\bigl(x_t;\alpha_t x_0 + (1-\alpha_t)m\bigr).
\]
Thus, at any time $t$, a token has only two possible forward states: it is
still equal to the original clean token, or it has become the mask token.
Figure~\ref{fig:masked-forward-kernel} shows the same kernel visually. The
important point is that the masked kernel is linear in the clean-token
distribution: it can be applied either to a hard one-hot token or to the
simplex-valued generator output used later in IDLM.

\begin{figure*}[t]
    \centering
    \begin{subfigure}[b]{0.48\textwidth}
        \centering
        \includegraphics[width=\linewidth]{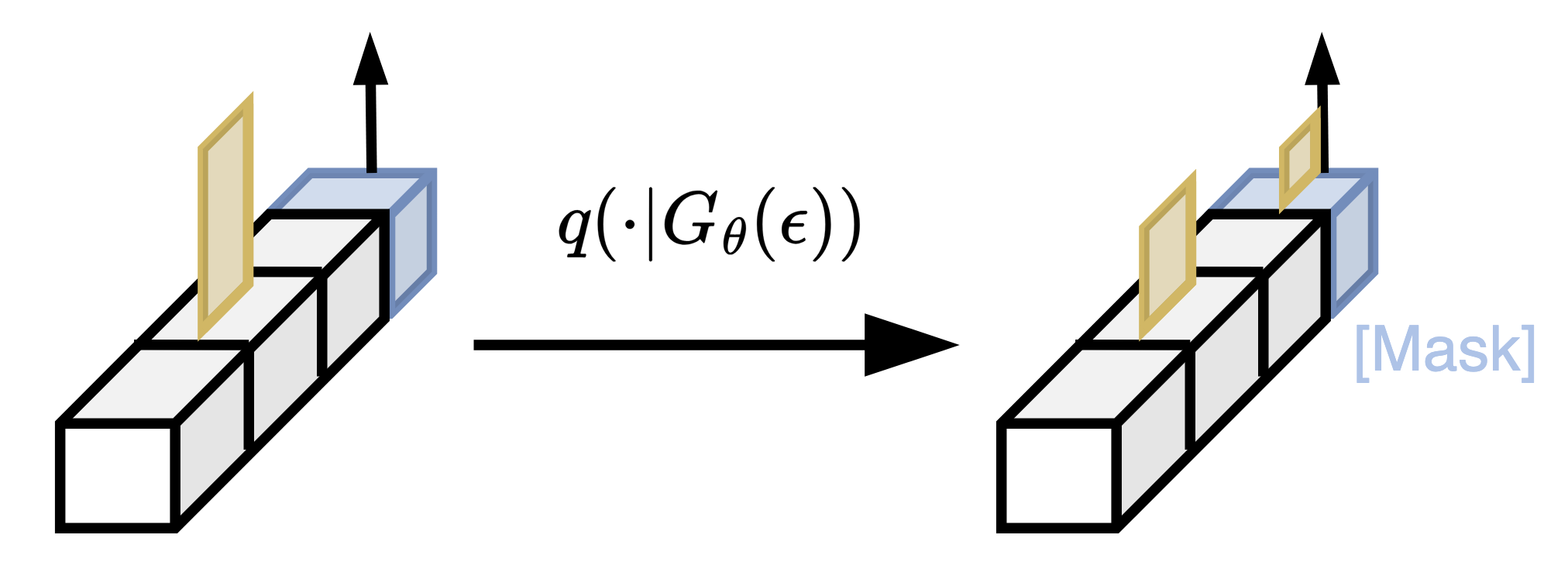}
        \caption{Hard one-hot input.}
    \end{subfigure}
    \hfill
    \begin{subfigure}[b]{0.48\textwidth}
        \centering
        \includegraphics[width=\linewidth]{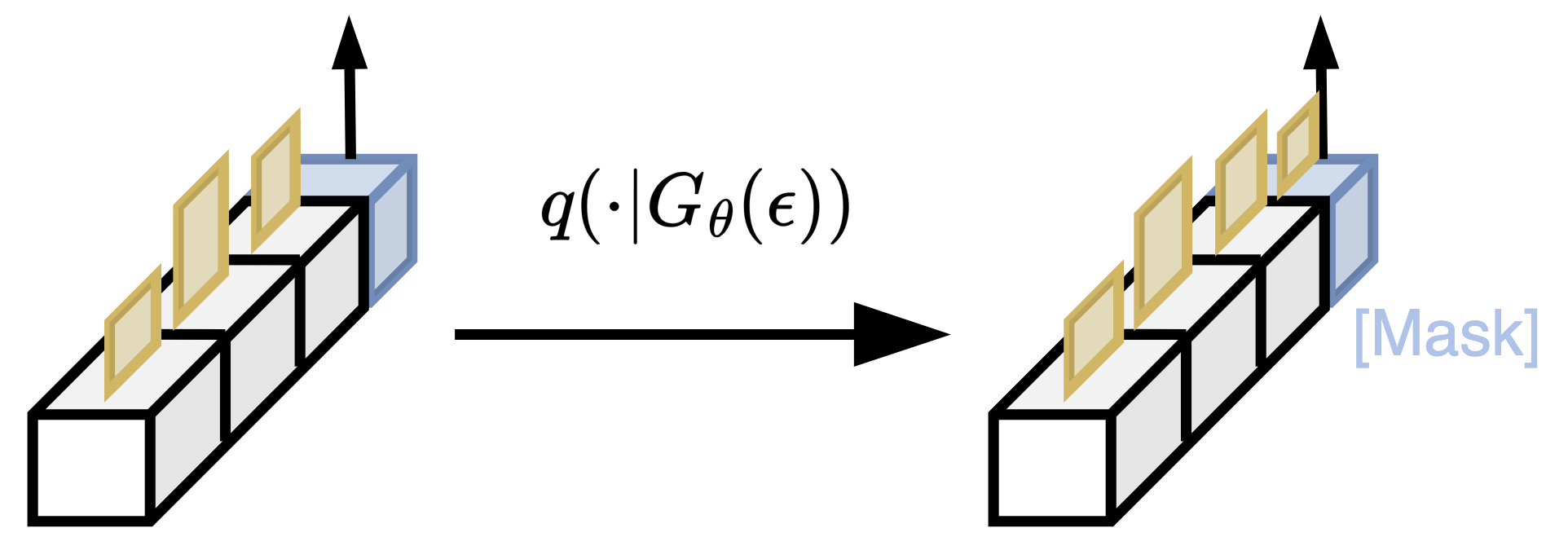}
        \caption{Simplex-valued input.}
    \end{subfigure}
    \caption{Masked forward process. The transition
    $q_t(\cdot\mid x_0)$ preserves the clean-token mass with probability
    $\alpha_t$ and moves the remaining mass to the mask token. Because this map
    is linear, the same operation applies to a one-hot token and to a
    differentiable simplex output $G_\theta(\epsilon)$.}
    \label{fig:masked-forward-kernel}
\end{figure*}

\paragraph{Training loss.}
MDLM~\citep{sahoo2024simple, shi2024simplified} uses this absorbing process
with a clean-token predictor $f(x_t,t)\in\Delta$. Training is weighted
masked-token prediction: when a position has been absorbed into the mask, the
model predicts its original clean token.

\paragraph{Sampling procedure.}
Sampling starts from the terminal distribution, i.e. from a fully masked
sequence, and then integrates the learned reverse process toward $t=0$.
For $0\leq s<t\leq1$, let $f_t\vcentcolon=f(x_t,t)$. MDLM uses the absorbing
posterior with the true clean token replaced by the model prediction:
\[
\begin{aligned}
p^{f}_{s\mid t}(x_s\mid x_t)
&=
\begin{cases}
    \mathrm{Cat}(x_s; x_t), & x_t\neq m,\\[2pt]
    \mathrm{Cat}\!\left(x_s;\dfrac{1-\alpha_s}{1-\alpha_t}m
    + \dfrac{\alpha_s-\alpha_t}{1-\alpha_t}f_t\right), & x_t=m .
\end{cases}
\end{aligned}
\]
Figure~\ref{fig:masked-diffusion-sampling-intuition} visualizes this reverse
sampling path. The two cases simply say that masked diffusion is one-way: a
revealed token stays revealed, while a masked token can either stay masked or be
revealed. This is the \textsc{subs} sampling rule used by MDLM. If the token has
already been revealed, \textsc{subs} copies it unchanged. If the token is still
masked, the reverse step either keeps it masked until time $s$ or reveals it
according to the clean-token prediction $f_t$. The noise schedule controls how
many tokens are still masked at each time. The model controls which token is
chosen when a masked position is revealed.

\begin{figure*}[t]
    \centering
    \includegraphics[width=0.95\textwidth]{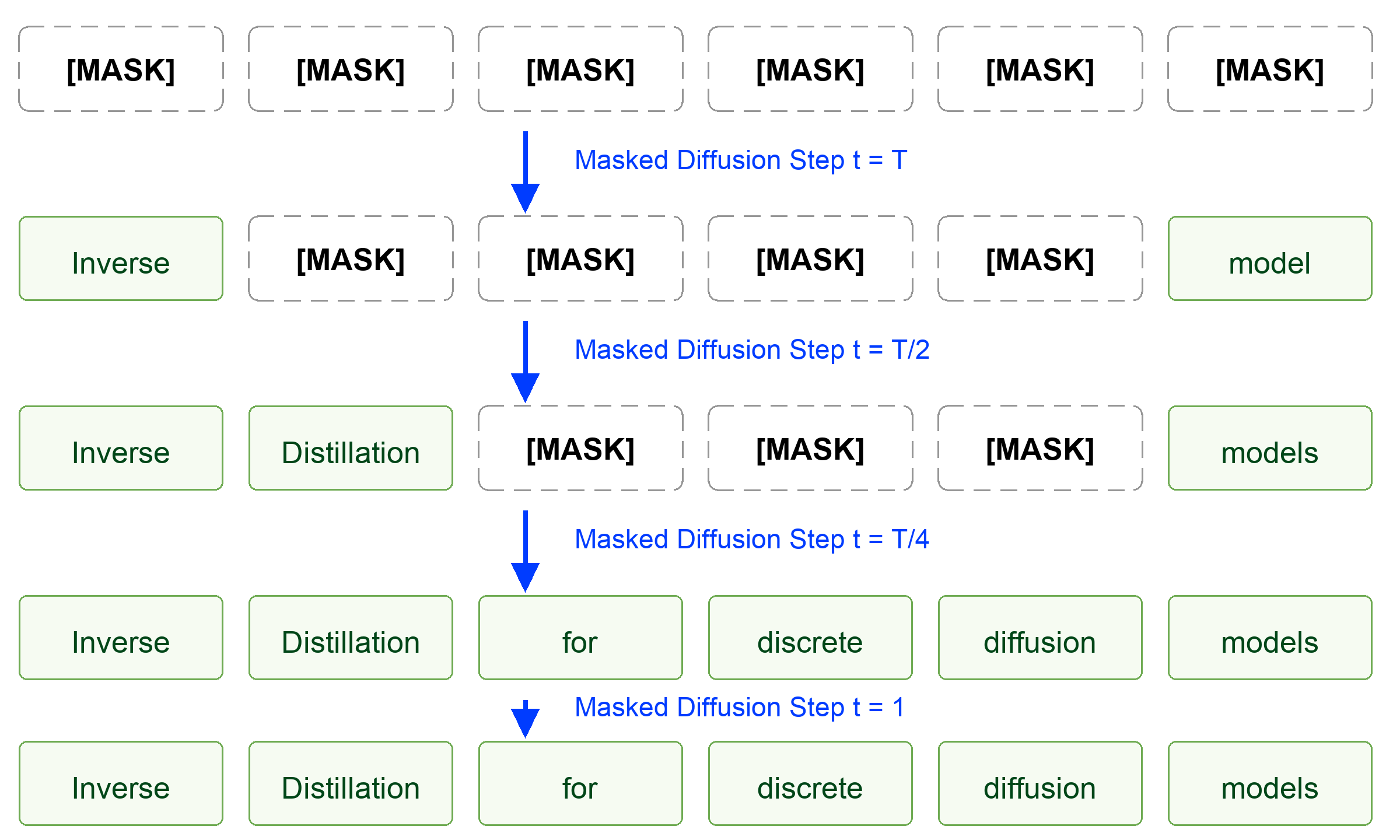}
    \caption{Masked-diffusion sampling procedure. The sampler
    starts from a
    fully masked sequence and reveals each position once; after a token is
    revealed, later reverse steps copy it. Thus the schedule decides when a
    position is revealed, and the model decides which token to put there.}
    \label{fig:masked-diffusion-sampling-intuition}
\end{figure*}

This is the main contrast with uniform diffusion. In masked diffusion, revealed
tokens stay fixed. In uniform diffusion, tokens can still change later, so the
IDLM signal for uniform diffusion must track intermediate states; see
Appendix~\ref{app:uniform process} for the corresponding uniform-diffusion
formulation.

\subsection{\texorpdfstring{Inverse Distillation for Masked Diffusion}{Inverse Distillation for Masked Diffusion}}
\label{app:mdlm-inverse-distillation}

We now connect the masked diffusion process to the IDLM objective. The general
inverse distillation loss, introduced in
Section~\ref{sec:inverse distillation}, asks for a student distribution
$p_\theta$ on which the pretrained teacher $f^*$ is better than any diffusion
model trained only on student samples:
\[
    \LC_{\mathrm{IDLM}}(\theta)
    =
    \LC(f^*,p_\theta)-\min_{\tilde f}\LC(\tilde f,p_\theta).
\]
In practice, the inner minimizer is represented by a fake model $f$ trained
on the current student distribution, as in the alternating optimization in
Section~\ref{sec:loss intuition}. Therefore, during a generator step we use the
teacher-minus-fake objective
\[
    \LC_{\mathrm{IDLM}}(\theta)
    =
    \LC(f^*,p_\theta)-\LC(f,p_\theta).
\]
For masked diffusion, this gap has an especially transparent form because of
the \textsc{subs} parameterization.

\paragraph{Simplex relaxation.}
The student generator should ultimately choose vocabulary tokens, but a hard
one-hot output blocks ordinary backpropagation. We therefore allow the
generator output to lie in the simplex, following
Section~\ref{sec:simplex relaxation},
\[
    G_\theta(\epsilon)\in\Delta .
\]
This is still a valid token distribution: its entries are nonnegative and sum
to one. Moreover, the masked forward kernel remains well-defined, as illustrated
in Figure~\ref{fig:masked-forward-kernel}, because
\[
    \alpha_t G_\theta(\epsilon)+(1-\alpha_t)m\in\Delta .
\]
Thus the generator update can use a weighted sum over possible clean tokens
instead of committing to one sampled token before gradients are computed.
Figure~\ref{fig:mdlm-simplex-relaxation} shows the relaxation: the discrete
target is a one-hot token, but during optimization the generator can expose
continuous token weights.

\begin{figure*}[t]
    \centering
    \begin{subfigure}[b]{0.18\textwidth}
        \centering
        \includegraphics[width=\linewidth]{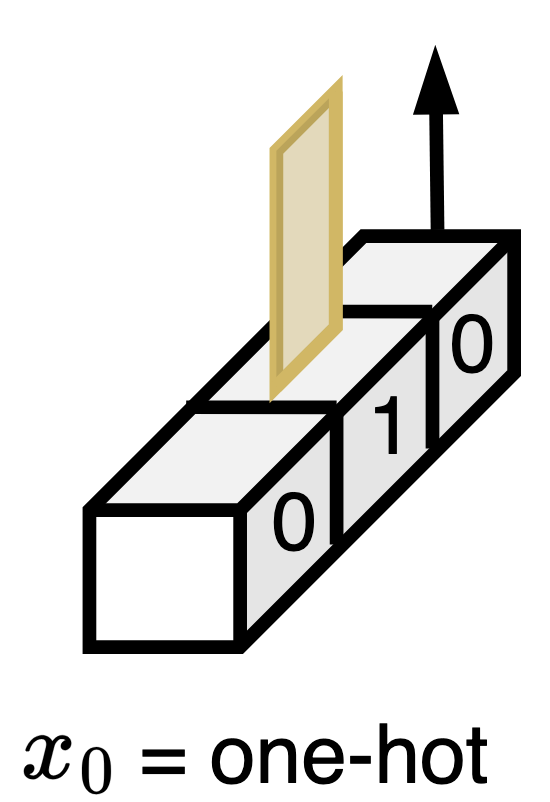}
        \caption{Discrete token.}
    \end{subfigure}
    \hfill
    \begin{subfigure}[b]{0.39\textwidth}
        \centering
        \includegraphics[width=\linewidth]{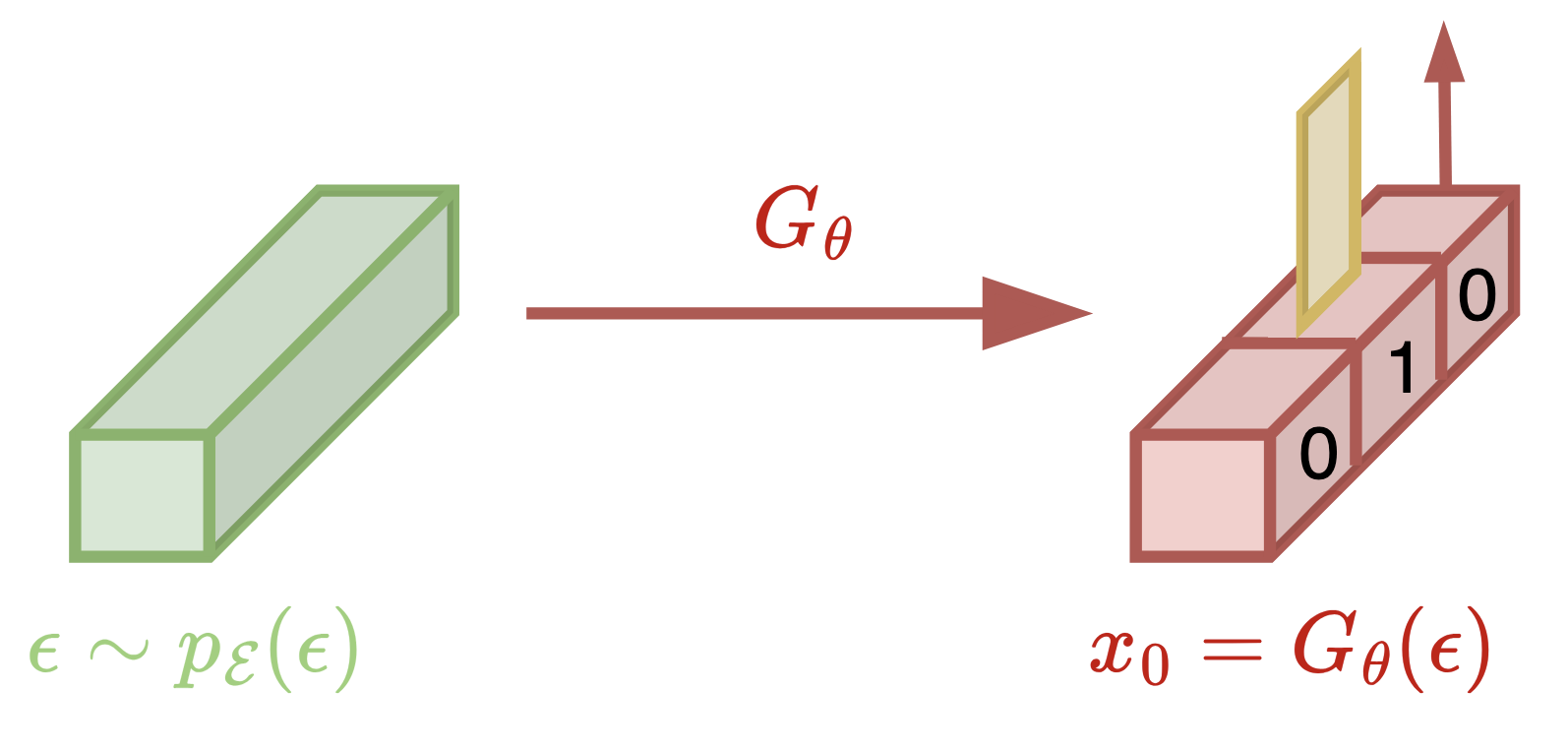}
        \caption{Hard generator output.}
    \end{subfigure}
    \hfill
    \begin{subfigure}[b]{0.39\textwidth}
        \centering
        \includegraphics[width=\linewidth]{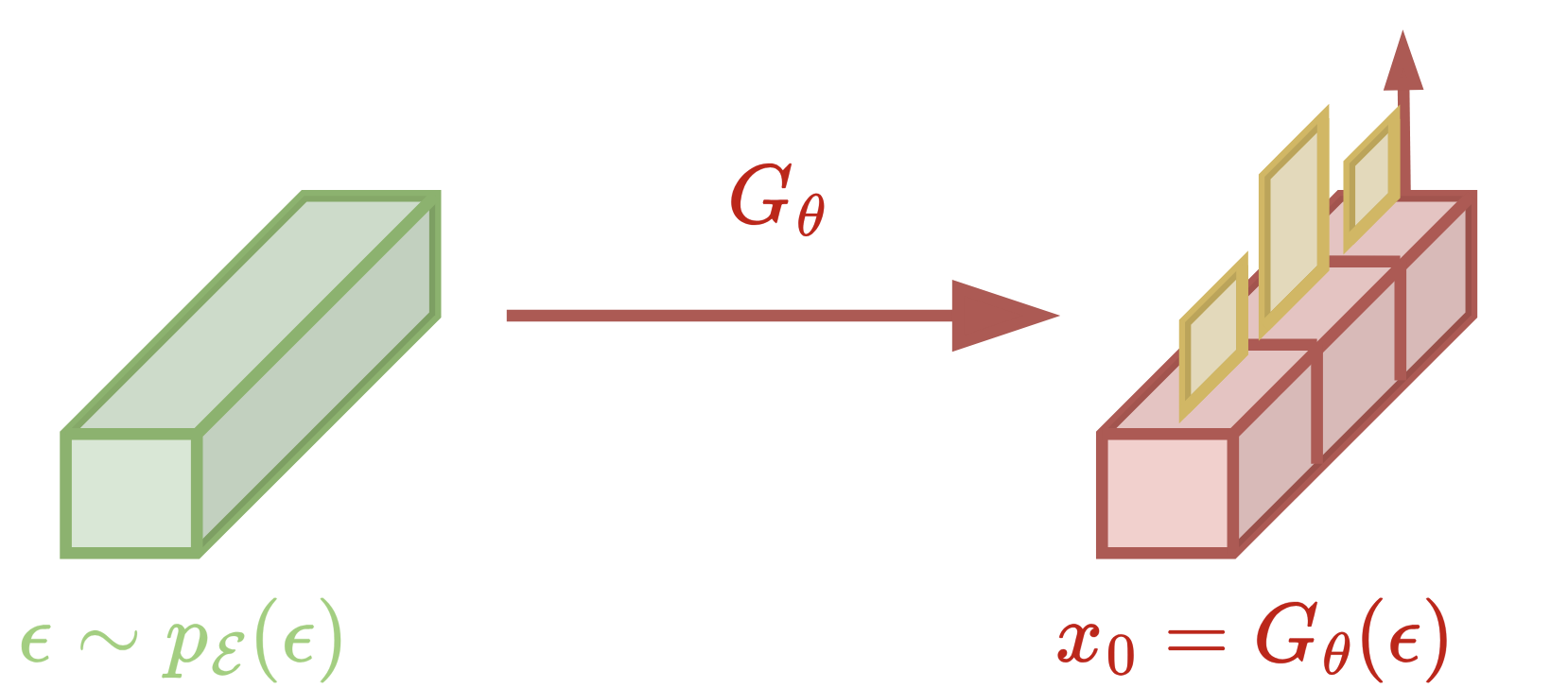}
        \caption{Simplex generator output.}
    \end{subfigure}
    \caption{Simplex relaxation for IDLM-MDLM, following
    Section~\ref{sec:simplex relaxation}. A generated token is ultimately
    interpreted as a one-hot vocabulary choice, but the generator update uses
    $G_\theta(\epsilon)\in\Delta$ so gradients can weight several candidate
    tokens before a hard sample is chosen.}
    \label{fig:mdlm-simplex-relaxation}
\end{figure*}

This also changes how the token loss should be read. With a hard one-hot
target, the inner product with $\log f(x_t,t)$ simply picks the log-probability
of one token. With $G_\theta(\epsilon)\in\Delta$, the same inner product becomes
a weighted sum over token log-probabilities, as shown in
Figure~\ref{fig:mdlm-token-loss}.

\begin{figure*}[t]
    \centering
    \begin{subfigure}[b]{0.48\textwidth}
        \centering
        \includegraphics[width=\linewidth]{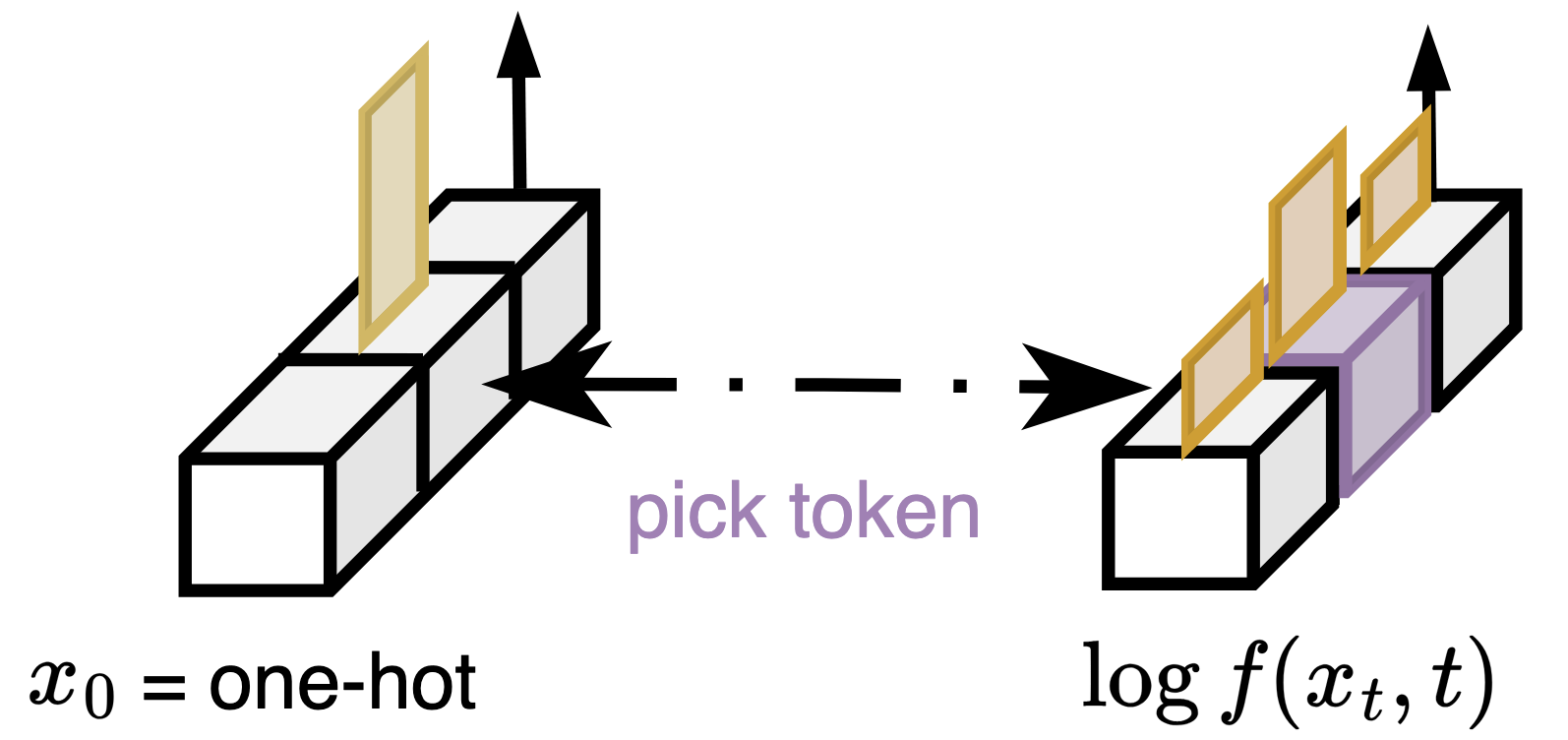}
        \caption{One-hot target: pick one coordinate.}
    \end{subfigure}
    \hfill
    \begin{subfigure}[b]{0.48\textwidth}
        \centering
        \includegraphics[width=\linewidth]{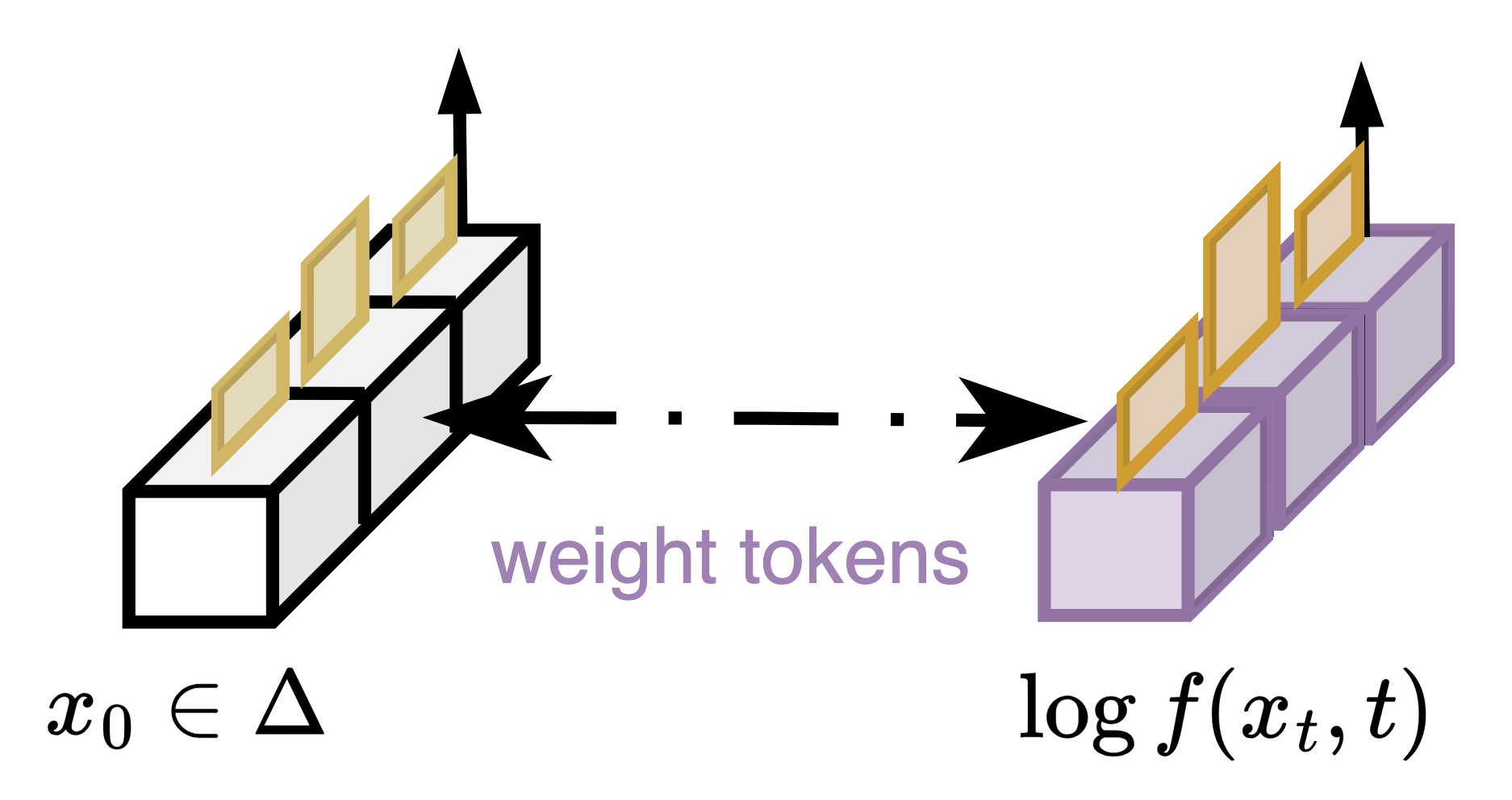}
        \caption{Simplex target: weight token coordinates.}
    \end{subfigure}
    \caption{Loss evaluation under the simplex relaxation.
    The one-hot case selects a single log-probability from $\log f(x_t,t)$,
    while the simplex case evaluates the weighted sum
    $\langle G_\theta(\epsilon),\log f(x_t,t)\rangle$.}
    \label{fig:mdlm-token-loss}
\end{figure*}

\paragraph{Mask-only cancellation.}
The \textsc{subs} parameterization copies any already unmasked token:
\[
    x_t\neq m
    \quad\Longrightarrow\quad
    f^*(x_t,t)=f(x_t,t)=x_t .
\]
Therefore, for non-mask states, the teacher and fake model make the same
prediction and their contributions cancel in the IDLM gap. Only masked states
remain. Figure~\ref{fig:mdlm-subs-cancellation} gives the presentation view of
this cancellation. If the forward process returns an unmasked state, the state
is already the clean token, so both the teacher and the fake model copy it. If
the forward process returns the mask state, the model must infer the clean
token, and this is where the IDLM signal survives.

\begin{figure*}[t]
    \centering
    \includegraphics[width=0.95\textwidth]{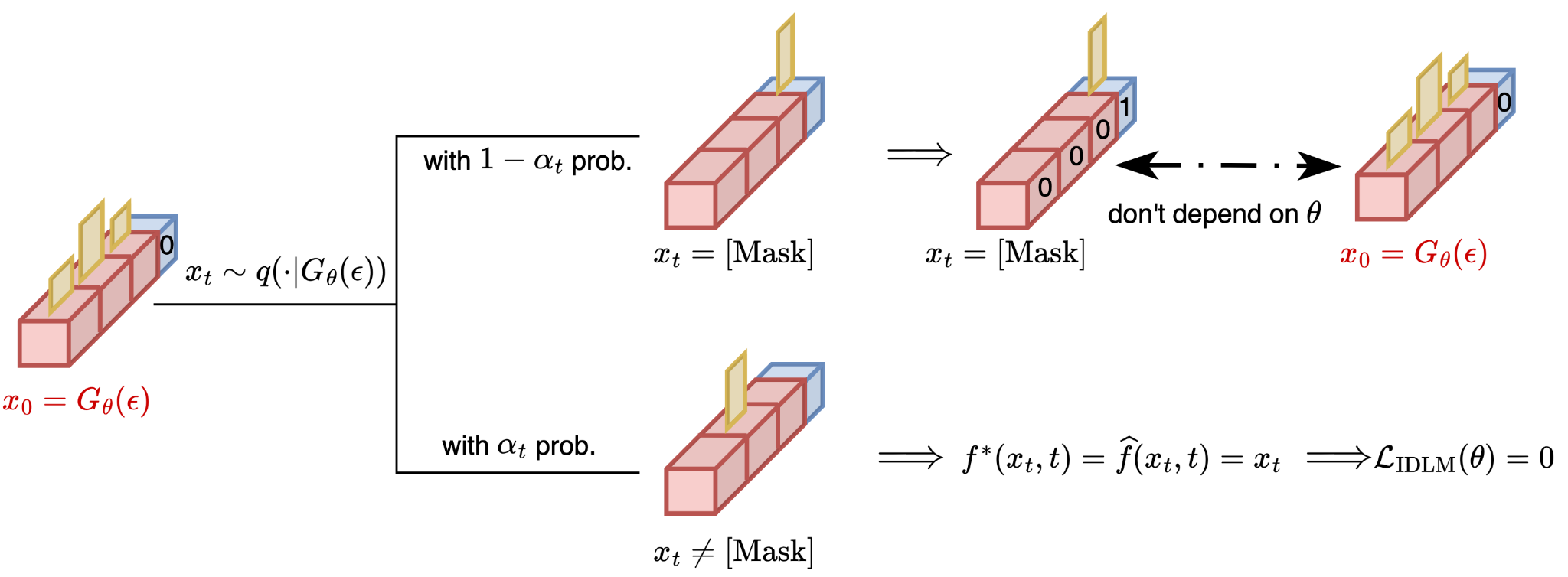}
    \caption{Mask-only cancellation under \textsc{subs}. The
    masked forward process produces either an unmasked token with probability
    $\alpha_t$ or the mask state with probability $1-\alpha_t$. The unmasked
    branch is copied by \textsc{subs}, so teacher and fake predictions agree and
    its IDLM contribution cancels. The remaining generator signal comes from
    the masked branch.}
    \label{fig:mdlm-subs-cancellation}
\end{figure*}

If $G_\theta(\epsilon)$ denotes the simplex-valued student output used
in the generator update, the MDLM term reduces to the mask-only loss used in
Section~\ref{sec:practical extension of idlm}:
\[
    \LC_{\mathrm{MDLM}}(f,p_\theta)
    =
    -\EE_{\epsilon,t}
    \left[
    (1-\alpha_t)\lambda_t
    \left\langle
    G_\theta(\epsilon),\log f(m,t)
    \right\rangle
    \right].
\]

\paragraph{Teacher-over-fake advantage.}
Consequently, the IDLM-MDLM generator update is driven by the relative
advantage of the teacher over the fake model at the mask state:
\[
    \LC_{\mathrm{IDLM-MDLM}}(\theta)
    =
    -\EE_{\epsilon,t}
    \left[
    (1-\alpha_t)\lambda_t
    \left\langle
    G_\theta(\epsilon),
    \log f^*(m,t)-\log f(m,t)
    \right\rangle
    \right].
\]
This is the mask-only signal used in Section~\ref{sec:practical extension of idlm}.
The sign structure gives the update an attraction--repulsion interpretation.
This is the same teacher-over-fake view developed in
Section~\ref{sec:loss intuition}.
The teacher term attracts the student toward tokens supported by the pretrained
model, while the fake-model term repels tokens that are already easy for a
model trained on the current student samples. Therefore, the student is not
simply pushed toward the teacher's most likely token. A token is promoted when
the teacher assigns it higher log-probability than the fake model does. Once
the fake model catches up with the teacher on student samples, this relative
signal disappears.

\paragraph{Why \textsc{subs} is special.}
The practical benefit is that the generator does not need to backpropagate
through a sampled intermediate token $x_t$. The only generator-dependent object
in the update is the differentiable simplex-valued output $G_\theta(\epsilon)$.
In DMD-like reductions of inverse distillation, an explicit stop-gradient is
usually needed to remove auxiliary gradients through intermediate states
\citep{selikhanovych2025one}. For MDLM, the \textsc{subs} parameterization
effectively handles this cancellation itself: non-mask states are copied and
masked states use the fixed input $m$. In this sense, \textsc{subs} implicitly
applies a stop-gradient to the intermediate state in the generator update. The
resulting mask-only expression resembles the token-level distribution-matching
loss used by discrete MMD distillation~\citep{hoogeboom2026beyond}.
The resemblance is structural rather than exact: D-MMD uses a different
posterior for sampling the intermediate state.
This is a special property of masked diffusion rather than a generic property
of all discrete diffusion processes. In more revisable processes such as
uniform diffusion, gradients through the intermediate state can change the
training behavior substantially; this is tested in Section~\ref{sec:ablation}
and Appendix~\ref{app:ablation-results}.

\paragraph{Beyond \textsc{subs}.}
One can try to pass gradients through masked-diffusion intermediate states
using Gumbel or Argmax relaxations, but this introduces a different problem:
the teacher MDLM is trained on one-hot or mask-token inputs, whereas these
relaxations may feed it simplex-valued states. The resulting teacher signal can
therefore be weaker or less stable. This is why the main IDLM-MDLM update uses
the mask-only \textsc{subs} signal and treats the other relaxations as
ablations in Appendix~\ref{app:ablation-results}.
\end{reviewblock}

\clearpage
\section{\texorpdfstring{\reviewchange{Uniform Diffusion}}{Uniform Diffusion}}
\label{app:uniform process}

\begin{reviewblock}
\reviewchange{This appendix gives the process-specific details for uniform diffusion used by UDLM and Duo in Section~\ref{sec:uniform process duo}, and then explains the corresponding IDLM update from Section~\ref{sec:loss intuition}.}

\subsection{Uniform Diffusion Formulation and Theory}
\label{app:uniform-formulation-theory}
Section~\ref{sec:uniform process duo} introduces uniform diffusion in compact form. Here we unpack the same process: first the CTMC generator and forward kernel, then the UDLM/Duo objective, the Gaussian relaxation used by Duo, and the reverse sampler. This makes explicit which quantities are reused later by the inverse distillation framework in Appendix~\ref{app:uniform-inverse-distillation}.

\paragraph{Forward process.}
Following the general CTMC notation of Appendix~\ref{app:details of considered processes}, the uniform process has terminal distribution $\pi=\frac{1}{N}\vec{1}$. Unlike masked diffusion in Appendix~\ref{app:absorbing process}, which only moves tokens toward the mask, the uniform process lets every token transition to every other token at the same rate. Its base matrix is
\begin{equation}
    Q_{\text{uni}} =
    \begin{bmatrix}
        1 - N & 1 & \cdots & 1 \\
        1 & 1 - N & \cdots & 1 \\
        \vdots & \vdots & \ddots & \vdots \\
        1 & 1 & \cdots & 1 - N
    \end{bmatrix}.
\end{equation}
The time-dependent diffusion matrix is
\begin{equation}
\label{eq: uniform diffusion matrix}
    Q_t = \sigma_t Q_{\text{uni}} =
    \begin{bmatrix}
        \sigma_t(1 - N) & \sigma_t & \cdots & \sigma_t \\
        \sigma_t & \sigma_t(1 - N) & \cdots & \sigma_t \\
        \vdots & \vdots & \ddots & \vdots \\
        \sigma_t & \sigma_t & \cdots & \sigma_t(1 - N)
    \end{bmatrix}.
\end{equation}
Its cumulative transition matrix is
\begin{equation}
    \exp(\bar{\sigma}_t Q_{\text{uni}}) =
    \begin{bmatrix}
        \exp(-N\bar{\sigma}_t) + \frac{1-\exp(-N\bar{\sigma}_t)}{N}
        & \frac{1-\exp(-N\bar{\sigma}_t)}{N} & \cdots & \frac{1-\exp(-N\bar{\sigma}_t)}{N} \\
        \frac{1-\exp(-N\bar{\sigma}_t)}{N}
        & \exp(-N\bar{\sigma}_t) + \frac{1-\exp(-N\bar{\sigma}_t)}{N}
        & \cdots & \frac{1-\exp(-N\bar{\sigma}_t)}{N} \\
        \vdots & \vdots & \ddots & \vdots \\
        \frac{1-\exp(-N\bar{\sigma}_t)}{N}
        & \frac{1-\exp(-N\bar{\sigma}_t)}{N}
        & \cdots & \exp(-N\bar{\sigma}_t) + \frac{1-\exp(-N\bar{\sigma}_t)}{N}
    \end{bmatrix}.
\end{equation}
With $\alpha_t \vcentcolon= \exp(-N\bar{\sigma}_t)$, the conditional distribution used in the main text is therefore
\[
    q_t(x_t\mid x_0)
    = \mathrm{Cat}\bigl(x_t;\alpha_t x_0 + (1-\alpha_t)\tfrac{1}{N}\vec{1}\bigr).
\]
Thus $x_t$ either keeps the clean token or mixes with the uniform distribution; unlike masked diffusion, a trajectory can revise the same position multiple times.

\paragraph{UDLM and Duo objective.}
UDLM~\citep{schiff2024discreteguidance} keeps the clean-token predictor $f(x_t,t)$ but uses the uniform noising law. Duo~\citep{sahoo2025the} uses the same discrete process while evaluating the denoiser on a relaxed input. The main text writes their shared training objective through the compact token-level function $g_{\mathrm{Duo}}$ in Eq.~\eqref{eq:duo loss}. Below we write this function explicitly.

Let $q_t(\cdot\mid x)$ denote the uniform-process transition distribution at time $t$. We use $\tilde{x}_t$ for the input passed to the denoiser and $\bar{x}_t$ for the hard intermediate token whose probability is evaluated in the token-level loss. For UDLM these two objects coincide, $\tilde{x}_t=\bar{x}_t=x_t$. For Duo, $\tilde{x}_t$ is the relaxed model input and $\bar{x}_t$ is the hard token induced by that relaxation. For the corresponding model prediction, write
\[
    f_t\vcentcolon=f(\tilde{x}_t,t).
\]
For the uniform CTMC above, the scalar weight becomes \citep{rutte2025generalized}
\[
    \kappa_t(\bar{x}_t,x_0)
    \vcentcolon=
    \frac{\sigma_t}{q_t(\bar{x}_t\mid x_0)}.
\]
The pointwise Itakura--Saito divergence is
\[
    D_{\mathrm{IS}}(a\|b) \vcentcolon= \frac{a}{b} - \log \frac{a}{b} - 1.
\]
The integrand used in Eq.~\eqref{eq:duo loss} is
\begin{equation}
\label{eq:uniform duo integrand}
g_{\mathrm{Duo}}\bigl(\bar{x}_t,x_0,f_t\bigr)
\vcentcolon=
\lambda_t\kappa_t(\bar{x}_t,x_0)\Bigl[\DC_{\mathrm{KL}}\bigl(q_t(\cdot\mid x_0)\|q_t(\cdot\mid f_t)\bigr) + D_{\mathrm{IS}}\bigl(q_t(\bar{x}_t\mid x_0)\|q_t(\bar{x}_t\mid f_t)\bigr)\Bigr].
\end{equation}
We keep $\lambda_t$ as the optional positive time reweighting used in the main notation. The KL term compares the true and predicted noised distributions. The Itakura--Saito term compares the probability assigned to the currently observed token $\bar{x}_t$, which is the term that later yields the holding-time intuition.

For comparison with the original CTMC notation, the UDLM objective can be written as
\begin{equation}
\label{eq:udlm loss}
\LC_{\mathrm{UDLM}}(f,x_0) = \int_0^1 \EE_{x_t\sim q_t(\cdot\mid x_0)}\bigl[g_{\mathrm{UDLM}}\bigl(x_t,x_0,f(x_t,t)\bigr)\bigr]dt,
\end{equation}
where $g_{\mathrm{UDLM}}$ is the hard-token version of Eq.~\eqref{eq:uniform duo integrand}: set $\tilde{x}_t=\bar{x}_t=x_t$ and $f_t=f(x_t,t)$. Here $q_t(\cdot\mid f_t)$ is obtained by applying the linear forward kernel to the predicted clean-token distribution.

\paragraph{Duo Gaussian relaxation.}
Duo exposes a Gaussian relaxation that makes the noised model input differentiable while keeping the same limiting discrete process. Following the notation of Section~\ref{sec:loss intuition}, introduce an auxiliary Gaussian noise variable and set
\[
    \xi\sim\NC(0,I),
    \qquad
    w_t=\tilde{\alpha}_t x_0+\sqrt{1-\tilde{\alpha}_t^2}\,\xi,
\]
where $\tilde{\alpha}_t$ is the rescaled Duo schedule. The hard intermediate token is obtained by
\[
    \bar{x}_t=x_t(w_t) \vcentcolon= \arg\max(w_t),
    \qquad
    \bar{x}_t\sim q_t(\cdot\mid x_0),
\]
where $\arg\max(w_t)$ denotes the one-hot vector at the largest coordinate. Instead of feeding this hard token to the model, Duo uses the soft approximation
\[
    \tilde{x}_t=x_t^\tau(w_t) = \mathrm{softmax}(w_t/\tau),
\]
which gives the relaxed parameterization $f\colon\Delta\times[0,1]\to\Delta$. With the notation of Eq.~\eqref{eq:duo loss}, Duo evaluates the same KL plus Itakura--Saito divergence objective at $g_{\mathrm{Duo}}(\bar{x}_t,x_0,f(\tilde{x}_t,t))$. As $\tau\to0^+$, the soft input recovers the hard uniform-process token and the objective recovers the UDLM/NELBO limit.

\paragraph{Sampling procedure.}
Sampling starts from the uniform terminal distribution and repeatedly applies reverse transitions estimated from the clean-token predictor. For $0\leq s<t\leq1$, set $\alpha_{t\mid s}\vcentcolon=\alpha_t/\alpha_s$ and $f_t\vcentcolon=f(x_t,t)$. The learned reverse transition used in Section~\ref{sec:uniform process duo} is
\[
p^f_{s\mid t}(x_s\mid x_t)
=
\mathrm{Cat}\!\left(x_s;\frac{N\alpha_t\,x_t\odot f_t+(\alpha_{t\mid s}-\alpha_t)x_t+(\alpha_s-\alpha_t)f_t+(1-\alpha_{t\mid s})(1-\alpha_s)\vec{1}/N}{N\alpha_t\langle x_t,f_t\rangle+1-\alpha_t}\right).
\]
This reverse step can either keep the current token or move to another token, so the sampler has a timing decision in addition to a token-choice decision, as illustrated in Figure~\ref{fig:appendix process intuition}. This is the main structural difference from masked diffusion.

\subsection{Inverse Distillation for Uniform Diffusion}
\label{app:uniform-inverse-distillation}
We now connect the uniform diffusion process to the IDLM objective. The ideal
inverse-distillation loss compares the pretrained teacher $f^*$ with the best
diffusion model that could be trained only on the current student distribution:
\begin{equation}
\label{eq:idlm-duo-compact}
    \LC_{\mathrm{IDLM\text{-}Duo}}(\theta)
    \vcentcolon=
    \LC_{\mathrm{Duo}}(f^*,p_\theta)
    -
    \min_{\tilde f}\LC_{\mathrm{Duo}}(\tilde f,p_\theta).
\end{equation}
In practice, as in Section~\ref{sec:loss intuition}, we alternate between
training the fake model and updating the generator. During the generator step,
the fake model $f$ is fixed. Write
\[
    f_t^* \vcentcolon= f^*(\tilde{x}_t,t),
    \qquad
    f_t \vcentcolon= f(\tilde{x}_t,t)
\]
for the teacher and fake predictions at the Duo input. The objective becomes
\begin{equation}
\label{eq:idlm-duo-fixed}
\begin{aligned}
    \LC_{\mathrm{IDLM\text{-}Duo}}(\theta)
    &\vcentcolon=
    \LC_{\mathrm{Duo}}(f^*,p_\theta)
    -
    \LC_{\mathrm{Duo}}(f,p_\theta)
    \\
    &=
    \EE_{\epsilon,t,\xi}
    \Bigl[
    \lambda_t\kappa_t(\bar{x}_t,G_\theta(\epsilon))
    \Bigl(
    \\
    &\quad
    \sum_{y\in\VC}q_t(y\mid G_\theta(\epsilon))\log\frac{q_t(y\mid f_t)}{q_t(y\mid f_t^*)}
    + q_t(\bar{x}_t\mid G_\theta(\epsilon))
    \left(\frac{1}{q_t(\bar{x}_t\mid f_t^*)}-\frac{1}{q_t(\bar{x}_t\mid f_t)}\right)
    + \log\frac{q_t(\bar{x}_t\mid f_t^*)}{q_t(\bar{x}_t\mid f_t)}
    \Bigr)
    \Bigr].
\end{aligned}
\end{equation}
Here $G_\theta(\epsilon)$ is the relaxed student sample, and
$(\bar{x}_t,\tilde{x}_t)$ are the hard and relaxed Duo states defined below.
This is the same teacher-minus-fake principle as in Appendix~\ref{app:mdlm-inverse-distillation}, but the uniform process does not collapse to a mask-only term.

\paragraph{Simplex relaxation.}
We use the same simplex relaxation as in Section~\ref{sec:simplex relaxation}
and Appendix~\ref{app:mdlm-inverse-distillation}. The generator output is
allowed to be a token distribution,
\[
    G_\theta(\epsilon)\in\Delta ,
\]
which is the same relaxation illustrated in
Figure~\ref{fig:mdlm-simplex-relaxation}. The uniform forward kernel is linear
in the clean token, so it remains well-defined on this relaxed output:
\[
    q_t(\cdot\mid G_\theta(\epsilon))
    =
    \alpha_tG_\theta(\epsilon)+(1-\alpha_t)\tfrac{1}{N}\vec{1}
    \in\Delta .
\]
Thus, as in the masked case, the generator can receive gradients through
weighted token probabilities instead of through a hard one-hot choice.

\paragraph{Gaussian reparameterization.}
The second issue is the intermediate uniform token. Unlike masked diffusion,
uniform diffusion can move from any token to any other token, so the sampled
intermediate state cannot be removed by the \textsc{subs} cancellation. Duo
resolves this by sampling a Gaussian latent variable whose noise is independent
of the generator output:
\[
    \xi\sim\NC(0,I),
    \qquad
    w_t=\tilde{\alpha}_tG_\theta(\epsilon)
    +\sqrt{1-\tilde{\alpha}_t^2}\,\xi .
\]
The hard token used in the token-level loss and the relaxed input passed to
the denoiser are
\[
    \bar{x}_t=\arg\max(w_t),
    \qquad
    \tilde{x}_t=\mathrm{softmax}(w_t/\tau).
\]
Equivalently, the full differentiable transition from the generator output to
the relaxed Duo input is
\[
    G_\theta(\epsilon)
    \mapsto
    \tilde{x}_t
    =
    \mathrm{softmax}\!\left(
    \frac{\tilde{\alpha}_tG_\theta(\epsilon)
    +\sqrt{1-\tilde{\alpha}_t^2}\,\xi}{\tau}
    \right),
    \qquad
    \xi\sim\NC(0,I).
\]
These definitions make the fixed-fake objective in Eq.~\eqref{eq:idlm-duo-fixed}
fully reparameterized and keep the intermediate-state path differentiable.

\paragraph{Teacher-over-fake advantage.}
Following the local-advantage notation of Section~\ref{sec:loss intuition}, define
\[
a_t \vcentcolon= g_{\mathrm{Duo}}(\bar{x}_t,G_\theta(\epsilon),f_t^*) - g_{\mathrm{Duo}}(\bar{x}_t,G_\theta(\epsilon),f_t).
\]
Using Eq.~\eqref{eq:uniform duo integrand}, and dropping the common positive
weight $\lambda_t\kappa_t(\bar{x}_t,G_\theta(\epsilon))$, the local advantage
expands to
\[
\begin{aligned}
\bigl(\lambda_t\kappa_t(\bar{x}_t,G_\theta(\epsilon))\bigr)^{-1}a_t
&=
-\left\langle q_t(\cdot\mid G_\theta(\epsilon)),\log\frac{q_t(\cdot\mid f_t^*)}{q_t(\cdot\mid f_t)}\right\rangle
\\
&\quad
+ q_t(\bar{x}_t\mid G_\theta(\epsilon))
\left(\frac{1}{q_t(\bar{x}_t\mid f_t^*)}-\frac{1}{q_t(\bar{x}_t\mid f_t)}\right)
+ \log\frac{q_t(\bar{x}_t\mid f_t^*)}{q_t(\bar{x}_t\mid f_t)}.
\end{aligned}
\]
The first term is the distribution-level teacher-over-fake signal at time $t$.
The remaining terms compare how much probability the teacher and fake model
assign to the current intermediate token $\bar{x}_t$.

\paragraph{Intuition.}
Uniform diffusion can revise the same position more than once. At each reverse
step, the sampler decides whether to keep the current token or replace it.
The first term in the local advantage is the time-$t$ analogue of the
distribution-level attraction--repulsion signal from masked diffusion
(Appendix~\ref{app:mdlm-inverse-distillation}):
\begingroup\small
\[
-\left\langle q_t(\cdot\mid G_\theta(\epsilon)),
\log\frac{q_t(\cdot\mid f_t^*)}{q_t(\cdot\mid f_t)}\right\rangle
=
\DC_{\mathrm{KL}}\!\left(q_t(\cdot\mid G_\theta(\epsilon))\|q_t(\cdot\mid f_t^*)\right)
-
\DC_{\mathrm{KL}}\!\left(q_t(\cdot\mid G_\theta(\epsilon))\|q_t(\cdot\mid f_t)\right).
\]
\endgroup
Thus, it favors generator samples whose distribution after uniform noising at
time $t$ is closer to the teacher than to the fake model. Compared with masked
diffusion, this comparison is not made directly on clean-token distributions:
the uniform CTMC first mixes each token toward the uniform stationary
distribution, and the amount of mixing depends on $t$. The signal is therefore
a difference between KL divergences of the time-dependent noised distributions.
The holding-time intuition is controlled by the probability assigned to the
current intermediate token: larger $q_t(\bar{x}_t\mid f_t^*)$ corresponds to a
longer teacher holding time at $\bar{x}_t$, while smaller
$q_t(\bar{x}_t\mid f_t^*)$ corresponds to a shorter one. This effect appears
directly in the current-token part of Eq.~\eqref{eq:idlm-duo-fixed}:
\[
q_t(\bar{x}_t\mid G_\theta(\epsilon))
\left(
\frac{1}{q_t(\bar{x}_t\mid f_t^*)}
-
\frac{1}{q_t(\bar{x}_t\mid f_t)}
\right)
+
\log\frac{q_t(\bar{x}_t\mid f_t^*)}{q_t(\bar{x}_t\mid f_t)}.
\]
For a fixed intermediate token, the coefficient of
$q_t(\bar{x}_t\mid G_\theta(\epsilon))$ is
\[
c_t(\bar{x}_t)
\vcentcolon=
\frac{1}{q_t(\bar{x}_t\mid f_t^*)}
-
\frac{1}{q_t(\bar{x}_t\mid f_t)}.
\]
If $q_t(\bar{x}_t\mid f_t^*)$ is small, then $c_t(\bar{x}_t)$ is large and
positive, so minimizing the loss pushes the generator to reduce the probability
of clean samples whose noising lands at $\bar{x}_t$. This agrees with the
holding-time view: the teacher reverse process has a large exit rate at such an
intermediate state and would leave it quickly, so the generator is discouraged
from producing samples whose uniform noising often visits that state. If the
teacher probability is large while the fake-model probability is small, then
$c_t(\bar{x}_t)<0$. Since it minimizes the IDLM loss, the generator is pushed
toward clean samples whose uniform noising visits such teacher-supported states
more often; the small fake-model probability strengthens this preference through
the negative term $-1/q_t(\bar{x}_t\mid f_t)$. This is the uniform-process
analogue of the attraction--repulsion interpretation in
Appendix~\ref{app:mdlm-inverse-distillation}: teacher-supported states attract
the generator, while the fake model's low probability makes this signal
stronger. In the next alternating step, the fake model is trained on the updated generator
distribution, so it adapts to these states and reduces the teacher--fake gap
there. Thus, the uniform loss teaches both which token should be chosen and how
long the sampler should stay in the current intermediate state.

\paragraph{DMD-like stop-gradient for uniform diffusion.}
The ablation in Section~\ref{sec:ablation} and
Appendix~\ref{app:ablation-results} compares the full IDLM-Duo update with a
stop-gradient variant. In the idealized objective, IDLM matches complete diffusion path
measures,
\[
    \LC_{\mathrm{IDLM}}(\theta)
    =
    \DC_{\mathrm{KL}}(\PP^\theta\|\PP^*),
\]
where $\PP^\theta$ and $\PP^*$ are the path measures induced by the student and
data distributions. For any time $t$, let
\[
    p_t^\theta(x_t)=\int q_t(x_t\mid x_0)\,p_\theta(dx_0),
    \qquad
    p_t^*(x_t)=\int q_t(x_t\mid x_0)\,p^*(dx_0).
\]
A DMD-like reduction, as shown by RSD~\citep{selikhanovych2025one}, replaces
the path-measure objective by noised-marginal matching,
\[
    \LC_{\mathrm{DMD\text{-}like}}(\theta)
    =
    \EE_t\!\left[
    \DC_{\mathrm{KL}}(p_t^\theta\|p_t^*)
    \right],
\]
and treats the intermediate state as fixed in the generator update. In the Duo
parameterization this corresponds to replacing the full relaxed path
\[
    w_t
    =
    \tilde{\alpha}_tG_\theta(\epsilon)
    +\sqrt{1-\tilde{\alpha}_t^2}\,\xi
\]
by its stop-gradient version
\[
    w_t^{\mathrm{sg}}
    =
    \tilde{\alpha}_t\mathrm{sg}\!\left(G_\theta(\epsilon)\right)
    +\sqrt{1-\tilde{\alpha}_t^2}\,\xi .
\]
Thus, the DMD-like stop-gradient objective optimizes a marginal matching problem
rather than the path-level IDLM-Duo objective. Empirically, this distinction is
important: Figure~\ref{fig:ablation-uniform} shows that IDLM-Duo with stop-gradient
quickly enters a high-entropy, high-GenPPL regime, while the full IDLM-Duo update
remains stable.

\begin{figure}[t]
    \centering
    \includegraphics[width=\textwidth]{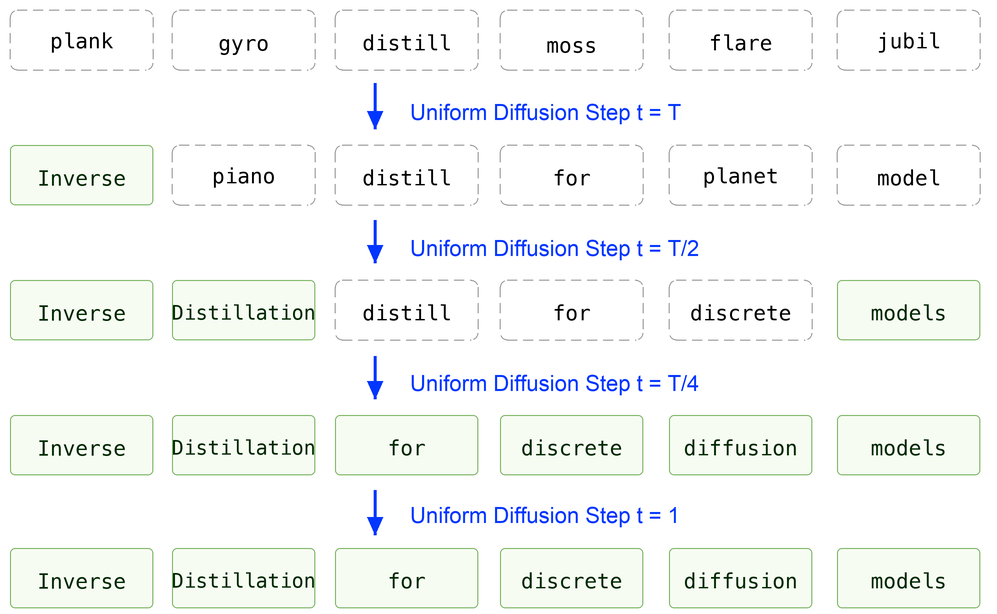}
    \caption{Uniform diffusion sampling. Unlike masked diffusion, a position can be revised multiple times, so each step must decide both whether to hold or change the current token and which token should replace it.}
    \label{fig:appendix process intuition}
\end{figure}
\end{reviewblock}

\section{Proofs}
\label{app: proofs}

\paragraph{Notation and idealized assumptions.}
Recall that
\[
    \VC=\{x\in\{0,1\}^N:\sum_{i=1}^N x_i=1\}
\]
is the set of one-hot tokens, and let \(D([0,T],\VC)\) be the Skorokhod space of
cadlag paths with values in \(\VC\). For a path measure \(\mathbb Q\), write
\(\mathbb Q_t\) for its marginal at time \(t\).

We assume that the teacher exactly minimizes the expected loss and therefore
induces the true reverse path measure, \(\PP^{f^*}=\PP^*\).
We also assume that the minimization over \(f\) is performed over a sufficiently
rich model class, so the exact reverse predictor is available for every clean
distribution considered. We assume that the student path does not use
transitions that are impossible under the teacher path. If
\(\PP^\theta\not\ll\PP^*\), then
\(\DC_{\mathrm{KL}}(\PP^\theta\|\PP^*)=+\infty\), so the main theorem becomes
immediate.

\paragraph{Time convention.}
For a clean distribution \(p\), let \(\PP^p\) denote the path measure associated
with the fixed forward noising process, written in reverse-time coordinates so
that the clean sample is the terminal marginal of \(\PP^p\). Its forward-time
noised marginal at noise level \(t\) is
\[
    p_t(x_t)=\int q_t(x_t\mid x_0)\,p(dx_0).
\]
We write \(\PP^\theta=\PP^{p_\theta}\) and \(\PP^*=\PP^{p^*}\). Since time
reversal is a deterministic reparameterization of paths, the KL divergence is
unchanged by choosing either coordinate convention.

\subsection{Connection with the main text}
\label{app: connection with the main text}

DLM training can be viewed as fitting the reverse-time path measure associated
with a fixed forward noising process. Following the main text, we write this
using expected losses of the form \(\LC(f,p)\), where \(p\) is a clean-sample
distribution. For SEDD, Dynkin's formula~\citep{hanson2007applied,
campbell2022continuous} gives the
following distribution-level objective, up to terms independent of the optimized
model:
\begin{equation}
\label{eq:sedd-distribution}
\begin{aligned}
    \LC_{\mathrm{SEDD}}(f,p)
    =
    \EE_{p(x_0),t,q_t}
    \sum_{y\neq x_t}
    \langle x_t,y\rangle_{Q_t}
    \bigl(
      \langle f(x_t,t),y\rangle
      -
      \langle s(x_t,t\mid x_0),y\rangle
      \log \langle f(x_t,t),y\rangle
    \bigr).
\end{aligned}
\end{equation}
This is the SEDD instance of the unified loss \(\LC(f,p)\) used in
Equation~\eqref{eq:sedd loss}. For a generic clean distribution \(p\), define
\[
    f^p=\arg\min_f \LC_{\mathrm{SEDD}}(f,p).
\]
At this minimizer,
\begin{equation}
\label{eq:posterior-score}
    \langle f^p(x_t,t),y\rangle
    =
    \EE_{p(x_0\mid x_t)}
    \langle s(x_t,t\mid x_0),y\rangle .
\end{equation}
In particular, \(f^*=f^{p^*}\) denotes the exact teacher and
\(f^\theta=f^{p_\theta}\) denotes the exact fake model for the student
distribution.

\subsection{Proof of the theorem}
\label{app: proof of the theorem}

We prove the theorem in three steps. First, we show that the path KL equals the
SEDD inverse-distillation gap. Second, we transfer the result to MDLM and the
uniform objective, with Duo covered in the hard-token limit. Third, we use the
chain rule to show that the path KL controls the clean-data KL.

\begin{lemma}[SEDD path KL equals the inverse gap]
\label{lem:sedd-path-kl-gap}
Under the population assumptions above,
\begin{equation}
\label{eq:sedd-path-kl-gap}
    \DC_{\mathrm{KL}}(\PP^\theta\|\PP^*)
    =
    \LC_{\mathrm{SEDD}}(f^*,p_\theta)
    -
    \min_f \LC_{\mathrm{SEDD}}(f,p_\theta).
\end{equation}
\end{lemma}

\begin{proof}
The relative entropy between two continuous-time jump-process path measures can
be written as an integral over the relative jump-rate divergence. In the SEDD
parameterization this gives
\begin{equation}
\label{eq:ctmc-kl}
\begin{aligned}
    \DC_{\mathrm{KL}}(\PP^\theta\|\PP^*)
    &=
    \int_0^T
    \EE_{p_t^\theta(x_t)}
    \sum_{y\neq x_t}
    \langle x_t,y\rangle_{Q_t}
    \langle f^*(x_t,t),y\rangle
    \phi
    \left(
      \frac{\langle f^\theta(x_t,t),y\rangle}
           {\langle f^*(x_t,t),y\rangle}
    \right)dt,
\end{aligned}
\end{equation}
where \(\phi(r)=r\log r-r+1\). Following the linearization trick
of~\citet{kornilov2025universal}, for the positive ratios used below,
\[
    \phi(r)=\max_l \{rl+1-\exp(l)\}.
\]
Using the parameterization
\[
    l(y,x_t,t)
    =
    \log
    \frac{\langle f(x_t,t),y\rangle}
         {\langle f^*(x_t,t),y\rangle},
\]
and taking the maximum over admissible \(f\), Equation~\eqref{eq:ctmc-kl}
becomes
\begin{equation}
\label{eq:sedd-max}
\begin{aligned}
    \DC_{\mathrm{KL}}(\PP^\theta\|\PP^*)
    =
    \max_f
    \int_0^T
    \EE_{p_t^\theta(x_t)}
    \sum_{y\neq x_t}
    \langle x_t,y\rangle_{Q_t}
    \biggl[
      &\langle f^\theta(x_t,t),y\rangle
      \log
      \frac{\langle f(x_t,t),y\rangle}
           {\langle f^*(x_t,t),y\rangle}
      \\
      &+
      \langle f^*(x_t,t)-f(x_t,t),y\rangle
    \biggr]dt.
\end{aligned}
\end{equation}
We now rewrite the expectation in Equation~\eqref{eq:sedd-max}. First,
apply the posterior identity~\eqref{eq:posterior-score} with
\(p=p_\theta\):
\[
    \langle f^\theta(x_t,t),y\rangle
    =
    \EE_{p_\theta(x_0\mid x_t)}
    \langle s(x_t,t\mid x_0),y\rangle .
\]
After substituting this identity, the only term depending on \(x_0\) is
\(\langle s(x_t,t\mid x_0),y\rangle\). The remaining term
\(\langle f^*(x_t,t)-f(x_t,t),y\rangle\) depends only on \(x_t\), \(t\), and
\(y\), so it can be placed inside the same conditional expectation:
\[
\begin{aligned}
&\EE_{p_t^\theta(x_t)}
\sum_{y\neq x_t}
\langle x_t,y\rangle_{Q_t}
\biggl[
  \langle f^\theta(x_t,t),y\rangle
  \log
  \frac{\langle f(x_t,t),y\rangle}
       {\langle f^*(x_t,t),y\rangle}
  +
  \langle f^*(x_t,t)-f(x_t,t),y\rangle
\biggr]
\\
&=
\EE_{p_t^\theta(x_t)}
\EE_{p_\theta(x_0\mid x_t)}
\sum_{y\neq x_t}
\langle x_t,y\rangle_{Q_t}
\biggl[
  \langle s(x_t,t\mid x_0),y\rangle
  \log
  \frac{\langle f(x_t,t),y\rangle}
       {\langle f^*(x_t,t),y\rangle}
  +
  \langle f^*(x_t,t)-f(x_t,t),y\rangle
\biggr].
\end{aligned}
\]
Finally, the joint law of \((x_0,x_t)\) can be written in two equivalent ways:
\[
    p_t^\theta(x_t)p_\theta(x_0\mid x_t)
    =
    p_\theta(x_0)q_t(x_t\mid x_0),
\]
so
\[
    \EE_{p_t^\theta(x_t)}
    \EE_{p_\theta(x_0\mid x_t)}[\cdot]
    =
    \EE_{p_\theta(x_0),\,x_t\sim q_t(\cdot\mid x_0)}[\cdot].
\]
The maximization over \(f\) stays outside this change of variables. Therefore
we obtain
\begin{equation}
\label{eq:sedd-population-recognition}
\begin{aligned}
    \DC_{\mathrm{KL}}(\PP^\theta\|\PP^*)
    =
    \max_f
    \EE_{p_\theta(x_0),t,q_t}
    \sum_{y\neq x_t}
    \langle x_t,y\rangle_{Q_t}
    \biggl[
      &\langle s(x_t,t\mid x_0),y\rangle
      \log
      \frac{\langle f(x_t,t),y\rangle}
           {\langle f^*(x_t,t),y\rangle}
      \\
      &+
      \langle f^*(x_t,t)-f(x_t,t),y\rangle
    \biggr].
\end{aligned}
\end{equation}
The right-hand side is exactly the population SEDD loss difference:
\[
    \max_f
    \left\{
      \LC_{\mathrm{SEDD}}(f^*,p_\theta)
      -
      \LC_{\mathrm{SEDD}}(f,p_\theta)
    \right\}.
\]
Because the first term is independent of \(f\), this equals
\[
    \LC_{\mathrm{SEDD}}(f^*,p_\theta)
    -
    \min_f \LC_{\mathrm{SEDD}}(f,p_\theta).
\]
This proves the claim. Importantly, the maximum remains outside the population
expectation throughout the argument.
\end{proof}

\begin{lemma}[Gap preservation under equivalent losses]
\label{lem:gap-preservation}
Suppose two losses satisfy
\[
    \LC_A(Tf,p)=\LC_B(f,p)+C(p),
\]
where \(C(p)\) is independent of the optimized model \(f\), and \(T\) maps the
candidate class for \(\LC_B\) onto the corresponding candidate class for
\(\LC_A\). Then their inverse gaps agree:
\[
    \LC_A(Tf^*,p)-\min_g \LC_A(g,p)
    =
    \LC_B(f^*,p)-\min_f \LC_B(f,p).
\]
\end{lemma}

\begin{proof}
The claim follows by cancellation of the model-independent constant:
\[
\begin{aligned}
    \LC_A(Tf^*,p)-\min_g\LC_A(g,p)
    &=
    \LC_B(f^*,p)+C(p)
    -
    \min_f\{\LC_B(f,p)+C(p)\}
    \\
    &=
    \LC_B(f^*,p)-\min_f\LC_B(f,p).
\end{aligned}
\]
\end{proof}

\paragraph{MDLM, UDLM, and Duo.}
For the absorbing diffusion matrix \(Q_{\mathrm{abs}}\) in
Equation~\eqref{eq: absorbing diffusion matrix}, the SEDD objective under the
clean-token parameterization is equivalent to the MDLM loss
\(\LC_{\mathrm{MDLM}}\) up to terms independent of the optimized predictor.
Lemma~\ref{lem:gap-preservation} therefore gives
\[
    \DC_{\mathrm{KL}}(\PP^\theta\|\PP^*)
    =
    \LC_{\mathrm{MDLM}}(f^*,p_\theta)
    -
    \min_f \LC_{\mathrm{MDLM}}(f,p_\theta).
\]
For the uniform diffusion matrix \(Q_{\mathrm{uni}}\) in
Equation~\eqref{eq: uniform diffusion matrix}, the same argument transfers the
SEDD gap to the hard-token UDLM objective. In distribution-level notation,
\[
    \LC_{\mathrm{UDLM}}(f,p)
    \vcentcolon=
    \EE_{p(x_0),t,x_t\sim q_t(\cdot\mid x_0)}
    \bigl[
      g_{\mathrm{UDLM}}\bigl(x_t,x_0,f(x_t,t)\bigr)
    \bigr].
\]
Thus
\[
    \DC_{\mathrm{KL}}(\PP^\theta\|\PP^*)
    =
    \LC_{\mathrm{UDLM}}(f^*,p_\theta)
    -
    \min_f \LC_{\mathrm{UDLM}}(f,p_\theta).
\]
Duo evaluates the same uniform-process objective on a Gaussian-relaxed input and
recovers the hard-token UDLM objective as \(\tau\to0^+\). Hence the theorem
applies to Duo in this hard-token limit, as stated in the main text. This
argument does not claim exact uniqueness for finite-temperature Duo.

\begin{theorem}[Unique solution]
\label{thm:unique-solution-appendix}
For SEDD, MDLM, and hard-token UDLM, and for Duo in the limit
\(\tau\to0^+\), the IDLM loss satisfies
\[
    \LC_{\mathrm{IDLM}}(\theta)
    =
    \DC_{\mathrm{KL}}(\PP^\theta\|\PP^*)
    \geq
    \DC_{\mathrm{KL}}(p_\theta\|p^*)
    \geq 0.
\]
Moreover,
\[
    \LC_{\mathrm{IDLM}}(\theta)=0
    \quad\Longleftrightarrow\quad
    p_\theta=p^*.
\]
\end{theorem}

\begin{proof}
Disintegrate the path measures with respect to their clean-sample marginal:
\[
    \PP^\theta(dx,dw)=p_\theta(dx)\PP_x^\theta(dw),
    \qquad
    \PP^*(dx,dw)=p^*(dx)\PP_x^*(dw).
\]
The chain rule for relative entropy gives
\[
    \DC_{\mathrm{KL}}(\PP^\theta\|\PP^*)
    =
    \DC_{\mathrm{KL}}(p_\theta\|p^*)
    +
    \EE_{x\sim p_\theta}
    \DC_{\mathrm{KL}}(\PP_x^\theta\|\PP_x^*).
\]
The conditional KL term is nonnegative, so
\[
    \DC_{\mathrm{KL}}(\PP^\theta\|\PP^*)
    \geq
    \DC_{\mathrm{KL}}(p_\theta\|p^*)
    \geq 0.
\]
Now consider the zero-loss condition. From the previous part of the proof,
\(\LC_{\mathrm{IDLM}}(\theta)=\DC_{\mathrm{KL}}(\PP^\theta\|\PP^*)\). Therefore,
if \(\LC_{\mathrm{IDLM}}(\theta)=0\), then
\[
    \DC_{\mathrm{KL}}(\PP^\theta\|\PP^*)=0.
\]
Substituting this into the chain-rule decomposition gives
\[
    0
    =
    \DC_{\mathrm{KL}}(p_\theta\|p^*)
    +
    \EE_{x\sim p_\theta}
    \DC_{\mathrm{KL}}(\PP_x^\theta\|\PP_x^*).
\]
Both terms on the right are nonnegative, so each term must be zero. In
particular,
\[
    \DC_{\mathrm{KL}}(p_\theta\|p^*)=0.
\]
The KL divergence between two probability distributions is zero only when the
two distributions are equal, hence \(p_\theta=p^*\).

Conversely, assume \(p_\theta=p^*\). The forward noising process is fixed: once
the clean distribution \(p\) is chosen, the whole path measure \(\PP^p\) is
obtained by applying the same noising dynamics to samples from \(p\). Therefore,
if the clean distributions are equal, the induced path measures are also equal:
\[
    \PP^\theta=\PP^{p_\theta}=\PP^{p^*}=\PP^*.
\]
Thus
\[
    \DC_{\mathrm{KL}}(\PP^\theta\|\PP^*)=0,
\]
and using again
\(\LC_{\mathrm{IDLM}}(\theta)=\DC_{\mathrm{KL}}(\PP^\theta\|\PP^*)\), we get
\(\LC_{\mathrm{IDLM}}(\theta)=0\).
\end{proof}

\paragraph{Sequence-level extension.}
\label{app:sequence-level-extension}
The proof above is for one token. For a full sequence, the same proof works in
an ideal setting where the whole sequence is treated as one state and the models
describe the full reverse process over sequences. In that case, the IDLM loss is
the KL divergence between the full sequence path measures, so it also controls
the difference between \(p_\theta(x_0^{1:L})\) and \(p^*(x_0^{1:L})\). This is
not exactly the loss used in experiments. In practice, we sum token-level losses,
which is convenient for training but does not by itself give the full sequence
path KL, because tokens may still depend on each other. Thus, the proof covers
the one-token objective and the ideal full-sequence objective, but not the
practical factorized sequence loss. We leave a formal theoretical analysis of
the factorized sequence-level IDLM objective as a promising direction for future
research.

\section{Experimental Details.}
\label{app: experimental details}

\subsection{Unconditional Generation on OpenWebText}
\label{app:owt-experimental-details}

\noindent \textbf{Codebase}. Our implementation builds upon the original Duo~\href{https://github.com/s-sahoo/duo}{repository}~\citep{sahoo2025the} and the SEDD~\href{https://github.com/louaaron/Score-Entropy-Discrete-Diffusion}{repository}~\citep{lou2024discrete}, which serve as the primary codebases for our experiments. We extend these repositories to incorporate our proposed training framework, \reviewchange{with the main algorithmic details provided in Algorithm~\ref{alg:idlm}.}

\noindent \textbf{Teacher checkpoints}. We use pretrained checkpoints for the distillation of MDLM and Duo models obtained from the Duo~\href{https://github.com/s-sahoo/duo}{repository}~\citep{sahoo2025the}, and a pretrained checkpoint for SEDD distillation from the SEDD~\href{https://github.com/louaaron/Score-Entropy-Discrete-Diffusion}{repository}~\citep{lou2024discrete}. Unfortunately, for SEDD, only a checkpoint corresponding to the absorbing process was available.

\noindent \textbf{Model architecture}. We follow~\citep{sahoo2025the} for distilling MDLM and Duo, and design the denoising model to operate on both continuous and discrete latents. We use a Transformer~\citep{vaswani2017attention} in which the first-layer token representations are computed via a matrix multiplication with the embedding matrix. For discrete inputs, we perform standard embedding lookups. In contrast, continuous inputs correspond to “soft lookups,” producing a convex combination of vocabulary embeddings. For SEDD distillation, we follow~\citep{lou2024discrete} and use a model that only accepts one-hot vectors. For all setups we use small models with $169$M parameters, specifically with $12$ layers, $12$ heads, hidden size $768$, sequence length $1024$, and dropout $0.1$.

\noindent \textbf{Training hyperparameters}.  
Across all experimental settings, we employed a per-device batch size of 8, resulting in a global batch size of 512. Optimization was performed using the AdamW optimizer~\citep{loshchilov2017decoupled}, with a fixed learning rate of $1 \times 10^{-6}$ for MDLM and Duo, and $3 \times 10^{-7}$ for SEDD. A constant learning rate schedule was used with 2500 warmup steps for all configurations. We adopted the log-linear noise formulation from~\citep{lou2024discrete} in the loss function, and applied exponential moving average (EMA) with a decay rate of 0.9999 throughout training.

\noindent \textbf{Evaluation protocol}. 
As observed by~\citep{zheng2024masked}, the GenPPL metric is sensitive to floating-point precision. To ensure consistency and numerical stability, we follow the protocol of~\citep{sahoo2025the} and perform all sampling experiments using \texttt{float64} precision. All evaluation metrics are computed using the official codebase provided in the Duo~\citep{sahoo2025the} repository.

\phantomsection\label{app:metric-definitions}
\noindent \textbf{Metric definitions}.
Let \(x_0^{1:L}\sim p_\theta\) denote generated samples and let \(p^*\) denote the validation-data distribution. For a finite generated set \(\mathcal{S}_\theta\), generative perplexity (GenPPL) is computed by scoring each generated sequence with a fixed GPT-2 Large autoregressive evaluator \(p_{\mathrm{AR}}\)~\citep{radford2019language}:
\[
\mathrm{GenPPL}
=
\exp\!\left(
-\frac{1}{M}
\sum_{x_0^{1:L}\in\mathcal{S}_\theta}
\sum_{i=1}^{L}
\log p_{\mathrm{AR}}(x_0^i\mid x_0^{1:i-1})
\right),
\]
where \(M\) is the total number of evaluated tokens. Lower GenPPL means that the evaluator assigns higher likelihood to generated samples, and we report it together with diversity metrics to avoid favoring overly repetitive text. Entropy is the empirical token entropy of generated samples:
\[
\mathrm{Entropy}
=
-\sum_{v\in\VC}\hat{p}_\theta(v)\log \hat{p}_\theta(v),
\]
where \(\hat{p}_\theta(v)\) is the generated-token frequency of token \(v\in\VC\); higher entropy means greater token-level diversity. MAUVE~\citep{lou2024discrete} compares samples from \(p_\theta\) with validation samples from \(p^*\) in language-model embedding space, higher values mean that the two distributions are closer. Finally, GM denotes the GPT-2 Gradient Moment metric~\citep{hoogeboom2026beyond}:
\[
\mathrm{GM}(p_\theta,p^*)
=
\left\|
\EE_{x_0^{1:L}\sim p_\theta}\nabla_{\psi}\log p_{\mathrm{AR},\psi}(x_0^{1:L})
-
\EE_{x_0^{1:L}\sim p^*}\nabla_{\psi}\log p_{\mathrm{AR},\psi}(x_0^{1:L})
\right\|_2^2,
\]
estimated with independent minibatches. Lower GM means that generated samples induce reference-model gradients closer to those induced by validation data.

\noindent \textbf{Dataset preparation}.
Following prior work~\citep{sahoo2024simple, sahoo2025the, lou2024discrete}, we preprocess the One Billion Words dataset using the detokenization procedure described by~\citet{lou2024discrete} and~\citet{sahoo2024simple}, with the official implementation available at \href{https://github.com/louaaron/Score-Entropy-Discrete-Diffusion/blob/main/data.py}{this link}. For the OpenWebText dataset, we utilize the \texttt{GPT2} tokenizer and apply a similar concatenation and wrapping procedure as in the previous works~\citep{sahoo2024simple, sahoo2025the}, targeting a sequence length of 1,024 tokens. During wrapping, \texttt{eos} tokens are inserted between consecutive sequences. As OpenWebText does not include an official validation split, we reserve the final 100{,}000 documents from the dataset as a validation set.

\noindent \textbf{Other baselines}.
Our primary baselines are SDTT~\citep{deschenaux2024beyond} and Duo-DCD~\citep{sahoo2025the}. Performance metrics for both methods are reported using values provided in the Duo~\citep{sahoo2025the}, corresponding to the best-performing results obtained after 5 rounds of distillation, described in~\citep{sahoo2025the}. While a valid pretrained checkpoint was available for Duo-DCD, we were unable to identify a usable checkpoint for SDTT~\citep{deschenaux2024beyond}, which precluded a fair comparison or distillation of this model within our experimental framework.

\subsection{Conditional Generation on TinyGSM}
\label{app:tinygsm-experimental-details}

\noindent \textbf{Dataset preparation}.
We use TinyGSM for training and validation and tokenize examples with the \texttt{HuggingFaceTB/SmolLM-135M} tokenizer. Following the conditional-generation setup, examples are kept at sequence length $512$ without document wrapping. We insert \texttt{eos} tokens, use a newline separator between the problem and solution fields, filter examples that exceed the maximum length, and reserve $1\%$ of the data for validation with seed $42$. The problem prompt is treated as conditioning context and is excluded from the training loss, while padding tokens are included.

\noindent \textbf{Model and teacher checkpoint}.
We use the same small denoising Transformer architecture as in the OpenWebText experiments: $12$ layers, $12$ attention heads, hidden size $768$, sequence length $512$, and dropout $0.1$. IDLM-MDLM uses the MDLM \textsc{subs} parameterization with log-linear noise and is initialized from the TinyGSM MDLM teacher checkpoint. IDLM-Duo uses uniform diffusion with the mean parameterization, time conditioning, and adaptive layer normalization; the student is initialized from the TinyGSM Duo teacher, using the teacher EMA weights.

\noindent \textbf{Training hyperparameters}.
Both IDLM-MDLM and IDLM-Duo are trained with adversarial distillation enabled at every update. We use a global batch size of $512$, $4$ GPUs, DDP training, and AdamW with learning rate $1\times10^{-6}$, weight decay $0$, $\beta_1=0.9$, $\beta_2=0.999$, and $\epsilon=10^{-8}$. IDLM-MDLM uses per-device batch size $32$ and bfloat16 precision, while IDLM-Duo uses per-device batch size $16$ and float32 precision. We train for $250{,}000$ steps with a constant learning-rate schedule after $2{,}500$ warmup steps, use antithetic time sampling, and maintain an EMA copy of the model with decay $0.9999$.

\noindent \textbf{Evaluation protocol}.
We follow the conditional-generation protocol of~\citet{deschenaux2026language}. Each GSM8K evaluation instance is formatted as a problem prompt followed by a solution region. During sampling, the problem prompt is clamped and the model generates only the solution region. We evaluate IDLM-MDLM and IDLM-Duo with $32$, $64$, and $128$ reverse steps and compare them against the corresponding $1024$-step teachers. IDLM-MDLM uses the cached ancestral predictor, IDLM-Duo uses the ancestral sampler, and both use \texttt{float64} precision at sampling time.

\noindent \textbf{Metric}.
For each problem, we generate one completion and use the execution-based TinyGSM/GSM8K scorer to extract the predicted numerical answer. We report exact-match accuracy against the GSM8K ground-truth answer in Table~\ref{tab:tinygsm}.

\clearpage
\section{Additional Experimental Results}
\label{app:additional-experimental-results}

\subsection{Ancestral Sampling Results}
\label{app:ancestral-sampling-results}
Table~\ref{tab:uniform-ancestral-expanded} reports the detailed Duo$^{\mathrm{a}}$ results corresponding to the main OWT comparison. As mentioned above, Figure~\ref{fig:owt-genppl-steps} shows the corresponding GenPPL and entropy curves together with the other step-scaling results. Under ancestral sampling, IDLM-DCD remains strongest in the aggressive low-step regime. In particular, the 8-step sampler gives the clearest speed-quality tradeoff: it stays close to the 1024-step Duo teacher while reducing the reverse-step budget by $128\times$.

\begin{table}[t]
    \caption{\textbf{Duo$^{\mathrm{a}}$ distillation comparison.} We report GenPPL $\downarrow$, MAUVE $\uparrow$, and Entropy $\uparrow$. Methods are grouped by Steps, with the best metrics bolded in each group.}
    \label{tab:uniform-ancestral-expanded}
    \centering
    \footnotesize
    \setlength{\tabcolsep}{5pt}
    \begin{tabular}{lcccc}
        \toprule
         & {\scriptsize Steps} & {\scriptsize GenPPL $\downarrow$} & {\scriptsize MAUVE $\uparrow$} & {\scriptsize Entropy $\uparrow$} \\
        \midrule
        FLM~\citep{lee2026flow} & \multirow{2}{*}{1024} & \textbf{62.23} & -- & 5.33 \\
        Duo$^{\mathrm{a}}$ (Teacher)~\citep{sahoo2025the} & & 77.69 & \textbf{0.96} $\pm$ 0.01 & \textbf{5.55} \\
        \midrule
        ELF~\citep{hu2026elf} & \multirow{4}{*}{32} & \textbf{24.08} $\pm$ 0.16 & -- & 5.14 \\
        FMLM~\citep{lee2026flow} & & 45.09 & -- & 5.25 \\
        Duo-DCD$^{\mathrm{a}}$~\citep{sahoo2025the} & & 61.31 & \textbf{0.97} $\pm$ 0.01 & \textbf{5.52} \\
        IDLM-DCD$^{\mathrm{a}}$ (\textbf{Ours}) & & 42.13 $\pm$ 3.65 & 0.94 $\pm$ 0.01 & 5.41 $\pm$ 0.05 \\
        \midrule
        ELF~\citep{hu2026elf} & \multirow{4}{*}{16} & \textbf{33.66} $\pm$ 1.09 & -- & 5.16 \\
        FMLM~\citep{lee2026flow} & & 63.63 & -- & 5.29 \\
        Duo-DCD$^{\mathrm{a}}$~\citep{sahoo2025the} & & 75.24 & \textbf{0.96} $\pm$ 0.01 & \textbf{5.53} \\
        IDLM-DCD$^{\mathrm{a}}$ (\textbf{Ours}) & & 52.06 $\pm$ 4.23 & 0.95 $\pm$ 0.01 & 5.44 $\pm$ 0.05 \\
        \midrule
        ELF~\citep{hu2026elf} & \multirow{4}{*}{8} & 67.32 $\pm$ 2.25 & -- & 5.14 \\
        FMLM~\citep{lee2026flow} & & 86.5 & -- & 5.36 \\
        Duo-DCD$^{\mathrm{a}}$~\citep{sahoo2025the} & & 111.88 & 0.94 $\pm$ 0.01 & \textbf{5.52} \\
        IDLM-DCD$^{\mathrm{a}}$ (\textbf{Ours}) & & \textbf{66.41} $\pm$ 4.59 & \textbf{0.95} $\pm$ 0.01 & 5.42 $\pm$ 0.03 \\
        \midrule
        FMLM~\citep{lee2026flow} & \multirow{3}{*}{4} & 111.31 & -- & 5.26 \\
        Duo-DCD$^{\mathrm{a}}$~\citep{sahoo2025the} & & 261.82 & 0.79 $\pm$ 0.02 & \textbf{5.50} \\
        IDLM-DCD$^{\mathrm{a}}$ (\textbf{Ours}) & & \textbf{111.29} $\pm$ 8.39 & \textbf{0.88} $\pm$ 0.02 & 5.32 $\pm$ 0.06 \\
        \bottomrule
    \end{tabular}
\end{table}

The main comparison is with Duo-DCD under the same ancestral sampler. At 8 steps, IDLM-DCD improves GenPPL from 111.88 to 66.41 while keeping MAUVE and entropy close to the teacher. At 4 steps, Duo-DCD degrades sharply, whereas IDLM-DCD still preserves a usable quality-diversity tradeoff. Thus the benefit of IDLM is largest when very few reverse steps must replace the original 1024-step teacher trajectory.

\subsection{\texorpdfstring{GenPPL and Entropy vs. Sampling Steps}{GenPPL and Entropy vs. Sampling Steps}}
\label{app:genppl-step-plots}
Figure~\ref{fig:owt-genppl-steps} provides the main GenPPL and entropy curves that complement the compact OWT summary in Figure~\ref{fig:owt-main-results}. These plots show how generation quality and token-level diversity change as the number of reverse sampling steps is reduced. For MDLM, IDLM-MDLM gives the strongest low-step tradeoff relative to the teacher and SDTT baselines, improving GenPPL while maintaining comparable entropy. For Duo$^{\mathrm{g}}$, IDLM-DCD is most effective in the aggressive low-step regime, reaching strong GenPPL at $4$ steps while preserving competitive entropy. We also include the ancestral Duo curve here, so the greedy and ancestral step trends can be compared in one place.

\begin{figure*}[t]
    \centering
    \begin{subfigure}[t]{0.32\textwidth}
      \centering
      \includegraphics[width=\linewidth]{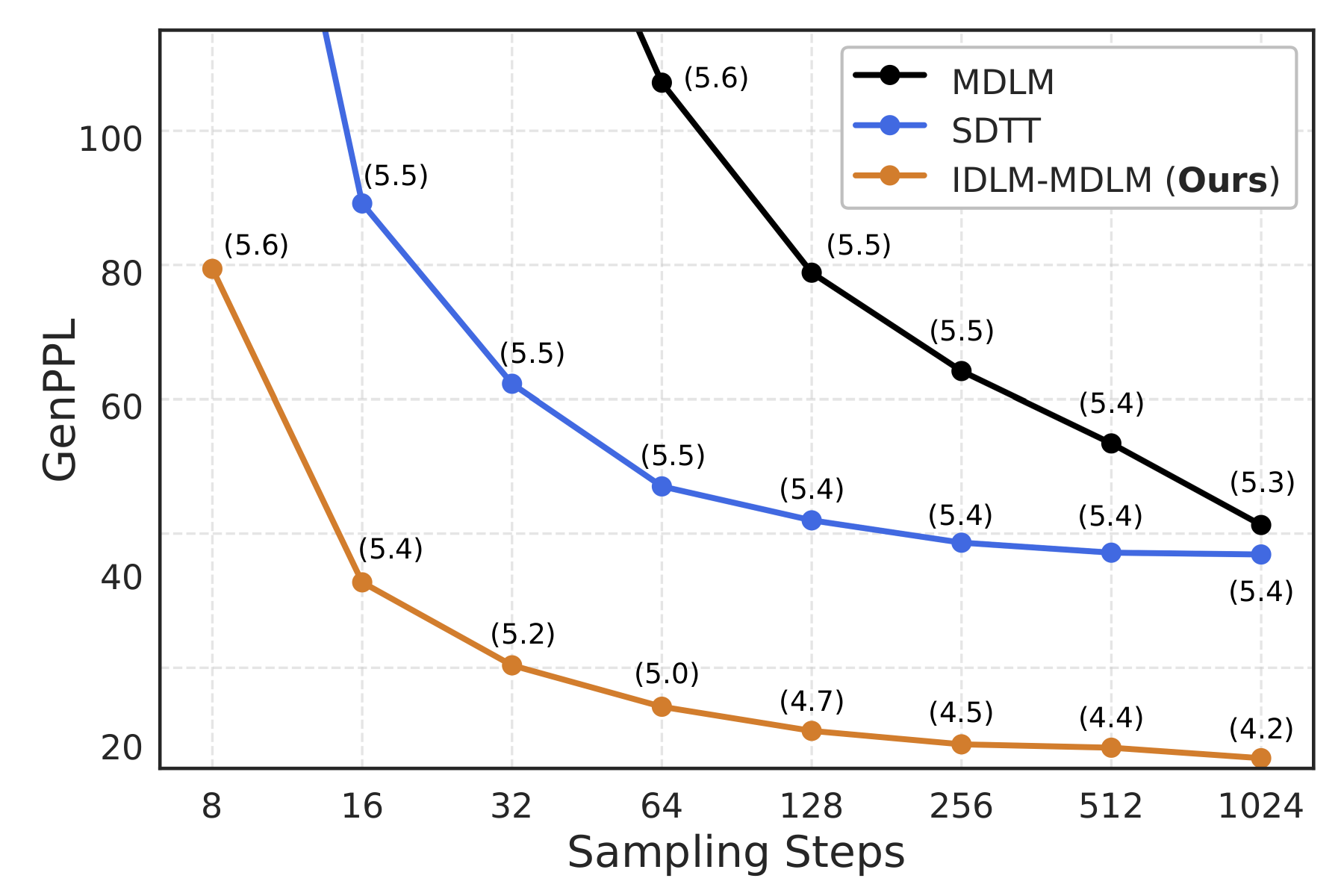}
      \caption{MDLM / SDTT}
    \end{subfigure}\hfill
    \begin{subfigure}[t]{0.32\textwidth}
      \centering
      \includegraphics[width=\linewidth]{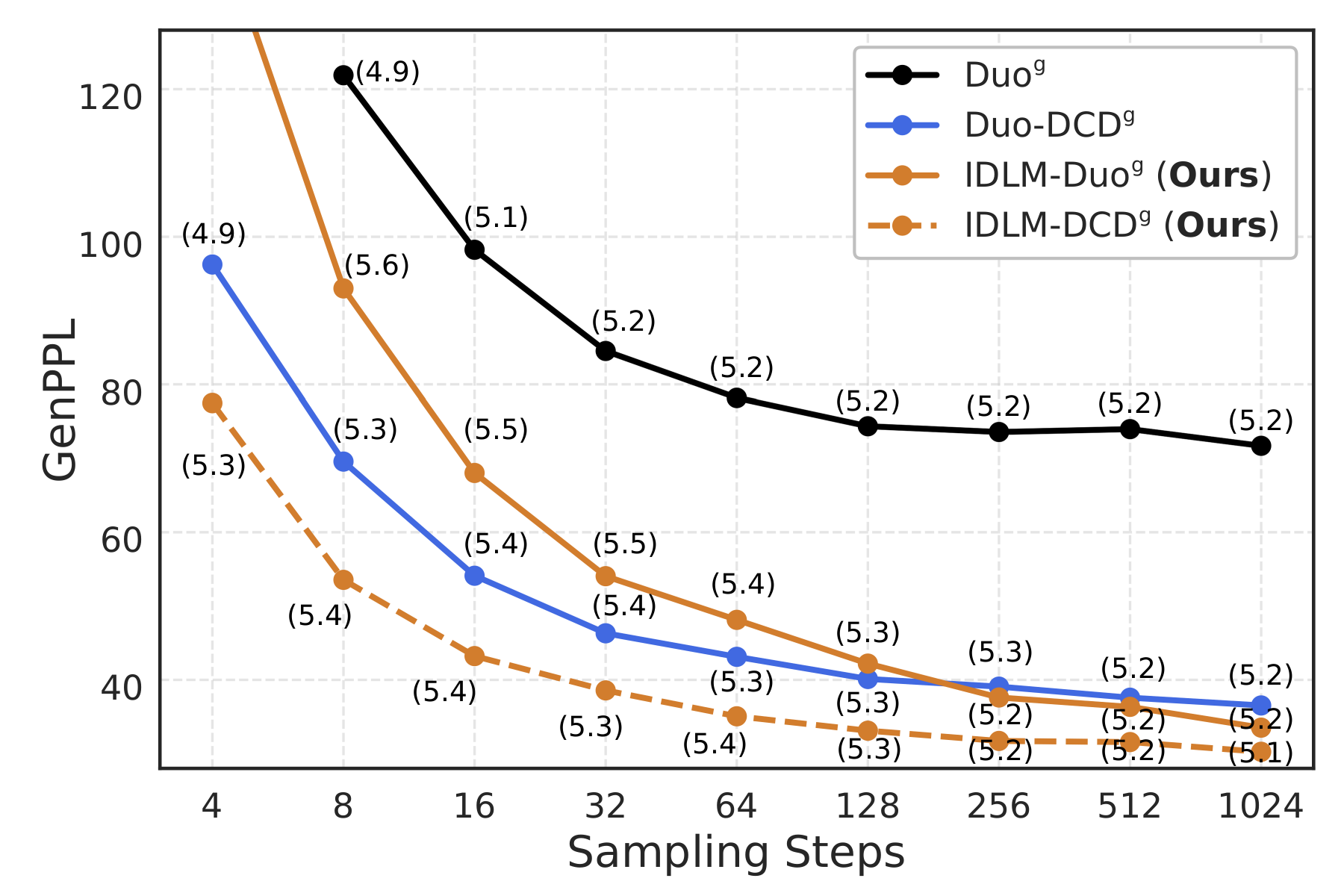}
      \caption{Duo$^{\mathrm{g}}$}
    \end{subfigure}\hfill
    \begin{subfigure}[t]{0.32\textwidth}
      \centering
      \includegraphics[width=\linewidth]{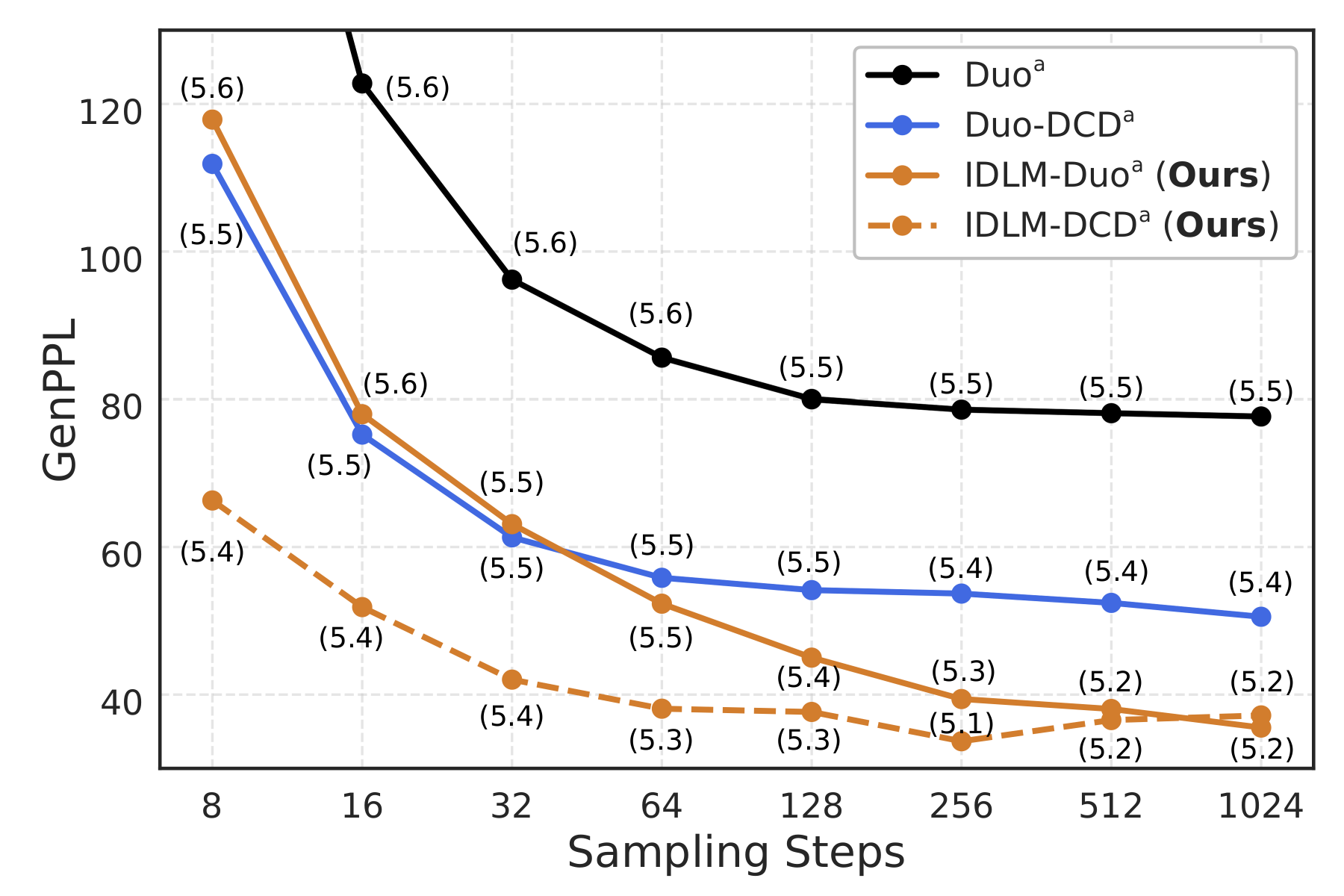}
      \caption{Duo$^{\mathrm{a}}$}
    \end{subfigure}
    \caption{\textbf{OWT GenPPL and entropy vs. sampling steps.} Each panel shows how generation quality and diversity change as the reverse sampling budget is reduced. GenPPL measures sample quality, where lower is better, while entropy tracks output diversity, where higher is better. The panels provide a unified comparison of step-scaling behavior for our IDLM-MDLM, IDLM-DCD$^{\mathrm{g}}$, and IDLM-DCD$^{\mathrm{a}}$ models.}
    \label{fig:owt-genppl-steps}
\end{figure*}

\subsection{Scaling to Larger MDLM}
\label{app:mdlm-scaling}
To test whether IDLM's low-step behavior persists beyond the $169$M setting, we distill an $863$M MDLM~\citep{sahoo2024simple} checkpoint from the SDTT~\citep{deschenaux2024beyond} work. This is the most direct larger-scale comparison available because SDTT uses the same model family and provides matched low-step baselines. As shown in Table~\ref{tab:mdlm-scaling}, IDLM-MDLM consistently improves over SDTT at $32$, $16$, and $8$ sampling steps, improving both GenPPL and entropy in the low-step regime. In particular, the $16$-step sampler preserves the $64\times$ reduction from the $1024$-step teacher, indicating that the acceleration behavior observed in the $169$M experiments also extends to the larger-scale regime.

\begin{table}[t]
\centering
\caption{\textbf{Scaling to an 863M MDLM teacher.} We report GenPPL $\downarrow$ and Entropy $\uparrow$.}
\label{tab:mdlm-scaling}
\footnotesize
\begin{tabular}{lccc}
\toprule
Method ($863$M params) & Steps & GenPPL $\downarrow$ & Entropy $\uparrow$ \\
\midrule
MDLM (Teacher)~\citep{sahoo2024simple} & 1024 & 34.99 & 5.24 \\
\midrule
SDTT~\citep{deschenaux2024beyond} & \multirow{2}{*}{32} & 31.16 & 5.12 \\
IDLM-MDLM \textbf{(Ours)} & & \textbf{13.71} $\pm$ 1.70 & \textbf{5.17} $\pm$ 0.10 \\
\midrule
SDTT~\citep{deschenaux2024beyond} & \multirow{2}{*}{16} & 52.13 & 5.20 \\
IDLM-MDLM \textbf{(Ours)} & & \textbf{24.07} $\pm$ 2.13 & \textbf{5.36} $\pm$ 0.04 \\
\midrule
SDTT~\citep{deschenaux2024beyond} & \multirow{2}{*}{8} & 121.51 & 5.25 \\
IDLM-MDLM \textbf{(Ours)} & & \textbf{67.37} $\pm$ 4.11 & \textbf{5.59} $\pm$ 0.03 \\
\bottomrule
\end{tabular}
\end{table}

\clearpage
\subsection{Ablation Study}
\label{app:ablation-results}

This subsection expands the ablations summarized in
Section~\ref{sec:ablation}. We keep the teacher and evaluation setup fixed, so
the comparisons isolate how gradients are passed to the generator and how much
latent stochasticity is available during masked-diffusion sampling.

\noindent\textbf{Beyond the \textsc{subs} parameterization.}
In masked diffusion, the \textsc{subs} parameterization makes the IDLM-MDLM
generator update mask-only: unmasked positions are copied by both the teacher
and fake model, so their contribution cancels. As discussed in
Section~\ref{sec:practical extension of idlm}, this removes the sampled
intermediate token from the generator gradient and behaves like an implicit
stop-gradient through the absorbing transition. We therefore test whether
explicitly differentiating through the intermediate state improves the update.
Let $x_0=G_\theta(\epsilon)$ and let $m$ denote the mask token. The Gumbel
relaxation softens the full categorical draw from $q_t(\cdot\mid x_0)$,
replacing the hard token by
\[
\tilde{x}^{\mathrm{G}}_t
=
\softmax\!\left(
\frac{\log q_t(\cdot\mid G_\theta(\epsilon)) + g}{\tau}
\right),
\qquad
g_i \overset{\mathrm{i.i.d.}}{\sim}\mathrm{Gumbel}(0,1).
\]
Its objective is
\[
\LC_{\mathrm{Gumbel}}(\theta)
=
-\EE_{t,\epsilon,g}
\left\langle
G_\theta(\epsilon),
\log f^*(\tilde{x}^{\mathrm{G}}_t,t)
-
\log \hat f(\tilde{x}^{\mathrm{G}}_t,t)
\right\rangle .
\]
The Argmax relaxation keeps the forward state closer to the hard masked
process by using the Gumbel-max form of
$q_t(x_t\mid x_0)=\alpha_t\delta_{x_0}(x_t)+(1-\alpha_t)\delta_m(x_t)$:
\[
b_{\mathrm{copy}}
=
\mathbf{1}\{\log \alpha_t + g_1
\geq
\log(1-\alpha_t)+g_2\},
\qquad
b_{\mathrm{mask}}=1-b_{\mathrm{copy}},
\]
\[
\tilde{x}^{\mathrm{A}}_t
=
b_{\mathrm{copy}}G_\theta(\epsilon)
+
b_{\mathrm{mask}}m,
\qquad
g_1,g_2 \overset{\mathrm{i.i.d.}}{\sim}\mathrm{Gumbel}(0,1).
\]
The corresponding objective is
\[
\LC_{\mathrm{Argmax}}(\theta)
=
-\EE_{t,\epsilon,g_1,g_2}
\left\langle
G_\theta(\epsilon),
\log f^*(\tilde{x}^{\mathrm{A}}_t,t)
-
\log \hat f(\tilde{x}^{\mathrm{A}}_t,t)
\right\rangle .
\]
Figure~\ref{fig:ablation-masked} shows that this additional gradient path does
not improve the masked-diffusion tradeoff. The \textsc{subs} update reaches the
useful low-GenPPL regime fastest while preserving substantially more entropy
than Argmax. Gumbel is smoother, but it reduces GenPPL more slowly and gives a
weaker quality--diversity profile near the shared training budget. This
supports the main-text hypothesis that feeding soft or branch-switched
intermediate states to an MDLM teacher can move the evaluation outside the
one-hot training domain and weaken the distillation signal.

\begin{center}
    \includegraphics[width=0.98\textwidth]{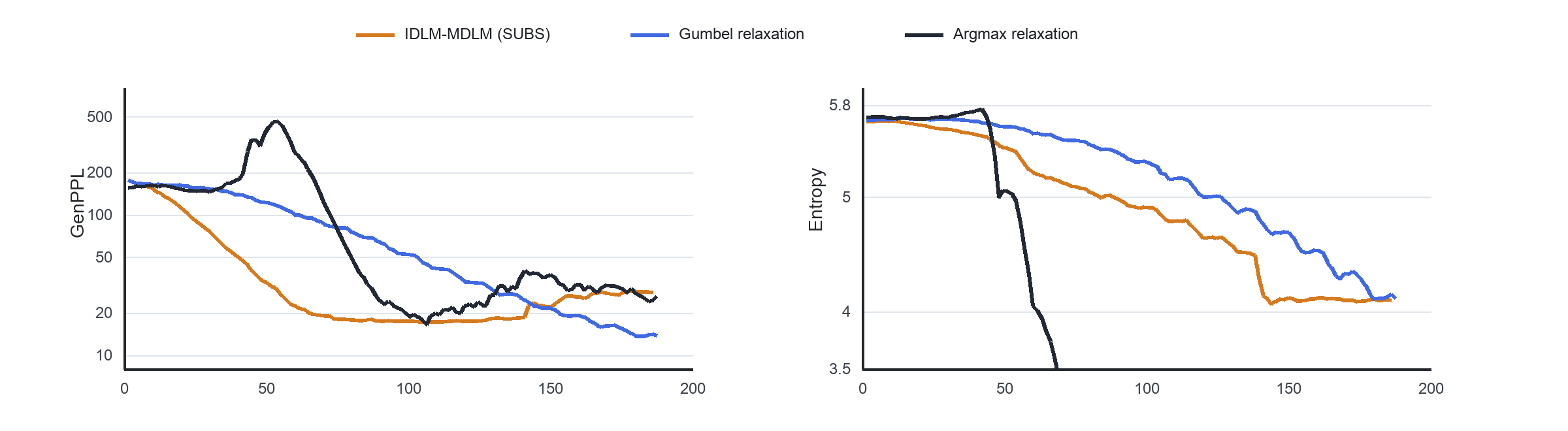}
    \vspace{-0.8em}

    {\small Training steps (k)}
    \captionsetup{hypcap=false}
    \captionof{figure}{
    \textbf{Masked-diffusion parameterization ablation.}
    Validation GenPPL and entropy for IDLM-MDLM with the \textsc{subs},
    Gumbel, and Argmax generator updates. The \textsc{subs} update gives the
    strongest early quality--diversity tradeoff under the shared training
    budget.
    }
    \label{fig:ablation-masked}
\end{center}

\noindent\textbf{DMD-like uniform distillation.}
For uniform diffusion, the Duo parameterization evaluates denoisers on relaxed
states, so the intermediate state remains differentiable. This makes the
difference between full inverse distillation and a DMD-like stop-gradient
reduction directly testable. Let
\[
\begin{aligned}
a_t(w,x)
&= g_{\mathrm{Duo}}\bigl(x_t(w),x,f^*(x_t^\tau(w),t)\bigr)\\
&\quad - g_{\mathrm{Duo}}\bigl(x_t(w),x,f(x_t^\tau(w),t)\bigr),
\end{aligned}
\]
where \(x_t(w)=\argmax(w)\) and \(x_t^\tau(w)=\softmax(w/\tau)\) are the
hard and relaxed Duo states. The full IDLM-Duo objective and the DMD-like
stop-gradient objective use the same teacher--fake Duo advantage, but differ in
whether the relaxed intermediate state carries gradients:
\begin{equation}
\label{eq:uniform-full-vs-sg}
\begin{aligned}
w_t
&= \tilde{\alpha}_t G_\theta(\epsilon)
 + \sqrt{1-\tilde{\alpha}_t^2}\,\xi,\\
w_t^{\mathrm{sg}}
&= \tilde{\alpha}_t \text{sg}\!\left(G_\theta(\epsilon)\right)
 + \sqrt{1-\tilde{\alpha}_t^2}\,\xi,\\
\LC_{\mathrm{IDLM}}^{\mathrm{Duo}}(\theta)
&= \EE\!\left[a_t(w_t,G_\theta(\epsilon))\right],\\
\LC_{\mathrm{DMD\text{-}like}}^{\mathrm{Duo}}(\theta)
&= \EE\!\left[a_t(w_t^{\mathrm{sg}},G_\theta(\epsilon))\right].
\end{aligned}
\end{equation}
Figure~\ref{fig:ablation-uniform} shows that this distinction matters in the
uniform case. The stop-gradient variant quickly enters a high-GenPPL,
high-entropy regime, so its entropy reflects instability rather than useful
diversity. In contrast, the full IDLM-Duo update remains controlled for most of
training. Thus, the formal connection between IDLM and DMD-like matching does
not imply that the two generator updates have the same optimization behavior
for uniform diffusion; the pathwise gradient through the relaxed intermediate
state is empirically important.

\begin{center}
    \includegraphics[width=0.98\textwidth]{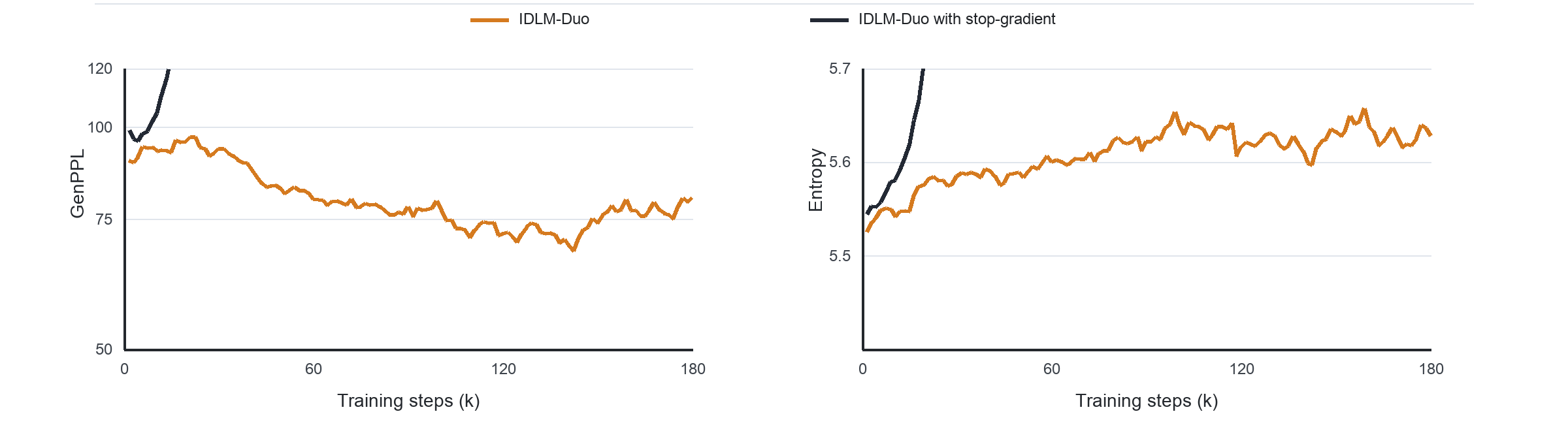}
    \captionsetup{hypcap=false}
    \captionof{figure}{
    \textbf{Uniform-diffusion stop-gradient ablation.}
    Validation GenPPL and entropy for IDLM-Duo and its DMD-like
    stop-gradient variant. Removing the intermediate-state gradient leads to a
    divergent high-GenPPL regime, whereas the full IDLM update remains stable.
    }
    \label{fig:ablation-uniform}
\end{center}

\noindent\textbf{Additional latent noise for masked diffusion.}
Finally, we test whether extra latent randomness helps masked diffusion in the
extreme low-step regime. Unlike uniform diffusion, masked diffusion starts from
a deterministic all-mask terminal state, so a deterministic one-step generator
has no rich latent source for sample diversity. We therefore inject Gaussian
noise into the IDLM-MDLM input as a simple control. Table~\ref{tab:gaussian-noise-control}
shows that this modification does not improve GenPPL or entropy:
\begin{center}
    \captionsetup{hypcap=false}
    \captionof{table}{
    \textbf{Additional latent noise for IDLM-MDLM.}
    Values are GenPPL with entropy in parentheses.
    }
    \label{tab:gaussian-noise-control}
    \small
\begin{tabular}{c|cc}
Steps & IDLM-MDLM & IDLM-MDLM + Gaussian noise \\
\hline
1  & 4056.12 (5.87) & 4096.54 (5.88) \\
8  & 79.75 (5.61)  & 81.61 (5.61) \\
16 & 32.75 (5.42)  & 35.76 (5.45) \\
32 & 20.45 (5.23)  & 21.03 (5.26)
\end{tabular}
\end{center}
The negative result is consistent with Section~\ref{sec:ablation}: simple
input-level Gaussian noise is insufficient for masked diffusion, and improving
the low-step regime may require stochasticity that is matched more closely to
the discrete absorbing process.

\section{Qualitative Samples}
\label{app: qualitative samples}

\newcommand{\qualpenicon}{%
\begin{tikzpicture}[baseline=-0.5ex, x=1cm, y=1cm]
    \filldraw[fill=owtgreen, draw=owtgreen, rounded corners=1.5pt] (0,0) rectangle (0.31,0.31);
    \draw[white, line width=0.95pt, line join=round, rounded corners=0.5pt]
        (0.08,0.10) -- (0.17,0.24) -- (0.25,0.20) -- (0.15,0.06) -- cycle;
    \draw[white, line width=0.8pt, line join=round]
        (0.08,0.10) -- (0.06,0.03) -- (0.15,0.06);
    \draw[white, line width=0.65pt, line cap=round] (0.18,0.23) -- (0.24,0.19);
    \draw[white, line width=0.85pt, line cap=round]
        (0.14,0.05) .. controls (0.18,0.01) and (0.22,0.06) .. (0.27,0.05);
\end{tikzpicture}}

\newcommand{\qualdocicon}{%
\begin{tikzpicture}[baseline=-0.5ex, x=1cm, y=1cm]
    \filldraw[fill=tinyorange, draw=tinyorange, rounded corners=1.5pt] (0,0) rectangle (0.31,0.31);
    \draw[white, line width=0.85pt, rounded corners=0.35pt]
        (0.085,0.055) -- (0.085,0.255) -- (0.205,0.255) -- (0.255,0.205) -- (0.255,0.055) -- cycle;
    \draw[white, line width=0.65pt] (0.205,0.255) -- (0.205,0.205) -- (0.255,0.205);
    \draw[white, line width=0.65pt, line cap=round] (0.12,0.175) -- (0.22,0.175);
    \draw[white, line width=0.65pt, line cap=round] (0.12,0.135) -- (0.22,0.135);
    \draw[white, line width=0.65pt, line cap=round] (0.12,0.095) -- (0.19,0.095);
\end{tikzpicture}}

\newcommand{\qualdivider}[1]{%
\vspace{-0.9mm}
\noindent\begin{tikzpicture}
\draw[#1, densely dashed, line width=0.5pt] (0,0) -- (\linewidth,0);
\end{tikzpicture}
\vspace{-2.0mm}}

\newcommand{\qualgroup}[1]{%
\par\vspace{1.4mm}
\noindent\begin{tikzpicture}
\draw[owtgreen, densely dashed, line width=0.45pt] (0,0) -- (\linewidth,0);
\end{tikzpicture}
\vspace{-1.0mm}
{\scriptsize\bfseries\textcolor{owtgreen}{#1}}\\[-0.2mm]}

\newcommand{\qualfirstgroup}[1]{%
\vspace{-0.4mm}
{\scriptsize\bfseries\textcolor{owtgreen}{#1}}\\[-0.2mm]}

\newcommand{\qualrowsep}{\par\vspace{1.0mm}}

\newcommand{\owtrow}[4]{%
\noindent
\begin{minipage}[t]{0.31\linewidth}
{\footnotesize\bfseries\textcolor{owtgreen}{#1 / #2 Sampling steps}}\\[-0.1mm]
{\scriptsize\bfseries #3}
\end{minipage}\hfill
\begin{minipage}[t]{0.66\linewidth}
{\scriptsize #4}
\end{minipage}}

\newcommand{\condmetric}[1]{%
{\scriptsize\bfseries\textcolor{tinyorange}{Accuracy:} #1}\\[-0.1mm]}

\newcommand{\condrowsep}{%
\par\vspace{0.8mm}
\noindent\begin{tikzpicture}
\draw[tinyorange, densely dashed, line width=0.45pt] (0,0) -- (\linewidth,0);
\end{tikzpicture}
\vspace{-0.8mm}}

This appendix provides compact qualitative excerpts for the unconditional OpenWebText setting and the conditional TinyGSM setting. Samples are shortened for readability.

\subsection{Unconditional Generation (OpenWebText)}
\label{app:owt-qualitative-samples}

\begin{tcolorbox}[
    enhanced,
    breakable,
    colframe=owtgreen,
    colback=green!3!white,
    boxrule=0.65pt,
    arc=4pt,
    left=5pt,
    right=5pt,
    top=3pt,
    bottom=3pt,
    width=\textwidth,
    before upper={\setlength{\parindent}{0pt}}]
\noindent
\begin{minipage}[c]{0.04\linewidth}
\centering
\qualpenicon
\end{minipage}%
\begin{minipage}[c]{0.94\linewidth}
{\normalsize\bfseries\textcolor{owtgreen}{Unconditional Generation (OpenWebText)}}
\end{minipage}
\qualdivider{owtgreen}

\qualfirstgroup{SEDD}
\owtrow{IDLM-SEDD}{256}
{\textcolor{owtgreen}{Gen. PPL:} 39.24; \textcolor{owtgreen}{H:} 5.3}
{\ldots{} We knew the injuries against Fulham and Newcastle prevented us from gaining victory, but we felt we deserved a big victory in a game which brought us back to the group stage. There were only two fifth division sides in the last eight games, and the Premier League came into the day in third place. \ldots{}}
\qualrowsep
\owtrow{SEDD}{1024}
{\textcolor{owtgreen}{Gen. PPL:} 43.31; \textcolor{owtgreen}{H:} 5.2}
{\ldots{} I think we want to win every championship with these players, and we have given all I have to say. I know and I do not know that we will rush forward. In our view, we will not give any assurances that we are going to stay. \ldots{}}

\qualgroup{Masked Diffusion}
\owtrow{IDLM-MDLM}{16}
{\textcolor{owtgreen}{Gen. PPL:} 32.75 $\pm$ 1.62; \textcolor{owtgreen}{H:} 5.42 $\pm$ 0.06}
{\ldots{} In the wake of angry remarks from the media, Sam Kicks Radio decided to take a shot, claiming that he had only ever learned how to get away with something bad. I want to get over it, he says, while the exchange returns to the same public dispute. \ldots{}}
\qualrowsep
\owtrow{MDLM}{1024}
{\textcolor{owtgreen}{Gen. PPL:} 41.29; \textcolor{owtgreen}{H:} 5.28}
{\ldots{} Google claims that Android is its own operating system and could soon be replaced by more complicated cloud services from Samsung and Mozilla. Several business applications on mobile devices bring many direct interaction options. \ldots{}}

\qualgroup{Uniform Diffusion (Ancestral Sampling)}
\owtrow{IDLM-Duo$^{\mathrm{a}}$}{16}
{\textcolor{owtgreen}{Gen. PPL:} 78.00; \textcolor{owtgreen}{H:} 5.6}
{\ldots{} I think it is important to worry about better credit security and better integrated trading of those metrics, says a professor of regenerative engineering at McMaster University. The resident range has grown to 7.4 million in the UAE's next 25 years. \ldots{}}
\qualrowsep
\owtrow{IDLM-DCD$^{\mathrm{a}}$}{8}
{\textcolor{owtgreen}{Gen. PPL:} 66.41 $\pm$ 4.59; \textcolor{owtgreen}{H:} 5.42 $\pm$ 0.03}
{\ldots{} I am mostly an economist; for me, I am one of the co-creators of the culture of socialization and a simple person. A computer is enhanced with incredible energy by itself, while the passage moves through biography and explanation. \ldots{}}
\qualrowsep
\owtrow{Duo-DCD$^{\mathrm{a}}$}{16}
{\textcolor{owtgreen}{Gen. PPL:} 75.24; \textcolor{owtgreen}{H:} 5.53}
{\ldots{} My favorite cuisine is pretty varied, changing from regular to country cooking. Some local dishes are fine when guests visit the hotel, and the Southern Japanese Thai restaurant is described as something to watch. \ldots{}}
\qualrowsep
\owtrow{Duo$^{\mathrm{a}}$}{1024}
{\textcolor{owtgreen}{Gen. PPL:} 77.69; \textcolor{owtgreen}{H:} 5.55}
{\ldots{} The Germanic literature is based on a set of cultural expectations, while the English is largely freighted with fiction. Thus, it would be a useful contribution to this field for a shift away from the vernacular. \ldots{}}

\qualgroup{Uniform Diffusion (Greedy Sampling)}
\owtrow{IDLM-Duo$^{\mathrm{g}}$}{16}
{\textcolor{owtgreen}{Gen. PPL:} 68.04; \textcolor{owtgreen}{H:} 5.6}
{\ldots{} Strong performance in the NBN Co network will drive costs toward power parity, said Paul McCormick, chief economist at Business Economics. Strong growth of the network in Telstra Powerplows will drive costs in real time. \ldots{}}
\qualrowsep
\owtrow{IDLM-DCD$^{\mathrm{g}}$}{4}
{\textcolor{owtgreen}{Gen. PPL:} 77.47 $\pm$ 4.43; \textcolor{owtgreen}{H:} 5.28 $\pm$ 0.06}
{\ldots{} This device is not the only way to use features in this article. People who conduct interviews in secure environments are attracted by a two-way monitor, making it an ideal candidate for an interviewee who can adapt to secure systems. \ldots{}}
\qualrowsep
\owtrow{Duo-DCD$^{\mathrm{g}}$}{8}
{\textcolor{owtgreen}{Gen. PPL:} 69.58; \textcolor{owtgreen}{H:} 5.30}
{\ldots{} Craft beer from Wisconsin was based on a local foundation, while the market spread through misleading advertising and state-level pushes for ways to craft beer. The passage repeats the theme of freshness and local identity. \ldots{}}
\qualrowsep
\owtrow{Duo$^{\mathrm{g}}$}{1024}
{\textcolor{owtgreen}{Gen. PPL:} 71.72; \textcolor{owtgreen}{H:} 5.22}
{\ldots{} Medical applications and patents are being struck down under the patent devices act. The patent does not include more than medical devices or the described holder, and readers are asked to note changes in size or functionality. \ldots{}}
\end{tcolorbox}

\subsection{Conditional Generation (TinyGSM)}
\label{app:tinygsm-qualitative-samples}

\begin{tcolorbox}[
    enhanced,
    breakable,
    colframe=tinyorange,
    colback=orange!3!white,
    boxrule=0.65pt,
    arc=4pt,
    left=5pt,
    right=5pt,
    top=3pt,
    bottom=3pt,
    width=\textwidth,
    before upper={\setlength{\parindent}{0pt}}]
\noindent
\begin{minipage}[c]{0.04\linewidth}
\centering
\qualdocicon
\end{minipage}%
\begin{minipage}[c]{0.94\linewidth}
{\normalsize\bfseries\textcolor{tinyorange}{Conditional Generation (TinyGSM)}}
\end{minipage}
\qualdivider{tinyorange}

\noindent
\begin{minipage}[t]{0.31\linewidth}
{\footnotesize\bfseries\textcolor{tinyorange}{IDLM-MDLM / 32 Sampling steps}}\\[-0.1mm]
\condmetric{12.8\%}
\end{minipage}\hfill
\begin{minipage}[t]{0.66\linewidth}
{\scriptsize\bfseries Prompt.}
{\scriptsize Marcus is half of Leo's age and five years younger than Deanna. Deanna is 26. How old is Leo?}\\[-0.1mm]
{\scriptsize\bfseries Generation.}\\[-2.4mm]
\begin{lstlisting}[language=Python,basicstyle=\ttfamily\tiny,columns=fullflexible,keepspaces=true,breaklines=true,aboveskip=0pt,belowskip=0pt,xleftmargin=0.04\linewidth,xrightmargin=0.04\linewidth]
def simple_math_problem() -> int:
    marcus = 26 - 5
    leo = marcus * 2
    result = leo
    return result
\end{lstlisting}
\end{minipage}

\par\vspace{1.0mm}
\noindent
\begin{minipage}[t]{0.31\linewidth}
{\footnotesize\bfseries\textcolor{tinyorange}{IDLM-MDLM / 64 Sampling steps}}\\[-0.1mm]
\condmetric{14.9\%}
\end{minipage}\hfill
\begin{minipage}[t]{0.66\linewidth}
{\scriptsize\bfseries Prompt.}
{\scriptsize Dan plants 3 rose bushes. Each rose bush has 25 roses. Each rose has 8 thorns. How many thorns are there total?}\\[-0.1mm]
{\scriptsize\bfseries Generation.}\\[-2.4mm]
\begin{lstlisting}[language=Python,basicstyle=\ttfamily\tiny,columns=fullflexible,keepspaces=true,breaklines=true,aboveskip=0pt,belowskip=0pt,xleftmargin=0.04\linewidth,xrightmargin=0.04\linewidth]
def simple_math_problem() -> int:
    thorns = 3 * 25 * 8
    result = thorns
    return result
\end{lstlisting}
\end{minipage}

\par\vspace{1.0mm}
\noindent
\begin{minipage}[t]{0.31\linewidth}
{\footnotesize\bfseries\textcolor{tinyorange}{IDLM-MDLM / 128 Sampling steps}}\\[-0.1mm]
\condmetric{19.9\%}
\end{minipage}\hfill
\begin{minipage}[t]{0.66\linewidth}
{\scriptsize\bfseries Prompt.}
{\scriptsize John collects peaches for 3 hours. He can collect 2 peaches a minute. How many peaches does he collect?}\\[-0.1mm]
{\scriptsize\bfseries Generation.}\\[-2.4mm]
\begin{lstlisting}[language=Python,basicstyle=\ttfamily\tiny,columns=fullflexible,keepspaces=true,breaklines=true,aboveskip=0pt,belowskip=0pt,xleftmargin=0.04\linewidth,xrightmargin=0.04\linewidth]
def simple_math_problem() -> int:
    minutes = 3 * 60
    peaches = 2 * minutes
    result = peaches
    return result
\end{lstlisting}
\end{minipage}

\condrowsep
\noindent
\begin{minipage}[t]{0.31\linewidth}
{\footnotesize\bfseries\textcolor{tinyorange}{IDLM-Duo / 32 Sampling steps}}\\[-0.1mm]
\condmetric{15.4\%}
\end{minipage}\hfill
\begin{minipage}[t]{0.66\linewidth}
{\scriptsize\bfseries Prompt.}
{\scriptsize Ted starts with \$200. He buys 3 books for 16 dollars each and 3 pencils for 6 dollars each. How much did he spend in total?}\\[-0.1mm]
{\scriptsize\bfseries Generation.}\\[-2.4mm]
\begin{lstlisting}[language=Python,basicstyle=\ttfamily\tiny,columns=fullflexible,keepspaces=true,breaklines=true,aboveskip=0pt,belowskip=0pt,xleftmargin=0.04\linewidth,xrightmargin=0.04\linewidth]
def simple_math_problem() -> int:
    total_cost = 3 * 16 + 3 * 6
    result = total_cost
    return result
\end{lstlisting}
\end{minipage}

\par\vspace{1.0mm}
\noindent
\begin{minipage}[t]{0.31\linewidth}
{\footnotesize\bfseries\textcolor{tinyorange}{IDLM-Duo / 64 Sampling steps}}\\[-0.1mm]
\condmetric{19.0\%}
\end{minipage}\hfill
\begin{minipage}[t]{0.66\linewidth}
{\scriptsize\bfseries Prompt.}
{\scriptsize Sam ran 3 miles on Monday, Wednesday and Friday. On Tuesday and Thursday, Sam ran 5 miles. How many miles did Sam run this week?}\\[-0.1mm]
{\scriptsize\bfseries Generation.}\\[-2.4mm]
\begin{lstlisting}[language=Python,basicstyle=\ttfamily\tiny,columns=fullflexible,keepspaces=true,breaklines=true,aboveskip=0pt,belowskip=0pt,xleftmargin=0.04\linewidth,xrightmargin=0.04\linewidth]
def simple_math_problem() -> int:
    total_miles = (3 * 3) + 5 * 2
    result = total_miles
    return result
\end{lstlisting}
\end{minipage}

\par\vspace{1.0mm}
\noindent
\begin{minipage}[t]{0.31\linewidth}
{\footnotesize\bfseries\textcolor{tinyorange}{IDLM-Duo / 128 Sampling steps}}\\[-0.1mm]
\condmetric{21.4\%}
\end{minipage}\hfill
\begin{minipage}[t]{0.66\linewidth}
{\scriptsize\bfseries Prompt.}
{\scriptsize James has to buy insurance. Since he had an accident, it was 60\% more than normal. The normal cost is \$120 a month. How much does he pay a year?}\\[-0.1mm]
{\scriptsize\bfseries Generation.}\\[-2.4mm]
\begin{lstlisting}[language=Python,basicstyle=\ttfamily\tiny,columns=fullflexible,keepspaces=true,breaklines=true,aboveskip=0pt,belowskip=0pt,xleftmargin=0.04\linewidth,xrightmargin=0.04\linewidth]
def simple_math_problem() -> int:
    yearly_cost = (120 * (1 + 0.60)) * 12
    result = yearly_cost
    return result
\end{lstlisting}
\end{minipage}
\end{tcolorbox}



\end{document}